\documentclass{jair}
\usepackage{algorithm}
\usepackage{algorithmicx}
\usepackage{algpseudocode}
\usepackage{stmaryrd}
\usepackage{graphicx}
\usepackage{subcaption}
\usepackage{xspace}
\usepackage{tikz}
\usepackage{booktabs}
\theoremstyle{definition}
\newtheorem{definition}{Definition}[section]
\newtheorem{lemma}{Lemma}
\newtheorem{remark}{Remark}
\usepackage{wrapfig}
\usepackage{eucal}
\usepackage{array}
\usepackage{standalone}
\usepackage{hyperref}
\usepackage{xcolor}
\usepackage{enumitem}
\newcommand{\vitax}{\textsc{ViTaX}\xspace}

\setcopyright{cc}
\copyrightyear{2026}
\acmYear{2026}
\acmDOI{10.1613/jair.1.xxxxx}

\JAIRAE{Eleonora Giunchiglia}
\JAIRTrack{Integration of Logical Constraints in Deep Learning}
\acmVolume{4}
\acmArticle{111}
\acmMonth{5}
\acmYear{2026}

\begin{document}

\title{Towards Verified and Targeted Explanations through Formal Methods}

\author{Hanchen David Wang}
\email{hanchen.wang.1@Vanderbilt.Edu}
\orcid{0000-0001-5990-5865}
\affiliation{%
  \institution{Vanderbilt University}
  \city{Nashville}
  \state{Tennessee}
  \country{USA}
}

\author{Diego Manzanas Lopez}
\email{diego.manzanas.lopez@vanderbilt.edu}
\affiliation{%
  \institution{Vanderbilt University}
  \city{Nashville}
  \state{Tennessee}
  \country{USA}
}

\author{Preston K. Robinette}
\email{preston.k.robinette@vanderbilt.edu}
\affiliation{%
  \institution{Vanderbilt University}
  \city{Nashville}
  \state{Tennessee}
  \country{USA}
}

\author{Ipek Oguz}
\email{ipek.oguz@Vanderbilt.Edu}
\affiliation{%
  \institution{Vanderbilt University}
  \city{Nashville}
  \state{Tennessee}
  \country{USA}
}

\author{Taylor T. Johnson}
\email{taylor.johnson@vanderbilt.edu}
\affiliation{%
  \institution{Vanderbilt University}
  \city{Nashville}
  \state{Tennessee}
  \country{USA}
}

\author{Meiyi Ma}
\email{meiyi.ma@Vanderbilt.Edu}
\affiliation{%
  \institution{Vanderbilt University}
  \city{Nashville}
  \state{Tennessee}
  \country{USA}
}
\renewcommand{\shortauthors}{Wang et al.}

\begin{abstract}
As deep neural networks are deployed in safety-critical domains such as autonomous driving and medical diagnosis, stakeholders need explanations of model behavior that are not only interpretable but also trustworthy with formal guarantees. Existing XAI methods fall short of this requirement: heuristic attribution techniques (e.g., LIME, Integrated Gradients) highlight influential features for individual predictions but offer no mathematical guarantees about decision boundaries, while formal explanation methods verify robustness properties yet remain untargeted, analyzing the nearest boundary regardless of whether it represents a critical risk. In safety-critical systems, however, not all misclassifications carry equal consequences; confusing a ``Stop'' sign for a ``60\,kph'' sign is far more dangerous than confusing it with a ``No Passing'' sign. Practitioners therefore lack a principled way to answer a fundamental safety question: \emph{how resilient is a model's classification against a specific, high-risk alternative?}
We introduce \textbf{\vitax{}} (Verified and Targeted Explanations), a formal XAI framework that addresses this gap by generating targeted semifactual explanations with mathematical guarantees. For a given input (class~$y$) and a user-specified critical alternative (class~$t$), \vitax{} performs two key steps: \textbf{(1)}~it identifies the minimal feature subset most sensitive to the $y \to t$ transition using class-specific sensitivity heuristics, and \textbf{(2)}~it applies formal reachability analysis to guarantee that perturbing these features by~$\epsilon$ is insufficient to flip the classification to~$t$. This guarantee constitutes a verified semifactual: ``even if these critical features change by~$\epsilon$, classification~$y$ persists against~$t$.'' We formalize this reasoning through \textit{Targeted $\epsilon$-Robustness}, a formal property that certifies whether an identified feature subset remains robust under perturbation toward a specific target class.
By unifying semifactual explanations, class-specific targeting, and formal verification, \vitax{} is the first method to provide formally guaranteed explanations of a model's resilience against specific, user-identified alternatives. Our evaluations on image classification (MNIST, GTSRB, EMNIST) and regression (TaxiNet) demonstrate that \vitax{} achieves significantly higher fidelity (e.g., over 30\% improvement) and minimal explanation cardinality compared to existing methods. These results establish \vitax{} as a scalable and trustworthy foundation for verifiable, targeted XAI\footnote{Our code is publicly available at \url{https://github.com/AICPS-Lab/formal-xai}.}.
\end{abstract} 



\received{31 May 2025}
\received[revised]{02 November 2025}
\received[accepted]{03 March 2026}

\maketitle

\section{Introduction}\label{sec:intro}

\definecolor{RiskOrange}{HTML}{FF9B33}
\definecolor{RiskRed}{HTML}{FF3233}

The deployment of Deep Neural Networks in safety-critical domains, such as autonomous driving and medical diagnosis, demands more than high accuracy; it requires trustworthiness. When failures occur, the consequences can be severe, necessitating rigorous methods to understand and verify model behavior. Formal methods in eXplainable AI (Formal XAI) have emerged to address this need, moving beyond heuristic estimations (e.g., LIME \cite{ribeiro2016should}, Integrated Gradients \cite{sundararajan_axiomatic_2017}) to provide mathematical guarantees about model decisions \cite{marques-silva_delivering_2022,jiang_formalising_2023}. More broadly, formal specifications such as temporal logic have been integrated into neural networks to enforce properties in sequential prediction~\cite{ma2020stlnet}, federated settings~\cite{an2024formal}, and predictive monitoring of cyber-physical systems~\cite{ma2021predictive}.

However, existing Formal XAI approaches face a critical limitation: they often ignore the asymmetric nature of real-world risk. In safety-critical systems, not all misclassifications are equal. For instance, confusing one speed limit sign for another might be minor, but confusing a `Stop' sign for a `60 kph' sign is catastrophic. Furthermore, formal verification is computationally expensive. It is therefore crucial to allocate these limited resources effectively, prioritizing the decision boundaries that matter most.

Many current formal explanation methods are inherently constrained by the \emph{nearest} decision boundary in the model's feature space, regardless of whether that boundary represents a critical risk. 
This limitation is illustrated in Figure~\ref{fig:motivation}(A). A traffic sign recognition model correctly identifies a `Stop' sign (Class 14). The nearest boundary is `No Passing' (Class 10), a low risk. The critical boundary is `60 kph' (Class 5), a high risk. 
When applying existing formal methods, the analysis often defaults to the nearest boundary. Methods providing instance-level explanations, such as robust semifactuals and counterfactuals for Human-Neural Multi-Agent Systems (HNMAS) (Figure~\ref{fig:motivation}(B)) \cite{leofante2023robust}, generate modified inputs but lack fine-grained feature insights and drift towards the nearest boundary (Class 10). Untargeted formal explanation methods, such as VeriX (Figure~\ref{fig:motivation}(C)) \cite{wu_verix_2023}, aim to verify robustness against \emph{any} alternative class. Constrained by Class 10, VeriX produces a broad explanation, offering no insight into the resilience against the catastrophic Class 5. These methods fail to answer the crucial question for safety-critical deployment: ``How resilient is the model's decision against a specific, high-risk alternative?''

To address this issue, we must distinguish between explanations of \emph{fragility} and explanations of \emph{resilience}. Counterfactual Explanations (CFEs) explain fragility by answering ``if-only'' questions: ``If only these features changed, the decision would flip'' \cite{guidotti2024counterfactual}. We focus instead on \textbf{\textit{semifactual}} explanations, which explain resilience by answering ``even-if'' questions: ``Even if these features change, the decision persists'' \cite{aryal2023even, kenny2023utility}.

We introduce \vitax{} (Verified Targeted Explanations), a novel framework that provides formally verified, targeted semifactual explanations. \vitax{} addresses the limitations of existing methods by prioritizing verification resources on user-specified, critical decision boundaries. For a given input (class $y$) and a user-specified critical alternative (class $t$), \vitax{} performs two key steps: \textbf{(1)} It identifies the minimal feature subset $A$ that is \emph{most sensitive} to the specific $y \to t$ transition. \textbf{(2)} It uses formal reachability analysis to \emph{guarantee} that perturbing these sensitive features by a magnitude $\epsilon$ is insufficient to flip the classification to $t$.

As shown in Figure~\ref{fig:motivation}(D), \vitax{} ignores the nearest boundary (Class 10) and focuses analysis on the critical boundary (Class 5). The resulting explanation is precise, highlighting only the features relevant to the Stop $\to$ 60 kph transition. The perturbed logits confirm that the analysis is focused on Class 5. This provides a formal guarantee of resilience (``even-if'') at the specific boundary that matters most.

\begin{figure*}[t]
    \centering
    \includegraphics[width=0.95\linewidth]{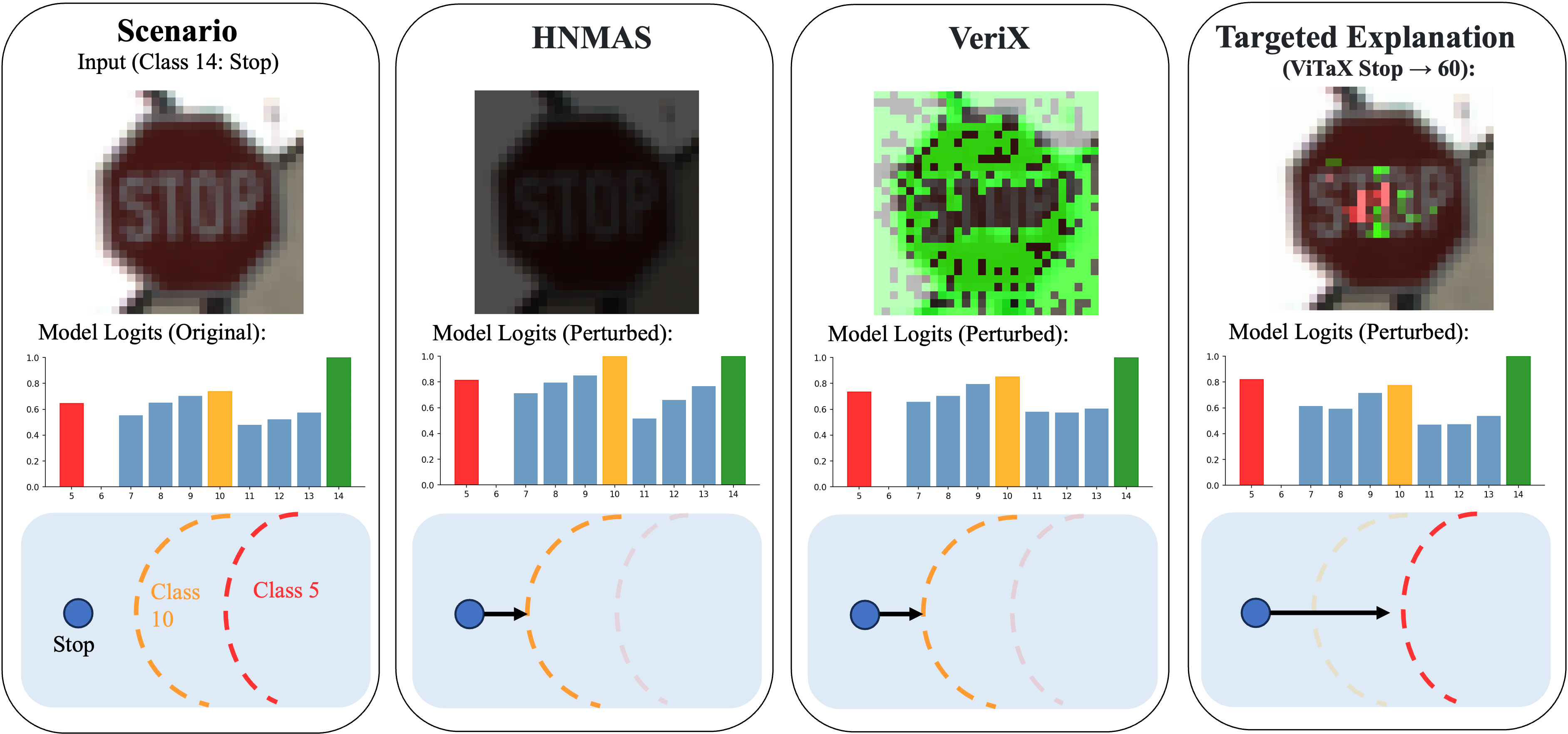}

    \caption{\textbf{The Necessity of Targeted Verification: Prioritizing Critical Boundaries over Nearest Boundaries.} All methods use the same base model (GTSRB CNN). Formal verification is expensive and must be deployed strategically.
    \textbf{(A) Scenario:} A model correctly classifies 'Stop' (Class 14). The decision space has asymmetric risks: Class 10 (`No Pass') is the \textit{nearest} boundary (low risk), while Class 5 (`60 kph') is the \textit{critical} boundary (high risk). Risk levels are defined by the application domain (e.g., traffic safety regulations) and provided by the user; \vitax{} then directs verification resources toward the user-specified critical boundary.
    \textbf{(B) HNMAS} (Formal Instance-level) \cite{leofante2023robust}: Provides instance-level semifactual explanations via robustness for multi-agent systems. The analysis is constrained by local geometry, and the evaluation drifts (arrow) towards the nearest boundary (Class 10).
    \textbf{(C) VeriX} (Untargeted Formal) \cite{wu_verix_2023}: Provides feature-level explanations of general resilience. The analysis is constrained by the nearest boundary, resulting in a broad explanation and evaluation drift (arrow) towards Class 10.
    \textbf{(D) \vitax{}} (Targeted Formal): Provides feature-level explanations of targeted resilience. \vitax{} focuses resources (arrow) on the critical boundary, resulting in a precise explanation and evaluation drift towards Class 5. In safety-critical deployment, understanding resilience against high-risk alternatives (Class 5, `60 kph') is essential, even when other boundaries (Class 10, `No Pass') are closer but less dangerous. \vitax{} provides guarantees where they matter most.
    }
    \label{fig:motivation}
\end{figure*}

Our main contributions are:
\begin{itemize}[nosep]
    \item We introduce \vitax{}, the first framework to combine semifactual explanations (``even-if'' robustness) with class-specific targeting ($y \to t$ transitions) and formal verification. This addresses a critical gap where existing methods either lack formal guarantees or ignore asymmetric risks in safety-critical decision boundaries.
    \item We define \textit{Targeted $\epsilon$-Robustness}, a novel formal property for characterizing the robustness of explanations. This property captures the resilience of a model against a specific alternative class under feature perturbations, enabling precise, verified characterization of class-specific decision boundaries.
    \item We propose a novel integration of sensitivity-driven heuristics and formal verification, utilizing an efficient search strategy ($\mathcal{O}(\log N)$ oracle calls) that makes targeted formal explanations computationally tractable, providing a practical alternative to intractable exact methods.
    \item We demonstrate on diverse tasks (image classification and regression) that \vitax{} significantly outperforms existing heuristic and formal methods (including VeriX, HNMAS, LIME, Anchors) by providing precise, high-fidelity insights into critical decision boundaries, achieving over 30\% improvement in explanation fidelity while maintaining minimal cardinality.
\end{itemize}

\subsection*{Paper Organization}
The remainder of this paper is structured as follows. Section~\ref{sec:reachability_solver_verification} lays the formal groundwork, introducing key concepts like standard $\epsilon$-robustness and defining our core notion of \textit{Targeted $\epsilon$-Robustness}. Section~\ref{sec:method} then details the \vitax{} framework, including its heuristic ranking, formal feature search algorithm for identifying explanations that satisfy Targeted $\epsilon$-Robustness, and the soundness of these explanations. Section~\ref{sec:experiment} presents a comprehensive empirical evaluation of \vitax{} across various datasets and tasks, comparing it against several baseline methods using metrics such as fidelity, cardinality, and robustness. We then situate our work within the broader XAI landscape in Section~\ref{sec:related_work} by reviewing related literature on attribution methods, formal XAI, causality, and explainability metrics. Section~\ref{sec:discussion} further discusses the unique nature and value of \vitax{}'s guarantees, including their connection to semifactual explanations, and thoroughly compares \vitax{} with contrastive explanation paradigms and research in adversarial robustness. Finally, Section~\ref{sec:conclusion} concludes the paper by summarizing our contributions, reiterating the impact of \vitax{}, and outlining promising directions for future research. Appendices provide supplementary details including solver methods, proofs, additional experimental results, and model specifications.

\section{From Standard to Targeted \texorpdfstring{$\epsilon$}{ε}-Robust Explanations}\label{sec:reachability_solver_verification}

In this section, we introduce the foundational concepts and formal problem formulation essential for understanding \vitax{}. Specifically, we formalize the concept of \textit{resilience} from our introduction, providing a mathematical basis for quantifying a model's ability to maintain its decisions under targeted perturbations.
These preliminaries establish the basis for our proposed method, which leverages formal reachability analysis to offer specific robustness guarantees pertinent to the identified explanations of neural network behavior. We present these concepts assuming the neural network input is an image with width $w$, height $h$, and $c$ color channels, in a classification task with $m$ classes.

Consider a deep learning model $f: \mathbb{R}^{w\times h\times c} \rightarrow \mathbb{R}^m$, where the function $f$ maps an $(w\times h\times c)$-dimensional input image to an $m$-dimensional output vector (representing min-max normalized logits).
The inputs and outputs of $f$ define the feature space and the prediction space, respectively. To formally assess the model's output stability under specific input perturbations, we employ a reachability solver, denoted as $\mathcal{V}(f, X_\epsilon, \Phi)$.
Here, $X_\epsilon \subseteq \mathbb{R}^{w\times h\times c}$ represents a continuous, infinite set of input images derived from an original input by applying perturbations up to a magnitude of $\epsilon$ (detailed in Definition~\ref{def:standard_robustness_def}). The term $\Phi$ represents the formal specification, a property we aim to verify, which in our context is a form of decision resilience.
The solver evaluates the model $f$ over the entire input set $X_\epsilon$ to compute the corresponding reach set of outputs $Y_\epsilon \subseteq \mathbb{R}^m$. That is, $f(X_\epsilon) = Y_\epsilon$, where $Y_\epsilon$ encapsulates all possible logit vectors for each class given the perturbed inputs.
This output reach set $Y_\epsilon$ is then projected into a one-dimensional representation $y_\epsilon$. This projection consists of a set of intervals, one for each class, indicating the lower ($l$) and upper ($u$) bounds of the predicted logits for that class.
For instance, $y_\epsilon = [y_1, y_2, ..., y_m]$ where $y_1 = [l_1, u_1], y_2 = [l_2, u_2], ..., y_{m} = [l_{m}, u_{m}]$. This projection is illustrated in the rightmost part of Figure~\ref{fig:overview}, where one example might depict a resilient sample (clear separation between the true class interval and others) and another a non-robust sample (overlapping intervals).
The reachability solver then formally evaluates whether the specification $\Phi$ is satisfied by this projected output $y_\epsilon$. In this work, $\Phi$ is primarily concerned with a notion of $\epsilon$\textit{-Robustness}, which we first define in its standard form before introducing our specific refinement for targeted explanations.

\begin{definition}[\textbf{Standard $\epsilon$\textit{-Robust Explanation}}]\label{def:standard_robustness_def}
Consider an input $x \in \mathbb{R}^{w\times h\times c}$ with true class $y \in \{1,2,..., m\}$. For a perturbation size $\epsilon > 0$ and a norm $p \in \{1, 2, \infty\}$, we define the input perturbation set $X_\epsilon = \{x' \in \mathbb{R}^{w\times h\times c} : \|x' - x\|_p \leq \epsilon\}$, where $\|\cdot\|_p$ denotes the $\mathcal{L}_{p}$ norm.
A model is \textbf{$\epsilon$-locally robust} at input $x$ if, for all perturbed inputs in $X_\epsilon$, the predicted class remains $y$. In terms of the projected reach set $Y_\epsilon$, this means the interval corresponding to the true class $\mathbf{y}$ does not overlap with the interval of any other class $\mathbf{k}$ \cite{tran2019star}. More formally, the lower bound of the logits for the true class, $l_\mathbf{y}$, must be strictly greater than the upper bound of the logits for any other class $u_\mathbf{k}$, as described in Equation~\ref{eq:robustness}. Here, $\mathbf{y}$ is the true class, and $\mathbf{k}$ represents any other class. This standard definition provides a baseline measure of the model's general resilience, quantifying its ability to maintain a decision against untargeted disturbances.
\end{definition}%
\begin{equation}
    \forall \mathbf{k} \in \{1,2,\ldots, m\} \wedge \mathbf{k} \neq \mathbf{y}, l_\mathbf{y} > u_\mathbf{k}
    \label{eq:robustness}
\end{equation}%

\begin{definition}[\textbf{Targeted $\epsilon$\textit{-Robust Explanation}}]\label{def:targeted_robustness_def}
Building upon this general notion of resilience, we now define \textbf{Targeted $\epsilon$-Robustness}, the formal property central to \vitax{}. An explanation satisfying this property is defined as a \textit{Targeted $\epsilon$-Robust Explanation}. Consider an input $x$ (true class $\mathbf{y}$) and a specific \emph{target class} $\mathbf{t} \neq \mathbf{y}$. This property verifies the model's semifactual robustness: whether features can be perturbed without flipping the classification. The ``targeting'' refers to feature selection: subset $A$ is chosen (via the heuristic described in Section \ref{sec:computingvitax}) because its features are most sensitive to potential transitions towards class $\mathbf{t}$.

The model is considered \textbf{targeted $\epsilon$-robust} concerning class $\mathbf{t}$ at input $x$ with respect to a feature subset $A \subseteq \{1,2,\ldots, n\}$ (where $n = w \times h \times c$ is the total number of features), if, when only these features in $A$ are perturbed within $X_{\epsilon, A}$ (i.e., $\| x'_A - x_A \|_p \leq \epsilon$ for features in $A$, and $x'_F = x_F$ for features $F \notin A$), the model's prediction for the original class $\mathbf{y}$ remains distinguishable from the target class $\mathbf{t}$.
Specifically, the lower bound of the projected logits for the true class $\mathbf{y}$ (considering perturbations only on $A$), denoted $l_{\mathbf{y}, A}$, must be strictly greater than the upper bound of the projected logits for the targeted class $\mathbf{t}$, $u_{\mathbf{t}, A}$. This property, defined in Equation~\ref{eq:targeted_robustness}, signifies that perturbing the feature set $A$ alone (up to $\epsilon$) in any direction within the $\epsilon$-ball is formally verified to be \emph{insufficient} to make the target class $\mathbf{t}$ more likely than or indistinguishable from the true class $\mathbf{y}$. This formal guarantee constitutes a verified semifactual statement: ``Even if features in $A$ are perturbed by $\epsilon$ (in any direction), the classification $\mathbf{y}$ persists and remains distinguishable from $\mathbf{t}$.''

\begin{equation}
\begin{aligned}
&\exists \, A \subseteq \{1, 2, \ldots, n\},  \mathbf{t} \neq \mathbf{y}, \quad \text{s.t.} \quad \| x'_A - x_A \|_p \leq \epsilon
\wedge  l_{\mathbf{y}, A} > u_{\mathbf{t}, A}.
\end{aligned}
\label{eq:targeted_robustness}
\end{equation}

The \textit{targeted} refers to how $A$ is selected (using sensitivity toward $\mathbf{t}$ as a heuristic, Eq. 3), not to the direction of perturbations in the verification. The formal guarantee holds for perturbations in any direction within the $\epsilon$-ball on features $A$. And an optional dominance constraint ($\forall k \neq y, t: u_{t,A} > u_{k,A}$) 
can ensure $\mathbf{t}$ is the most prominent alternative, though our evaluation 
uses only the core condition above (Section \ref{sec:computingvitax} and Appendix \ref{appendix:not_robust_to_others}).

\end{definition}

\begin{figure*}[t!]
    \centering
    \includegraphics[trim={0 0.55cm 0 0.4cm}, clip,width=0.85\textwidth]{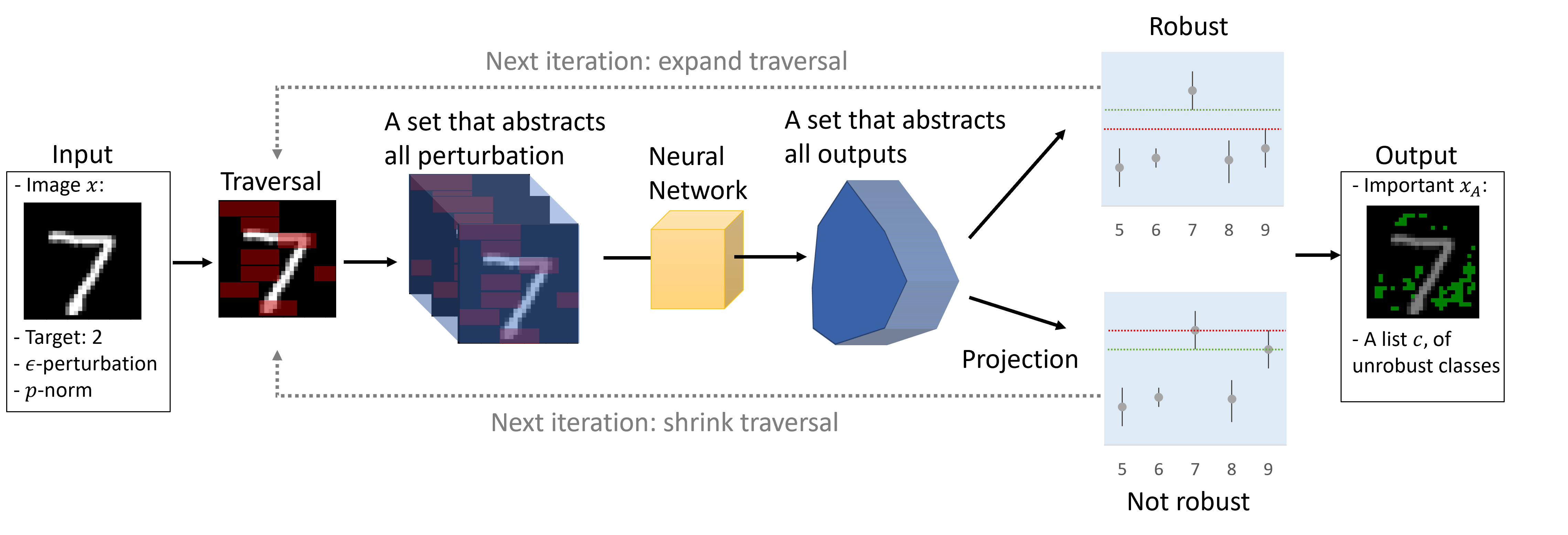}
    \caption{\vitax{} takes an initial input image of `7', a target class \( \mathbf{t} \) of `2', a perturbation \(\epsilon\), and a norm \( p \). Using a heuristic ranking, \vitax{} identifies a subset of input pixels to be $\epsilon$-perturbed, generating a subset of important input features.}
    \label{fig:overview}
\end{figure*}
\section{\vitax{}}\label{sec:method}

To quantify a model's decision resilience efficiently, we present a novel framework called \vitax{} that builds upon $\epsilon$-perturbation-based approaches, as shown in Figure~\ref{fig:overview}. Similar to solving a cardinality minimal problem \cite{ignatiev2019abduction,gainer2017minimal,ignatiev2021sat}, we incorporate a heuristic ranking function to perform a binary search based on sensitivity, solving the problem in \(\log_2(n)\) time. Our approach differs in three key aspects: 
(1) Instead of evaluating features individually and iteratively, we focus on understanding the correlations between input features;
(2) we use $\epsilon$-ball perturbations and the sensitivity of the output with respect to the input to find targeted explanations;
(3) we emphasize using a star-based reachability solver to generate explanations that are most \textit{sensitive} to the model. The \vitax{} framework is described in more detail below.

\newcommand{\createnodes}[3]{
    \foreach \row in {1,...,#2} {
        \foreach \col in {1,...,#3} {
            \pgfmathtruncatemacro{\num}{(\row - 1) * #3 + \col}
            \node at (\col*0.5 - 0.25, #2*0.5 + 0.25 - \row*0.5) {\( #1_{\scriptstyle \num} \)};
        }
    }
}
\definecolor{lightgreen}{rgb}{0.56, 0.93, 0.56}
\definecolor{lightgrey}{rgb}{0.83, 0.83, 0.83}

\begin{wrapfigure}{R}{.5\textwidth} 
    \centering
    \tiny
    \begin{tabular}{ccc}
        \parbox{2.2cm}{\centering \textbf{Step 1:} Test $\pi[1..16]$ \\ (Not Robust)
            \begin{tikzpicture}[scale=.90]
                \fill[lightgreen] (0,0) rectangle (2,2);
                \draw[step=0.5cm,color=gray] (0,0) grid (2,2);
                \createnodes{x}{4}{4}
            \end{tikzpicture}
        } &
        \parbox{2.2cm}{\centering \textbf{Step 2:} Test $\pi[1..8]$ \\ (Not Robust)
            \begin{tikzpicture}[scale=.90]
                \fill[lightgreen] (0,1) rectangle (2,2);
                \draw[step=0.5cm,color=gray] (0,0) grid (2,2);
                \createnodes{x}{4}{4}
            \end{tikzpicture}
        } &
        \parbox{2.2cm}{\centering \textbf{Step 3:} Test $\pi[1..4]$ \\ (Robust)
            \begin{tikzpicture}[scale=.90]
                \fill[lightgreen] (0,1.5) rectangle (2, 2);
                \fill[lightgrey] (0,0) rectangle (2, 1);
                \draw[step=0.5cm,color=gray] (0,0) grid (2,2);
                \createnodes{x}{4}{4}
            \end{tikzpicture}
        } \\
        \\ 
        \parbox{2.2cm}{\centering \textbf{Step 4:} Test $\pi[1..6]$ \\ (Not Robust)
            \begin{tikzpicture}[scale=.90]
                \fill[lightgreen] (0,1) rectangle (1, 1.5);
                \fill[lightgreen] (0,1.5) rectangle (2, 2);
                \fill[lightgrey] (0,0) rectangle (2, 1);
                \draw[step=0.5cm,color=gray] (0,0) grid (2,2);
                \createnodes{x}{4}{4}
            \end{tikzpicture}
        } &
        \parbox{2.2cm}{\centering \textbf{Step 5:} Test $\pi[1..5]$ \\ (Robust)
            \begin{tikzpicture}[scale=.90]
                \fill[lightgreen] (0,1) rectangle (.5, 1.5);
                \fill[lightgreen] (0,1.5) rectangle (2, 2);
                \fill[lightgrey] (1,1) rectangle (2, 1.5);
                \fill[lightgrey] (0,0) rectangle (2, 1);
                \draw[step=0.5cm,color=gray] (0,0) grid (2,2);
                \createnodes{x}{4}{4}
            \end{tikzpicture}
        } &
        \parbox{2.2cm}{\centering \textbf{Step 6:} Final Result \\ ($A=\pi[1..5]$)
            \begin{tikzpicture}[scale=.90]
                \fill[lightgreen] (0,1) rectangle (.5, 1.5);
                \fill[lightgreen] (0,1.5) rectangle (2, 2);
                \fill[lightgrey] (.5,1) rectangle (2, 1.5);
                \fill[lightgrey] (0,0) rectangle (2, 1);
                \draw[step=0.5cm,color=gray] (0,0) grid (2,2);
                \createnodes{x}{4}{4}
            \end{tikzpicture}
        }
    \end{tabular}
    \caption{Computing the explanation towards one particular class using reachability solver for a simple input $x = [1, .. 16]$. Green is the current set of features that solver is evaluating. Gray denotes the irrelevant feature and white is undetermined features.  }
    \label{fig:grid_walkthrough}
\end{wrapfigure}
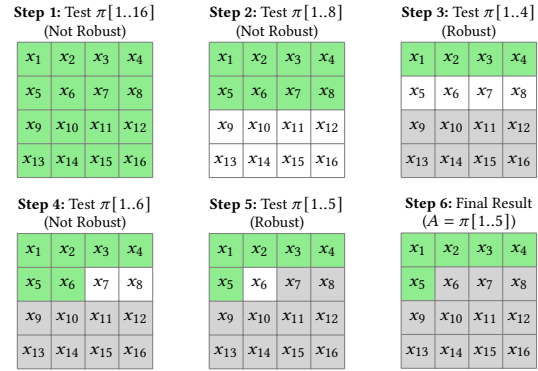

\subsection{A Running Example}\label{sec:running_example}

Before presenting the algorithm in detail, we first illustrate it via an example. Suppose $x$ is the 16-features input $[x_1, .. x_{16}]$ as shown in Figure~\ref{fig:grid_walkthrough}, having a classification model $f$, a perturbation magnitude $\epsilon$, and $p = \infty$. For simplicity, in this example, the features are ordered $x_1$ to $x_{16}$ from most to least sensitive. Initially, we might evaluate the entire set of features for \textit{robustness} with respect to a \textit{target class} $\mathbf{t}$ under $\epsilon$-perturbation. (In practice, Algorithm \ref{alg:method} uses a binary search and doesn't necessarily start with the whole set if a smaller set is already assumed non-robust under large $\epsilon$, or robust if $\epsilon$ is very small). Concretely, when a candidate subset of features $d$ (e.g., $x_1$ to $x_{16}$) is evaluated, \vitax{} composes a pre-condition $X_{\epsilon, d}$ representing the $\epsilon$-perturbation applied only to these features. This $X_{\epsilon, d}$ is passed to a reachability solver to compute the output reach set $Y_\epsilon$. A verifier then checks if this reach set implies the post-condition of \textit{Targeted $\epsilon$-Robustness} (Definition \ref{def:targeted_robustness_def}), as per Lines \ref{alg:method:spec2} - \ref{alg:method:spec2.2} of Algorithm \ref{alg:method}.

Assume the verifier returns that the full set ($x_1$ to $x_{16}$) is not robust. Following the binary search logic, \vitax{} would then test a smaller prefix, say the first half: $x_1$ to $x_8$. The features $x_9$ to $x_{16}$ are temporarily held as outside this current test prefix (as shown on the top middle of Figure \ref{fig:grid_walkthrough}).

In the third grid of Figure \ref{fig:grid_walkthrough}, let's assume the solver then returns that this prefix of eight features ($x_1$ to $x_8$) is also not robust. This indicates that if a robust prefix exists among the most sensitive features, it must be shorter than eight features. Therefore, the binary search adjusts its upper bound, effectively excluding features $x_9$ to $x_{16}$ (now colored grey) from consideration as part of the maximal robust prefix because the current search focuses on prefixes of length less than eight. The algorithm would then proceed to evaluate an even smaller prefix, for instance, $x_1$ to $x_4$, while features $x_5$ to $x_8$ are held as unspecified for this iteration.

If, subsequently, the subset $x_1$ to $x_4$ is found to be robust (resilience holds), the algorithm would then attempt to find a larger robust prefix by exploring features between $x_4$ and $x_8$ (e.g., testing $x_1$ to $x_6$, then $x_1$ to $x_5$). This iterative process continues.

Eventually, \vitax{} identifies a subset $A = [x_1, x_2, x_3, x_4, x_5]$ as the largest prefix of the ranked features for which Targeted $\epsilon$-Robustness holds (bottom right grid of Figure \ref{fig:grid_walkthrough}). This means that for this set $A$, perturbing its features by $\epsilon$ (in the direction of sensitivity towards $\mathbf{t}$) formally guarantees that the original class remains distinguishable from the target class $\mathbf{t}$ (as per Definition \ref{def:targeted_robustness_def}), and any minimal expansion of this set $A$ (by including the next most sensitive feature, $x_6$ in this example) would cause this guarantee to fail.

\subsection{Computing \vitax{} Explanations}\label{sec:computingvitax}

\vitax{} consists of three main stages: 1) \textbf{heuristic ranking}, 2) \textbf{feature search}, and 3) \textbf{final assessment} (as illustrated in Section~\ref{sec:running_example}). We describe each of these components in more detail below, followed by Algorithm \ref{alg:method}.

\usetikzlibrary{decorations.pathreplacing,calc}
\newcommand*{\AddNote}[4]{%
    \begin{tikzpicture}[overlay, remember picture]
        \draw [decoration={brace,amplitude=0.5em},decorate,thin, black]
            ($(#3)!(#1.north)!($(#3)-(0,3)$)+(+3.5,0)$) --  
            ($(#3)!(#2.south)!($(#3)-(0,3)$)+(+3.5,0)$)
                node [align=center, text width=2.5cm, pos=0.5, anchor=west, xshift=-0.2cm] {#4};
    \end{tikzpicture}
}%
\newcommand{\tikzmark}[1]{\tikz[overlay,remember picture] \node (#1) {};}

\textbf{(1) Heuristic Ranking (Saliency):} To effectively analyze a subset of features, we first introduce a heuristic ranking method based on feature sensitivity to produce a traversal order for the reachability solver, inspired by the occlusion and saliency methods \cite{zeiler_visualizing_2013,simonyan_deep_2014,wu_verix_2023,ghorbani2019interpretation}. Sensitivity (Saliency) is quantified by the partial derivatives of the network's output with respect to each input feature for a target class. For a neural network \(f\) and input vector \(x = [x_1,x_2 \ldots, x_n]\), the sensitivity for the \(i\)-th feature \(x_i\) is defined as:
\begin{equation}
    \Call{HeuristicRanking}{x, \mathbf{t}} = \Call{S}{x, \mathbf{t}} = \nabla_x f_{\mathbf{t}}(x)
\end{equation}

\begin{wrapfigure}{R}{0.5\textwidth}
\begin{minipage}{0.5\textwidth}
\begin{algorithm}[H]
\caption{\vitax{} (Reachability eXplanation)}
\label{alg:method}
\begin{algorithmic}[1]

\Statex\hspace{-1.5em}\textbf{Input:} neural network $f$, input $x$
\Statex \hspace{-1.5em}\textbf{Parameters:} target class $\mathbf{t}$, norm $p$, perturbation $\epsilon$
\Statex \hspace{-1.5em}\textbf{Output:} important feature $x_A$, unsatisfied class $c$

\Function{ViTaX}{} \tikzmark{right}
    \tikzmark{top1}
    \State $\pi \mapsfrom \Call{HeuristicRanking}{x, \mathbf{t}}$  \label{alg:method:3}  \tikzmark{bot1}
    \State $I \mapsfrom 0$ \label{alg:method:1} \tikzmark{top2a}
    \State $J \mapsfrom n$ \label{alg:method:2}
    \State $A, A_{candidate}  \mapsfrom \emptyset$ \label{alg:method:4} \tikzmark{bottom2a}
    \While{$I \leq J$} 
        \State $u \mapsfrom (I + J) / 2$ \label{alg:method:7} \tikzmark{top2b}
        \State $d \mapsfrom \pi[0:u]$ \label{alg:method:8} 
        \State $X^p_{\epsilon,d} \mapsfrom \| x' - x \|^d_p \leq \epsilon $  \label{alg:method:spec1} \tikzmark{bottom2b}
        \State $\Phi \mapsfrom  l_{\mathbf{y}, d} > u_{\mathbf{t}, d}$ \label{alg:method:spec2}  \tikzmark{top2c}
        \For{$\mathbf{k} \in \{1, \dots, m\} \setminus \{\mathbf{y}, \mathbf{t}\}$} \label{alg:method:spec2.1}
        \State $ \Phi \mapsfrom \Phi \wedge u_{\mathbf{t}, d} > u_{\mathbf{k}, d}$ \label{alg:method:spec2.2} 
      \EndFor                                                                                  
        \State $(\text{FLAG}, c) \mapsfrom$ \Call{$\mathcal{V}$}{$f$, $X_{\epsilon,d}, \Phi$}  \tikzmark{bottom2c}
        \If{\text{FLAG}} \tikzmark{top2d} 
            \State $I \mapsfrom u +1$ \label{alg:method:igetsu}
                \State $A_{candidate} \mapsfrom \pi[0: u]$
        \Else
            \State $J \mapsfrom u - 1$ \label{alg:method:jgetsu}
        \EndIf  \tikzmark{bottom2d}
    \EndWhile 
    \State $A \mapsfrom A_{candidate} $ \label{alg:method:atoPi} \tikzmark{top3}
    \State \Return $x_A, c$ \tikzmark{bottom3}
\EndFunction
\end{algorithmic}
\AddNote{top1}{bot1}{right}{\textbf{(1)}}
\AddNote{top2a}{bottom2a}{right}{\textbf{(\textit{2a})}}
\AddNote{top2b}{bottom2b}{right}{\textbf{(\textit{2b})}}
\AddNote{top2c}{bottom2c}{right}{\textbf{(\textit{2c})}}
\AddNote{top2d}{bottom2d}{right}{\textbf{(\textit{2d})}}
\AddNote{top3}{bottom3}{right}{\textbf{(3)}}
\end{algorithm}
\end{minipage}
\end{wrapfigure}

The gradient of the output with respect to each feature measures the influence of each input component on the model's decision. A high magnitude indicates a significant impact, while the sign indicates the change direction. We utilize the sensitivity as a ranking $\pi \in \mathbb{R}^n$ of all features based on the absolute magnitude of their sensitivities. We then rank the feature indices from most sensitive to least sensitive. In the running example (Section~\ref{sec:running_example}), this corresponds to the assumed ordering $x_1$ to $x_{16}$ from most to least sensitive. For instance, consider an input $x = [x_1, x_2, x_3]$. We evaluate the sensitivity of each feature to create the heuristic ranking $\pi = [x_3, x_1, x_2]$ s.t. $|S(x_3)| \geq |S(x_1)| \geq |S(x_2)|$. This heuristic is motivated by the intuition that features exhibiting a high gradient magnitude towards the target class $\mathbf{t}$ are most likely to influence a potential transition towards $\mathbf{t}$. Prioritizing such features provides a relevant ordering for efficiently identifying a minimal critical set $A$ that satisfies our Targeted $\epsilon$-Robustness condition.

\textbf{(2) Feature Search:} To identify the minimal subset of the most sensitive features (as ranked by $\pi$) that are critical for the transition towards target $\mathbf{t}$, we aim to determine the largest prefix $d$ of $\pi$ (i.e., $\pi[0:u]$) for which Targeted $\epsilon$-Robustness (Definition~\ref{def:targeted_robustness_def}) still holds. 
Consider a subset of features $x_d$ corresponding to indices $d \subset \{1,2, \ldots, n\}$. An $\epsilon$-perturbation applied only to this subset is represented as $X_{\epsilon, d} = \{x' : \|x' - x\|^d_p \leq \epsilon \text{ and } x'_k=x_k \text{ for } k \notin d\}$. Our objective is to find the largest set $d$ (largest $u$) such that $\mathcal{V}(f, X_{\epsilon,d}, \Phi)$ returns true for the specification $\Phi$ derived from Definition~\ref{def:targeted_robustness_def}.

To efficiently determine the largest prefix $d$ satisfying our formal condition, we employ a binary search strategy over the heuristically ranked features $\pi$. This approach is motivated by the need to minimize queries to the computationally intensive formal solver $\mathcal{V}$ (requiring only a logarithmic number of calls) while still guaranteeing the identification of the maximal prefix that meets the specification $\Phi$. The ``loss of robustness'' occurs when attempting to include an additional feature from $\pi$ causes this condition to fail. The search process, detailed as steps (2a)-(2d) in Algorithm~\ref{alg:method}, systematically finds this boundary. Each Step in Figure~\ref{fig:grid_walkthrough} corresponds to one full iteration of this loop (2b$\to$2c$\to$2d): Step~1 tests $\pi[1..16]$, Step~2 tests $\pi[1..8]$, Step~3 tests $\pi[1..4]$, Step~4 tests $\pi[1..6]$, Step~5 tests $\pi[1..5]$, and Step~6 shows the final result.

\textit{\textbf{(2a)} Initialization (Lines \ref{alg:method:1}-\ref{alg:method:4}):} Set search pointers $I=0$ (start of $\pi$) and $J=n$ (end of $\pi$). $A$ will store the resulting feature subset.

\textit{\textbf{(2b)} Feature Selection (Lines \ref{alg:method:7}-\ref{alg:method:spec1}):}  In each iteration of the binary search, compute the midpoint $u$ of the current interval $[I, J]$. Select the top $u$ features from $\pi$ as the current candidate subset $d = \pi[0:u]$. Define the perturbation set $X^p_{\epsilon,d}$ for these features.

\textit{\textbf{(2c)} Verification (Lines \ref{alg:method:spec2}-\ref{alg:method:spec2.2} and following solver call):}
Construct the formal specification $\Phi$. The primary condition, $\Phi_{target}$, checks if $l_{\mathbf{y}, d} > u_{\mathbf{t}, d}$ (from Equation~\ref{eq:targeted_robustness}), ensuring the original class $\mathbf{y}$ remains distinct from the target class $\mathbf{t}$ when only features in $d$ are perturbed. Moreover, an optional dominance constraint $\Phi_{others}$ verifies that $u_{t,d} > u_{k,d}$ for all other classes $k \neq y, t$. This ensures the explanation unambiguously targets $\mathbf{t}$ rather than another alternative $\mathbf{k}$. However, this constraint can be restrictive when multiple classes are closely clustered in logit space. Our evaluation uses only $\Phi_{target}$, providing valid semifactual explanations for robustness against $\mathbf{t}$ without requiring $\mathbf{t}$ to dominate all alternatives. Cases where both constraints apply are analyzed in Appendix~\ref{appendix:not_robust_to_others}. The specification $\Phi$ is then verified using the solver $\mathcal{V}(f,X_{\epsilon,d},\Phi)$.

\textit{\textbf{(2d)} Update (Following \textbf{If} statement):}
If FLAG is false (meaning $\Phi$ is not satisfied for subset $d$), or specifically if $\Phi_{target}$ is not met, it implies that perturbing the current subset $d$ breaks the model's resilience against $\mathbf{t}$ (or another class $\mathbf{k}$ became more prominent than $\mathbf{t}$). Thus, $d$ is too large or contains features that cause this loss of targeted robustness. The search space is narrowed by setting $J \mapsfrom u - 1$.
If FLAG is true, the subset $d$ satisfies the Targeted $\epsilon$-Robustness (and the conditions for other classes $\mathbf{k}$).  We then try to find an even larger robust subset by searching in the right half: $I \mapsfrom u + 1$. In Figure~\ref{fig:grid_walkthrough}, Step~3 (robust at $\pi[1..4]$) triggers expansion to Step~4, and Step~5 (robust at $\pi[1..5]$) confirms the final boundary.

\textbf{(3) Final Assessment (Line \ref{alg:method:atoPi} and Return):} The binary search continues until $I > J$. The set $A_{candidate}$ from the last successful verification (largest $u$ for which $\Phi_{target}$ held) is the desired set of important features $A$. This corresponds to Step~6 of Figure~\ref{fig:grid_walkthrough}, where $A = \pi[1..5]$ is the final result. If no such $A_{candidate}$ was found (e.g., even perturbing the single most sensitive feature causes loss of Targeted $\epsilon$-Robustness), $A$ would be empty. The algorithm returns $x_A$ (the features corresponding to indices in $A$) and $c_{candidate}$ (classes $\mathbf{k}$ that were not robustly less likely than $\mathbf{t}$ for the final $A$). The heatmap for explanation is generated from $x_A$, showing features whose $\epsilon$-perturbation (in the direction of sensitivity) maintains the condition in Definition~\ref{def:targeted_robustness_def}. If the model is inherently non-robust or highly sensitive near input $x$, \vitax{} will correctly return a very small or empty feature set $A$. Conversely, if the model is robust at input $x$ for the given target class $\mathbf{t}$, the algorithm will find a large set $A$, indicating that many features can be simultaneously perturbed by $\epsilon$ without flipping the classification to $\mathbf{t}$. Therefore, an empty $A$ arises exclusively from non-robustness and serves as a diagnostic signal that the model's decision boundary is fragile at that input. In such cases, reducing $\epsilon$ may yield a non-empty explanation, as a smaller perturbation magnitude can better find the robustness guarantee for the most sensitive features. Because the heuristic concentrates the most critical features at the top of the ranking, a smaller maximal robust prefix (lower cardinality) indicates that the vulnerability is tightly isolated, yielding a highly focused, minimal explanation of the decision boundary.

\subsection{Formal Guarantees of \vitax{} Explanations}
\label{sec:Soundness}
In formal XAI, \textit{soundness} implies that an explanation accurately reflects the model's reasoning, while \textit{completeness} means it captures all necessary aspects of that reasoning \cite{haufe2024explainable, key2025sat, marques-silva_delivering_2022, paul2024formal}. 
For \vitax{}, the explanation $x_A$ is sound if it genuinely represents a feature set satisfying our definition of Targeted $\epsilon$-Robustness.
We include the definitions, rigorous proofs, and empirical evaluations for Lemma~\ref{lemma:feature_inclusive_property} and Remark~\ref{theorem:subset_feature_inclusive_property} and Theorem~\ref{theorem:soundness} in Appendix \ref{appendix:lemma:feature_inclusive_property}, \ref{appendix:remark:subset_feature_inclusive_property}, \ref{appendix:theorem:soundness}, respectively. To simplify notation in these proofs, the input is treated as a flattened vector of length $n = w \times h \times c$.

\begin{lemma}[Feature Inclusion Property] \label{lemma:feature_inclusive_property}
Given a model \( f: \mathbb{R}^n \rightarrow \mathbb{R}^m \), for \( y = f(x) \) and \( i \in \{1,2, \ldots, n\} \), define the set \( X_{\epsilon, i} \) as:
\begin{equation}
    X_{\epsilon, i} = \{x' \in \mathbb{R}^n : \|x' - x\|^i \leq \epsilon\} \text{ (perturbation on } i^{th} \text{ feature only)}
\end{equation}
Then, the reachability property follows:
\begin{equation}
Y_\epsilon = f(X_\epsilon) \supseteq \bigcup_{i=1}^n f(X_{\epsilon, i})
\end{equation}
The output range resulting from perturbing all features simultaneously up to $\epsilon$ is a superset of the union of the output ranges resulting from perturbing each feature $i$ individually.
\end{lemma}

\begin{remark}[Monotonicity of Reach Sets under Feature Perturbation] \label{theorem:subset_feature_inclusive_property}
    Given a model \( f: \mathbb{R}^n \rightarrow \mathbb{R}^m \) and the sets defined as \( X_{\epsilon} = \{ x' \in \mathbb{R}^n : \| x' - x \|_p \leq \epsilon \} \) and \( X_{\epsilon, i} = \{ x' \in \mathbb{R}^n : \|x'_i - x_i\|_p  \leq \epsilon \} \) for \( i \in \{1,2,  \ldots, n\} \). Suppose we have two subsets of feature indices \( d \) and \( l \) from \( \{1,2,  \ldots, n\} \) such that \( |d| < |l| \) and \( d \subset l \). Define the output reach sets when perturbing only features in $d$ and $l$ respectively:
\begin{equation}
Y_{\epsilon, d} = f(X_{\epsilon, d}) \text{  and  } Y_{\epsilon, l} = f(X_{\epsilon, l}).
\end{equation}
Then, the following holds:
\begin{equation}
Y_{\epsilon, d} \subseteq Y_{\epsilon, l} 
\end{equation}
\end{remark}

\begin{theorem}[Soundness of \vitax{} Explanation] \label{theorem:soundness}
If the reachability solver, $\mathcal{V}$, used as a sub-routine in Algorithm~\ref{alg:method} is sound and complete (i.e., it correctly determines if the property $\Phi(d)$ holds for the input set $X_{\epsilon,d}$ corresponding to feature subset $d$), then the feature subset $A$ identified by \vitax{} (and its corresponding $x_A$) satisfies the following properties:

\textbf{(1) Targeted $\epsilon$-Robustness:} The set $A$ satisfies Definition~\ref{def:targeted_robustness_def} ($\Phi(A)$ holds, meaning $l_{\mathbf{y}, A} > u_{\mathbf{t}, A}$ and $u_{\mathbf{t}, A} > u_{\mathbf{k}, A}$ for other classes $\mathbf{k}$).

\textbf{(2) Maximal Robust Set:} $A$ is the largest prefix of the heuristic ranking $\pi$ for which $\Phi$ holds. Formally, if $A = \pi[0:u]$, then $\Phi(\pi[0:u]) = \text{True}$ and $\Phi(\pi[0:u+1]) = \text{False}$ (or $u = |\pi|$).

\textbf{(3) Heuristic Dependence:} The size $|A|$ depends on the quality of $\pi$. A different ranking $\pi'$ may yield a smaller or larger robust set $A'$.

\vitax{} provides an $\mathcal{O}(\log_2(N))$ solution to the problem of finding a feature set satisfying $\Phi$, where the quality of the solution (minimality of $|A|$) depends on how well $\pi$ captures feature importance for the $\mathbf{y} \to \mathbf{t}$ transition.
\end{theorem}

\begin{proof}[Proof Sketch]
Let $A_{final}$ be the feature subset returned by Algorithm~\ref{alg:method}. The algorithm identifies $A_{final}$ as the largest prefix $\pi[0:u]$ of the heuristically ranked features $\pi$ for which the solver $\mathcal{V}$ returned `FLAG = True' when verifying the property $\Phi(A_{final})$. The property $\Phi(d)$ is defined as $\left( l_{\mathbf{y}, d} > u_{\mathbf{t}, d} \right) \wedge \left( \forall \mathbf{k} \neq \mathbf{y}, \mathbf{k} \neq \mathbf{t} \implies u_{\mathbf{t}, d} > u_{\mathbf{k}, d} \right)$, which directly embodies Definition~\ref{def:targeted_robustness_def} along with the condition for other classes.

Since $\mathcal{V}$ returned `FLAG = True' for $A_{final}$ and $\mathcal{V}$ is assumed to be sound (if $\mathcal{V}$ says $\Phi(d)$ holds, it does), it directly implies that $\Phi(A_{final})$ holds. Thus, $A_{final}$ satisfies Targeted $\epsilon$-Robustness (Property 1).

Furthermore, the binary search procedure is designed to find the maximum $u$ (and thus the largest prefix $A_{final} = \pi[0:u]$) for which $\Phi$ holds. If a larger prefix $\pi[0:u']$ (with $u' > u$) existed that also satisfied $\Phi$, the search logic (specifically, updating the lower bound $I \leftarrow u+1$ and storing $A_{candidate}$) would have aimed to find it. The algorithm terminates when this process converges, ensuring $A_{final}$ is the largest such prefix for which $\mathcal{V}$ (assumed complete: if $\Phi(d)$ holds, $\mathcal{V}$ returns `FLAG = True') confirms $\Phi$ (Property 2).

Property 3 follows from the algorithm's structure: the binary search operates on the ranking $\pi$, and a different ranking would lead to testing different feature subsets in different orders, potentially yielding a different final set $A'$.

A detailed proof is available in Appendix~\ref{appendix:theorem:soundness}.
\end{proof}

\section{Evaluation}\label{sec:experiment}

This section presents a comprehensive empirical validation of \vitax{}. Our evaluation is structured around key research questions designed to assess its performance, scalability, and the unique nature of its formally verified semifactual explanations.

\begin{itemize}[nosep]
    \item \textbf{RQ1 (Performance):} How does \vitax{} compare quantitatively against state-of-the-art heuristic, formal, and alternative semifactual methods in terms of fidelity, cardinality, and robustness?
    \item \textbf{RQ2 (Scalability \& Applicability):} How does \vitax{} perform on larger architectures, and how does it generalize to other tasks (e.g., regression)?
    \item \textbf{RQ3 (Qualitative Insight):} What is the qualitative nature of \vitax{}'s targeted semifactual explanations, and how do key parameters and components affect the results?
\end{itemize}

\subsection{Experimental Setup}\label{subsec:experimental_setup}

\paragraph{Datasets and Architectures}
We utilize several datasets: MNIST \cite{lecun2009mnist}, GTSRB \cite{stallkamp2012man}, EMNIST Letters \cite{cohen2017emnist}, and the TaxiNet regression dataset \cite{julian2020validation}. 
To address concerns regarding scalability, we analyze performance across models of varying complexity (MLP, CNN, ResNet, Inception; details in Appendix~\ref{appendix:models}) and discuss the computational trade-offs inherent in formal verification. For primary evaluations, we use 100 correctly predicted test samples.

\paragraph{Baselines}
We compare \vitax{} against a diverse set of methods. Recognizing that standard implementations of LIME and Anchors methods are not inherently designed for targeted semifactual explanations, we specifically adapted these baselines to align their objectives with the semifactual goal of explaining resilience (Details in Appendix \ref{appendix:baseline_adaptation}). The comparison set includes: 
\begin{itemize}[nosep]
    \item \textbf{Probabilistic Semifactual (Adapted LIME \& Anchors):} We adjusted LIME~\cite{ribeiro_why_2016} and Anchors~\cite{ribeiro_anchors_2018} to generate probabilistic semifactual explanations by identifying minimal feature sets that ensure prediction persistence with high probability. For LIME, we implemented a two-stage approach (local linear model + greedy search for stability); for Anchors, we formalized its greedy optimization to meet a high precision threshold.
    
    \item \textbf{Heuristic Semifactual (TSA \& Prototype):} TSA (Targeted Semifactual Adversarial) uses iterative optimization (e.g., PGD) to find a perturbation $\delta$ that maximizes the target logit while maintaining the original classification via a margin penalty \cite{kenny2023utility, aryal2023even, kenny2021generating, chowdhury2025looking}. Prototype identifies the closest existing data sample that shares the original class but is nearest the target class in logit space, using the pixel-wise difference as the explanation \cite{aryal2023even, fernandez2022explanation}.
    
    \item \textbf{Formal (Untargeted):} VeriX~\cite{wu_verix_2023}, which identifies features sufficient for the original classification against any alternative.

    \item \textbf{Formal (Instance-level):} HNMAS~\cite{leofante2023robust}, which generates robust counterfactual instances for multi-agent systems. As it provides instance-level explanations (entire input modifications) rather than feature subsets, standard feature-level metrics like cardinality and fidelity (Eq.~\ref{eq:fidelity_revised}) are not directly applicable; we include it for qualitative comparison (Figure~\ref{fig:motivation}).

    \item \textbf{Algorithmic Baseline:} Brute-Force (linear search using the \vitax{} solver).
\end{itemize}

\paragraph{Implementation Details}
\vitax{} utilizes the NNV tool \cite{tran2019star,lopez2023nnv} as the subroutine solver $\mathcal{V}$. We utilize both the `Approx-Star' method and the `CP-Star' method \cite{hashemi2025probabilistic}. 
Experiments are conducted on a workstation equipped with an AMD Ryzen Threadripper PRO 7975WX 32-Core CPU and an NVIDIA RTX 6000 Ada Generation GPU.

\paragraph{Metrics}
\textbf{Fidelity} (Eq.~\ref{eq:fidelity_revised}) measures how effectively the explanation influences the model's prediction towards the target class $\mathbf{t}$, while penalizing influence towards other classes $\mathbf{k}$.

\begin{equation}
\text{Fidelity} = \frac{f(x')_\mathbf{t} - f(x)_\mathbf{t}}{f(x)_\mathbf{y}} - \frac{\sum_{\mathbf{k} \neq \mathbf{y}, \mathbf{k} \neq \mathbf{t}} \max\left[0, f(x')_\mathbf{k} - f(x')_\mathbf{t}\right]}{f(x)_\mathbf{y}}
\label{eq:fidelity_revised}
\end{equation}
Here, $x'$ is generated by applying the worst-case adversarial perturbation of magnitude $\epsilon$ to the features in $A$ towards the target class $\mathbf{t}$, as determined by the reachability solver.

\textbf{Cardinality} measures the average number of features in $A$. \textbf{Time} measures the average computational time in seconds.

\textbf{Robustness (Robustness against Noisy Execution - NE \cite{jiang_formalising_2023, jiang2024robust})}. We adapt the procedure from Anchors \cite{ribeiro_anchors_2018} to evaluate the sufficiency of $A$. We randomly perturb the \emph{non-explanatory} regions (features outside $A$) by replacing them with features from target class $t$ samples (500 trials), and measure the fraction of trials for which the model's prediction \textit{does not flip} to $t$. A high robustness score indicates that $A$ alone is sufficient to maintain the decision, demonstrating stability despite adversarial perturbation of irrelevant features. Note the intentional duality here: while our formal guarantee (Definition~\ref{def:targeted_robustness_def}) certifies safety when perturbing $A$, the empirical NE metric tests stability by freezing $A$. Because $A$ captures the features most sensitive to the $\mathbf{y} \to \mathbf{t}$ transition, freezing $A$ anchors the prediction, protecting it against adversarial noise in the background.

\subsection{Comparative Analysis (RQ1)}\label{subsec:RQ1}

To answer RQ1, we evaluate \vitax{} against all baselines on MNIST and GTSRB across fidelity, robustness (NE), cardinality, and computational cost. The quantitative results are summarized in Table~\ref{tab:comprehensive}, and Figure~\ref{fig:comparison_to_table1} provides a qualitative visualization of the explanations generated by each method.

\begin{figure}[t]
    \centering
    \begin{subfigure}[t]{0.13\textwidth} 
        \centering
        \includegraphics[width=\textwidth]{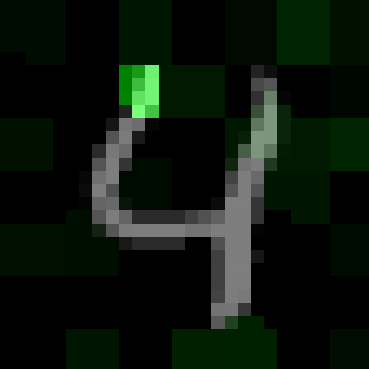}
        \subcaption*{LIME \newline}
    \end{subfigure}
    \begin{subfigure}[t]{0.13\textwidth}
        \centering
        \includegraphics[width=\textwidth]{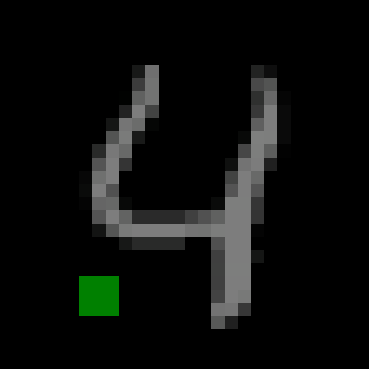}
        \subcaption*{Anchors \newline}
    \end{subfigure}
    \begin{subfigure}[t]{0.13\textwidth}
        \centering
        \includegraphics[width=\textwidth]{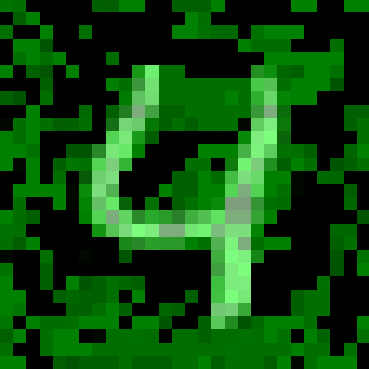}
        \caption*{TSA}
        \label{fig:tsa}
    \end{subfigure}
    \begin{subfigure}[t]{0.13\textwidth}
        \centering
        \includegraphics[width=\textwidth]{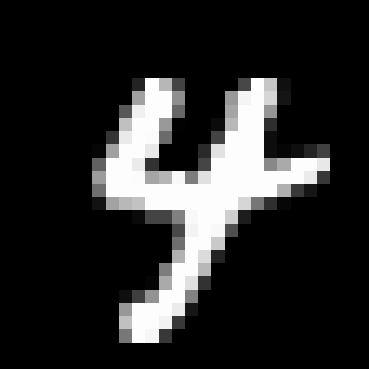}
        \caption*{Prototype}
        \label{fig:prototype}
    \end{subfigure}
    \begin{subfigure}[t]{0.13\textwidth}
        \centering
        \includegraphics[width=\textwidth]{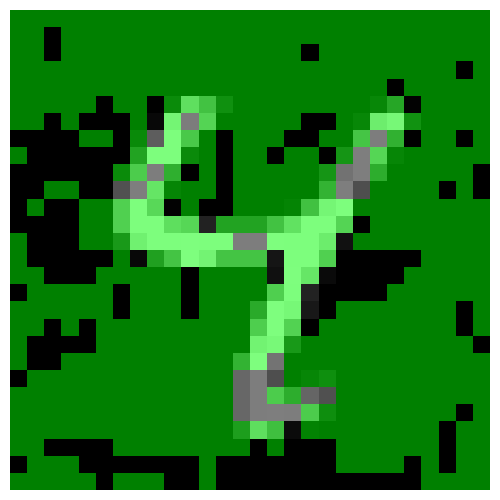}
        \caption*{VeriX \newline}
    \end{subfigure}
    \begin{subfigure}[t]{0.13\textwidth}
        \centering
        \includegraphics[width=\textwidth]{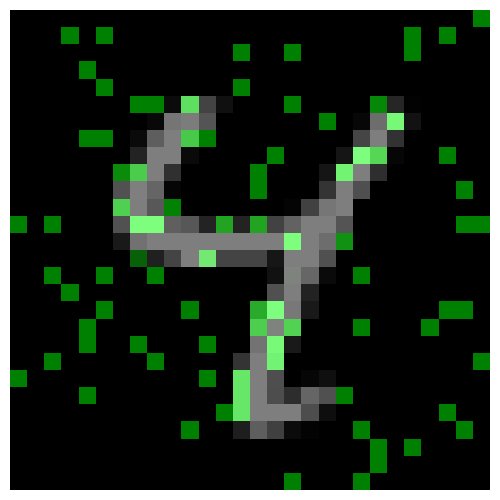}
        \caption*{Brute Force}
        \label{fig:brutal}
    \end{subfigure}
    \begin{subfigure}[t]{0.13\textwidth}
        \centering
        \includegraphics[width=\textwidth]{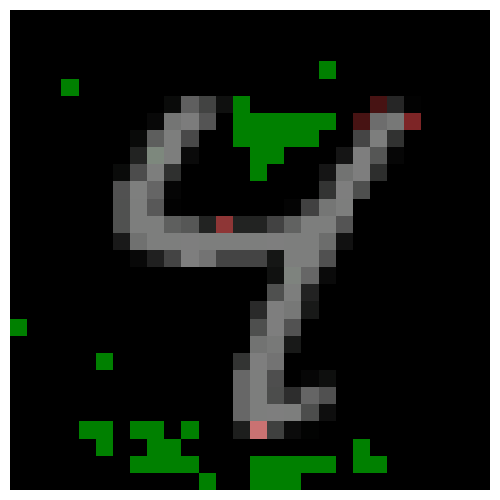}
        \caption*{\textbf{\vitax{} (4 to 9)} \newline}
    \end{subfigure}
    \caption{Qualitative comparison of explanations for an MNIST sample (original `4', target `9' for \vitax{}). All methods use the same base model (MNIST MLP). \vitax{} highlights the critical feature set $A$. Note that this visualization shows the sensitive features (a semifactual insight), not a counterfactual instance that successfully flips the class.} 
    \label{fig:comparison_to_table1}
\end{figure}

\begin{table}[!t]
\caption{Comprehensive Evaluation Results. $\uparrow$ indicates higher is better; $\downarrow$ indicates lower is better. Best results among heuristic/probabilistic methods are underlined; best results among formal methods are bolded.}
\label{tab:comprehensive}
\centering
\small
\begin{tabular}{lcccc|cccc}
\toprule
& \multicolumn{4}{c}{\textbf{MNIST (MLP)}} & \multicolumn{4}{c}{\textbf{GTSRB (CNN)}} \\
\cmidrule(lr){2-5} \cmidrule(lr){6-9}
\textbf{Method} & \textbf{Fid. $\uparrow$} & \textbf{Rob. (NE) $\uparrow$} & \textbf{Time (s) $\downarrow$} & \textbf{Card. $\downarrow$} & \textbf{Fid. $\uparrow$} & \textbf{Rob. (NE) $\uparrow$} & \textbf{Time (s) $\downarrow$} & \textbf{Card. $\downarrow$} \\
\midrule
\multicolumn{9}{l}{\textit{Heuristic/Probabilistic Methods}} \\
LIME (Adapted) & 0.09 & 0.93 & 3.42 & 44.27 & 0.32 & 0.60 & 4.65 & \underline{47.11} \\
Anchors (Adapted) & 0.11 & 0.94 & 4.06 & \underline{41.59} & \underline{0.57} & 0.56 & 8.73 & 75.11 \\
TSA & \underline{0.52} & \underline{0.97} & 0.01 & 514.37 & 0.10 & 0.92 & 0.02 & 621.47 \\
Prototype & 0.15 & 0.93 & \underline{0.00} & 214.44 & 0.33 & \underline{1.00} & \underline{0.00} & 782.59 \\
\midrule
\multicolumn{9}{l}{\textit{Formal Methods}} \\
VeriX & 0.16 & \textbf{0.99} & 699.91 & 472.30 & 0.10 & \textbf{0.93} & 1536.34 & 764.97 \\
\midrule
\textbf{\vitax{} (Ours)} & \textbf{0.56} & 0.97 & \textbf{11.89} & \textbf{121.47} & \textbf{0.78} & 0.70 & \textbf{47.23} & \textbf{30.57} \\
\bottomrule
\end{tabular}
\end{table}

\paragraph{Performance Overview}
\vitax{} demonstrates superior balance of explanation quality and efficiency among formal methods. It achieves the highest fidelity across both datasets (0.56 on MNIST, 0.78 on GTSRB) while simultaneously maintaining the lowest cardinality among formal methods (121.47 on MNIST, 30.57 on GTSRB). This combination is crucial: high fidelity demonstrates that the identified features are highly relevant to the targeted class transition, while low cardinality ensures the explanation is concise and interpretable. \vitax{}'s exceptionally low cardinality on GTSRB (30.57 features) indicates highly focused explanations, demonstrating the effectiveness of binary search with sensitivity-based ranking for identifying minimal critical feature sets.

\paragraph{Comparison with Semifactual Methods}
We compare \vitax{} against several methods adapted or designed for semifactual insights: the probabilistic approaches (LIME, Anchors) and the heuristic search methods (TSA, Prototype). \vitax{} outperforms all these baselines in achieving the optimal balance of fidelity and cardinality.

\vitax{} significantly outperforms the adapted probabilistic methods (LIME, Anchors) in fidelity. On MNIST, \vitax{} achieves 0.56 compared to 0.09 (LIME) and 0.11 (Anchors)—representing over 6× and 5× improvements respectively. On GTSRB, \vitax{} achieves 0.78 versus 0.32 (LIME) and 0.57 (Anchors)—a 37\% improvement over the best probabilistic baseline. This integration of formal verification with targeted heuristic search is far more effective at isolating critical features than methods relying on local surrogate models and sampling.

The heuristic baselines reveal significant challenges in generating effective semifactual explanations without formal guarantees. The optimization-based TSA method achieves competitive fidelity on MNIST (0.52, closest to \vitax{} among heuristics) but fails dramatically on GTSRB (0.10), while producing extremely high cardinality (514-621 features). This architecture-dependent performance suggests that iterative optimization struggles to maintain the semifactual constraint across different model geometries and decision boundaries.

The Prototype method is exceptionally fast but shows inconsistent performance: low fidelity on MNIST (0.15) despite high robustness (0.93), and moderate fidelity on GTSRB (0.33) with perfect robustness (1.00). Its cardinality remains consistently high (214-782 features). This highlights a fundamental limitation: relying on existing data samples depends heavily on dataset density and distribution, and cannot systematically minimize explanations for targeted transitions.

The critical distinction between \vitax{} and all these heuristic/probabilistic methods lies in the nature of the guarantee (summarized conceptually in Table~\ref{tab:sfe_conceptual}). While LIME/Anchors provide probabilistic estimates, and TSA/Prototype rely on localized optimization or distance metrics, \vitax{} provides a formal Targeted $\epsilon$-Robustness guarantee. This formal underpinning ensures trustworthiness by verifying the entire continuous perturbation space, offering mathematical certainty where other methods provide estimations or localized examples.

\begin{table}[h!]
\centering
\caption{Conceptual Comparison of Semifactual Explanation Generation Methods, illustrated with HELOC dataset \cite{fico2019challenge}. \vitax{} uniquely combines targeted insights with formal guarantees and scalability.}
\small
\label{tab:sfe_conceptual}
\resizebox{\textwidth}{!}{%
\begin{tabular}{@{}>{\raggedright\arraybackslash}p{0.12\linewidth}p{0.30\linewidth}>{\raggedright\arraybackslash}p{0.38\linewidth}p{0.15\linewidth}@{}}
\toprule
\textbf{Method} & \textbf{Mechanism}  &\textbf{Example (HELOC)} & \textbf{Nature of Guarantee} \\
\midrule
LIME/Anchors (Adapted) &
Local Surrogate + Probabilistic Sampling. Identifies features that ensure prediction persistence with high probability.  &``Maintaining the current `ExternalRiskEstimate' and `AverageMInFile' is sufficient to keep the `Good Risk' status with 95\% probability.'' &
Probabilistic \& Empirical. \\
\midrule
Prototype / TSA &
Example-based search or iterative optimization. Finds specific instances or perturbations that do not cross the boundary.  &``Even when `NumInqLast6M' was increased by 2, the status remained `Good Risk'.'' &
Heuristic / Best Effort. \\
\midrule
VeriX &
Formal Verification (MILP/SMT). Finds features that guarantee the prediction regardless of other feature values (Sufficiency).  &``It's guaranteed that as long as `ExternalRiskEstimate' < 65, the loan is always `Good Risk'.'' &
Deterministic \& Formal (Untargeted). \\
\bottomrule
 \textbf{\vitax{}} & Formal Reachability + Efficient Search. Proves perturbation of critical features (A) by $\epsilon$ is insufficient to cross the boundary towards a specific target (t).  &For a `Good Risk' applicant, it's guaranteed that if you change `ExternalRiskEstimate' and `AverageMInFile' by up to 10\% in any combination, the result will \textbf{not} change to `Bad Risk'.&Deterministic \& Formal (Targeted). \\
\end{tabular}%
}
\end{table}

\paragraph{Analysis of Trade-offs and Formal Methods}
When considering empirical robustness (NE), VeriX achieves the highest scores (0.99 on MNIST, 0.93 on GTSRB). This is expected, as VeriX aims for general sufficiency—identifying features that maintain the original classification against \emph{any} alternative class, often operating further from specific decision boundaries. In contrast, \vitax{} operates precisely at the \emph{targeted} boundary between the original and the target class.

\vitax{} achieves competitive robustness on MNIST (0.97, matching TSA and nearly matching VeriX's 0.99) while delivering dramatically superior fidelity (0.56 vs 0.16 for VeriX, a 3.5× improvement; 0.78 vs 0.10 for VeriX on GTSRB, a 7.8× improvement). This indicates substantially higher relevance to the targeted transition. The lower robustness on GTSRB (0.70) reflects the inherent trade-off: \vitax{} operates at the targeted boundary to maximize fidelity, accepting reduced general stability for targeted precision. Crucially, \vitax{} achieves this with dramatically faster runtimes (11.89s vs 699.91s on MNIST; 47.23s vs 1536.34s on GTSRB—representing 59× and 33× speedups respectively) and much lower cardinality. This highlights \vitax{}'s effective balance between formal rigor, targeted explanatory power, and computational feasibility.

\paragraph{Qualitative Analysis}
The quantitative advantages are complemented by qualitative differences. For the transition from `4' to `9', \vitax{} precisely highlights the upper loop closure necessary for the transition—the specific region where adding or modifying strokes would create the characteristic upper circle of a `9' (Figure~\ref{fig:comparison_to_table1}). In contrast, baselines like LIME and Anchors identify broader, less focused regions that overlap with general features of `4', while VeriX focuses on features sufficient for maintaining `4' classification generally, rather than those specific to the boundary with `9'. This demonstrates \vitax{}'s core advantage: providing targeted, boundary-specific insights rather than general attribution.

\subsection{Scalability and Applicability (RQ2)}\label{subsec:RQ2}
We analyze the scalability of \vitax{} regarding model architecture complexity and input dimensionality, as well as its applicability beyond classification. A key challenge in formal XAI is the computational cost of the underlying verification solvers when applied to deep architectures or high-dimensional inputs. We evaluate the feasibility of \vitax{} under these conditions and discuss the computational trade-offs.

\paragraph{Scalability to Complex Architectures}
To address concerns about the applicability of \vitax{} to larger models, we extended our evaluation to five complex architectures on the MNIST dataset: a standard CNN, ResNet-Tiny, ResNet-Small, ResNet-Large, and an Inception model. We evaluated performance across two distinct perturbation budgets: a smaller budget ($\epsilon = 25/255 \approx 0.10$) and a significantly larger budget ($\epsilon = 55/255 \approx 0.22$).

Traditional deterministic reachability analysis solvers (like NNV Approx-Star) often struggle with the scale and complexity (e.g., residual connections, depth) of these architectures. The \vitax{} framework is agnostic to the underlying solver $\mathcal{V}$. To manage this complexity, we utilized the CP-Star verification backend, which is better suited for these complex network structures. This demonstrates \vitax{}'s capability to integrate different verification strategies to achieve scalability.

\begin{table}[!t]
    \caption{\vitax{} Performance on Complex Architectures (MNIST) for two perturbation budgets, $\epsilon \approx 0.10$ (25/255) and $\epsilon \approx 0.22$ (55/255), using the CP-Star Verification Backend. (Averages over 10 samples).}
    \label{tab:complex_architectures_cpstar}
    \centering
    \small
    \resizebox{\textwidth}{!}{
    \begin{tabular}{lcrcc cc cc cc}
        \toprule
        \textbf{Model}& \textbf{\# Params} & \textbf{Acc (\%)} & \multicolumn{2}{c}{\textbf{Fid. $\uparrow$}} & \multicolumn{2}{c}{\textbf{Rob. (NE) $\uparrow$}} & \multicolumn{2}{c}{\textbf{Time (s) $\downarrow$}} & \multicolumn{2}{c}{\textbf{Card. $\downarrow$}} \\
        \cmidrule(lr){4-5} \cmidrule(lr){6-7} \cmidrule(lr){8-9} \cmidrule(lr){10-11}
         & & & \textbf{$\epsilon \approx 0.10$} & \textbf{$\epsilon \approx 0.22$} & \textbf{$\epsilon \approx 0.10$} & \textbf{$\epsilon \approx 0.22$} & \textbf{$\epsilon \approx 0.10$} & \textbf{$\epsilon \approx 0.22$} & \textbf{$\epsilon \approx 0.10$} & \textbf{$\epsilon \approx 0.22$} \\
        \midrule
        CNN           & 15,634   & 99.0 & 0.39 & 0.55 & 0.99 & 0.97 & 56.89 & 55.24 & 476.4 & 237.5 \\
        ResNet-Tiny   & 10,738   & 97.1 & 0.30 & 0.37 & 0.77 & 0.82 & 58.52 & 57.65 & 270.3 & 68.6 \\
        ResNet-Small  & 173,850  & 99.4 & 0.27 & 0.39 & 0.97 & 0.92 & 60.80 & 58.73 & 463.3 & 177.2 \\
        ResNet-Large  & 693,802  & 99.1 & 0.29 & 0.34 & 0.94 & 0.86 & 63.74 & 61.83 & 421.1 & 90.5 \\
        Inception     & 490,954  & 84.1 & 0.16 & 0.16 & 0.93 & 0.92 & 66.71 & 66.40 & 124.8 & 55.2 \\
        \bottomrule
    \end{tabular}
    }
\end{table}
The results demonstrate that \vitax{} successfully generates formally grounded explanations for these complex architectures, including deep ResNets and Inception models, within practical timeframes (under 70 seconds), regardless of the $\epsilon$ budget (Table~\ref{tab:complex_architectures_cpstar}). This is a significant demonstration of scalability, highlighting the benefit of \vitax{}'s solver-agnostic design.

\textit{Impact of Model Complexity:} We observe that Fidelity is generally lower across these models (0.16-0.55) compared to the simpler MLP (0.56 in Table~\ref{tab:comprehensive}). This is expected behavior. Deeper architectures often rely on distributed, hierarchical feature representations and tend to have more robust decision boundaries. Consequently, perturbing a feature subset ($A$) at the input level results in a smaller relative shift in logits (lower fidelity). This highlights a known limitation of input-space attribution methods: their explanatory impact inherently diminishes on very deep or highly robust models.

\textit{Impact of Perturbation Budget ($\epsilon$):} The comparison between $\epsilon \approx 0.10$ and $\epsilon \approx 0.22$ confirms the trends observed in RQ3 (Section~\ref{subsec:epsilon_effect}). Increasing $\epsilon$ consistently leads to significantly lower Cardinality (e.g., 476.4 down to 237.5 for CNN), as the explanation must become more focused to maintain the guarantee under larger perturbations. Concurrently, the larger perturbation often leads to higher Fidelity (e.g., 0.39 up to 0.55 for CNN), as the features in A are perturbed more significantly towards the target class.

The Robustness (NE) shows some variation with the larger $\epsilon$ (e.g., 0.99 down to 0.97 for CNN; 0.77 up to 0.82 for ResNet-Tiny). The pattern is less consistent than in the simpler MLP case, reflecting the complex interplay between model architecture, decision geometry, and perturbation budget. For some models, larger $\epsilon$ brings explanations closer to boundaries (reducing NE), while for others, the more focused feature sets (lower cardinality) at higher $\epsilon$ may improve stability. This highlights a vital distinction between \vitax{}'s formal guarantee (Targeted $\epsilon$-Robustness) and empirical stability against untargeted noise. The formal guarantee ensures that perturbing $A$ by $\epsilon$ does not cross the targeted boundary, but proximity to that boundary affects how robust the prediction is to noise in \emph{other} features (as tested by NE). This contrasts with methods seeking general sufficiency (like VeriX, which achieves 0.99 NE on MNIST MLP in Table~\ref{tab:comprehensive}), which operate further from boundaries.

\paragraph{Scalability to Input Dimensions}
To further evaluate scalability regarding input size, we analyzed the impact of increasing image dimensions on computation time. We trained MLP models (using the same hidden layer structure as the baseline MLP) on upscaled versions of the MNIST dataset: 28x28 (784 features), 56x56 (3136 features), and 78x78 (6084 features). We measured the execution time using both the Approx-Star and CP-Star solvers (Figure~\ref{fig:combined_scalabilitiy_input_sizes}).

\begin{figure}[t]
    \centering
    \begin{subfigure}[b]{0.48\linewidth}
        \centering
        \includegraphics[width=\linewidth]{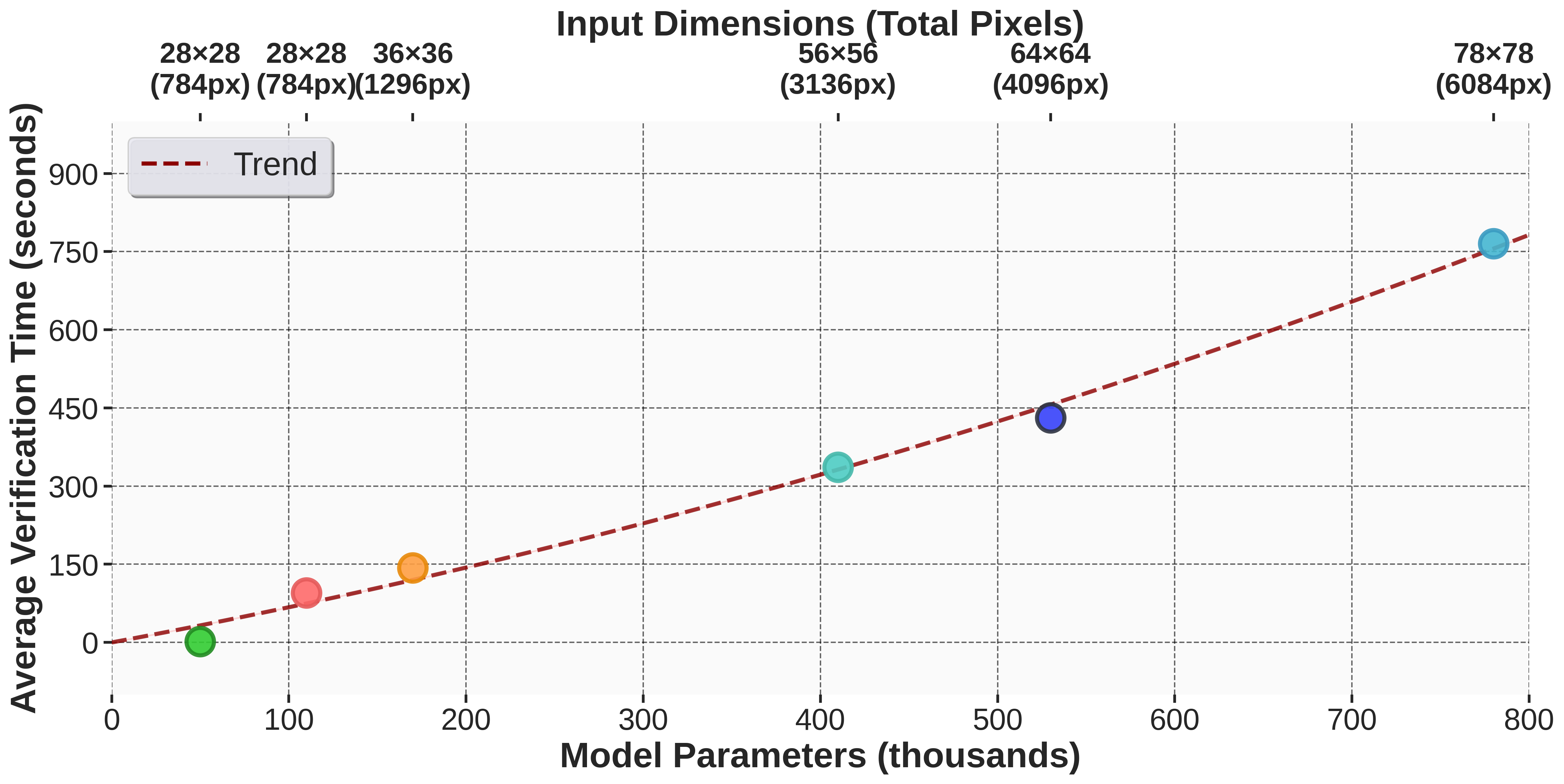}
        \caption{Approx Star}
        \label{fig:scaletoimagesize_approx_star}
    \end{subfigure}
    \hfill 
    \begin{subfigure}[b]{0.48\linewidth}
        \centering
        \includegraphics[width=\linewidth]{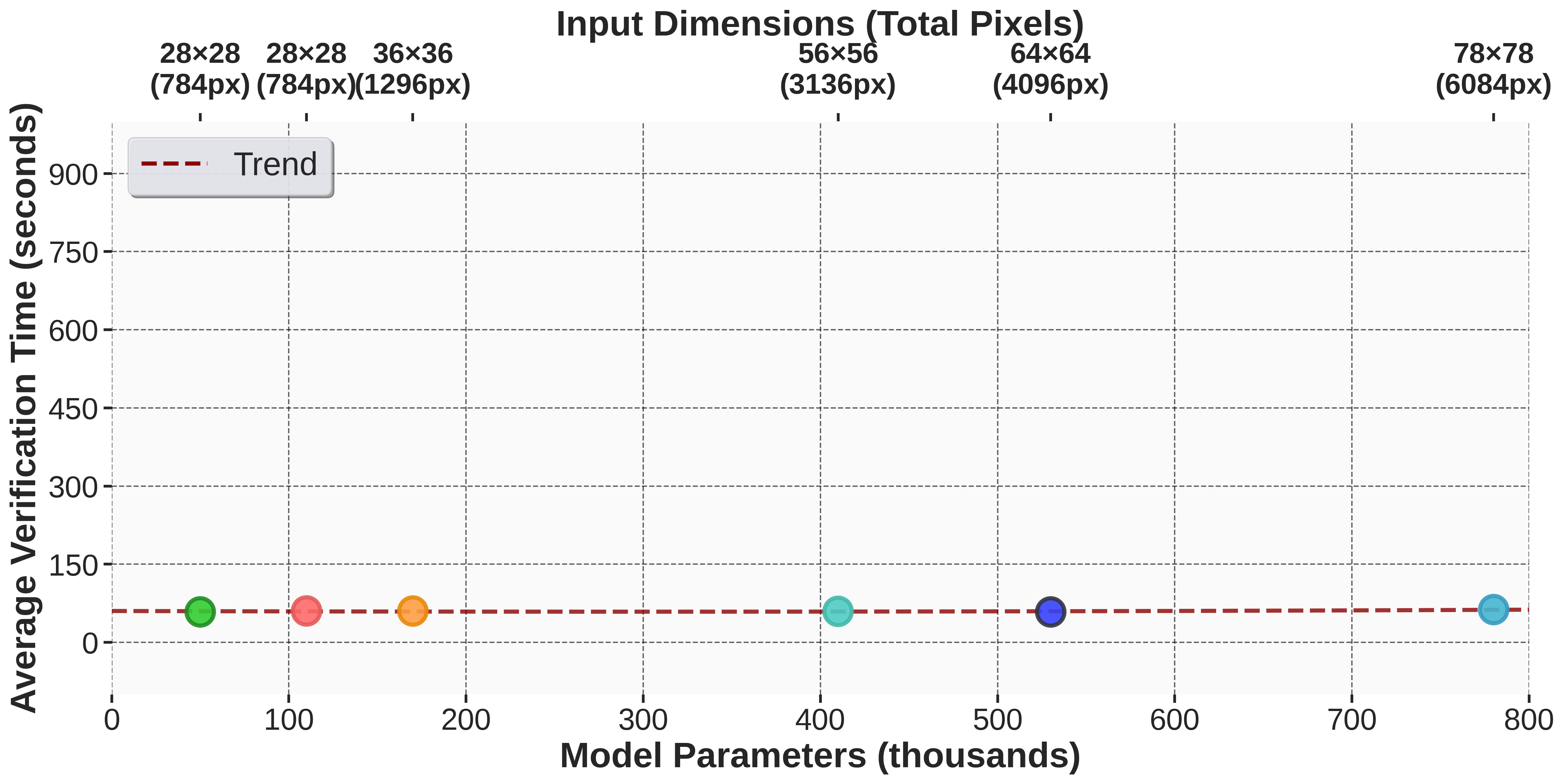} 
        \caption{CP Star}
        \label{fig:second_diagram_cp_star}
    \end{subfigure}
    \caption{Scalability analysis on MLP models with varying input image sizes (MNIST 28x28, 56x56, 78x78). Runtime increases with input dimensions, with Approx-Star showing superior efficiency over CP-Star for MLP architectures.}
    \label{fig:combined_scalabilitiy_input_sizes}
\end{figure}

As the input dimension increases, the computation time increases for both solvers, reflecting the higher complexity of the search space. Approx-Star (Figure~\ref{fig:scaletoimagesize_approx_star}) is significantly faster than CP-Star (Figure~\ref{fig:second_diagram_cp_star}) on these MLP architectures. Approx-Star scales from approx. 2s (28x28) to approx. 15s (78x78), whereas CP-Star scales from approx. 55s to 80s. This indicates that while \vitax{} can handle larger images, the choice of the underlying solver is critical and often architecture-dependent—Approx-Star is preferable for MLPs, while CP-Star may be necessary for complex CNNs/ResNets (as shown previously in Table~\ref{tab:complex_architectures_cpstar}).

This analysis confirms that the \vitax{} methodology remains effective for identifying critical, targeted feature sets even on deep networks and larger inputs. We acknowledge the computational cost associated with formal verification and position \vitax{} as a rigorous tool for precise boundary analysis, where this cost is justified by the need for formal guarantees in safety-critical domains (see Section 6). 

\paragraph{Application to TaxiNet Regression Task}\label{subsec:taxinet_eval}
To demonstrate broader applicability beyond classification, we applied \vitax{} to the TaxiNet regression task \cite{julian2020validation} (Figure~\ref{fig:taxinet}). For regression, \vitax{} identifies features critical for explaining the model's resilience around its current output value (steering angle). \vitax{} reveals how the critical feature set $A$ changes based on the targeted output shift—focusing on immediate path projection (e.g., the white lines near the center) for small adjustments, and features further ahead for larger changes. This demonstrates \vitax{}'s utility for understanding the dynamic feature reliance in continuous control models.

\begin{figure}[!t]
    \centering
    \begin{subfigure}[t]{0.19\textwidth}
        \includegraphics[width=\textwidth]{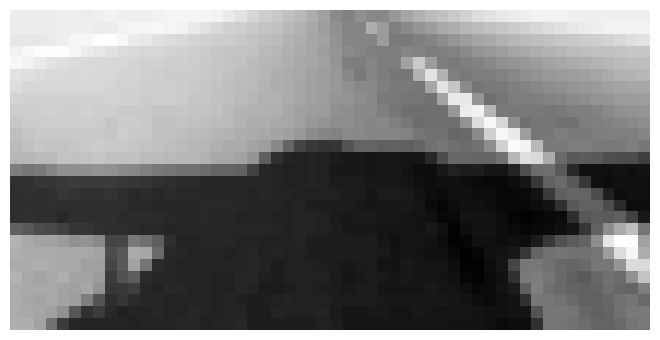}
    \end{subfigure}
    \hfill
    \begin{subfigure}[t]{0.19\textwidth}
        \includegraphics[width=\textwidth]{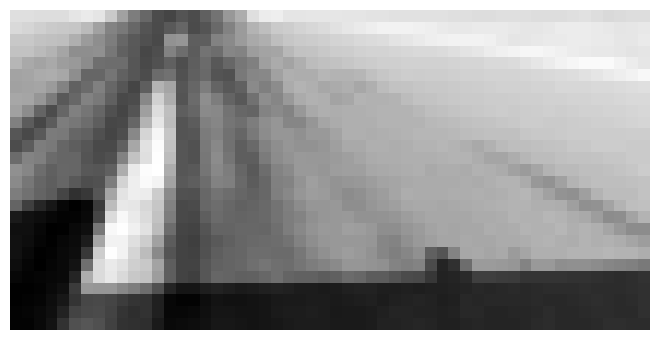}
    \end{subfigure}
    \hfill
    \begin{subfigure}[t]{0.19\textwidth}
        \includegraphics[width=\textwidth]{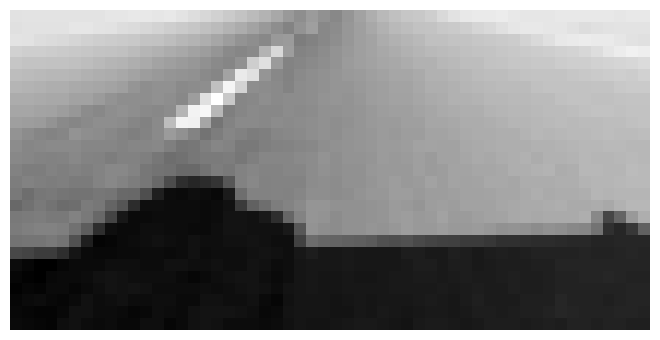}
    \end{subfigure}
    \hfill
    \begin{subfigure}[t]{0.19\textwidth}
        \includegraphics[width=\textwidth]{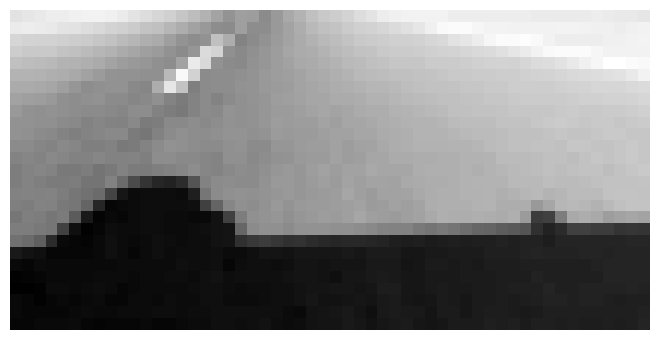}
    \end{subfigure}
    \hfill
    \begin{subfigure}[t]{0.19\textwidth}
        \includegraphics[width=\textwidth]{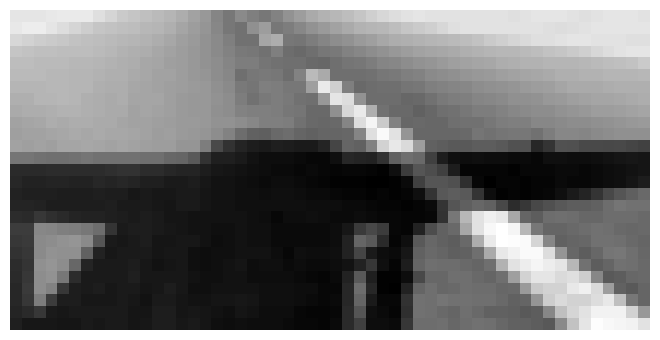}
    \end{subfigure}

    \begin{subfigure}[t]{0.19\textwidth}
        \includegraphics[width=\textwidth]{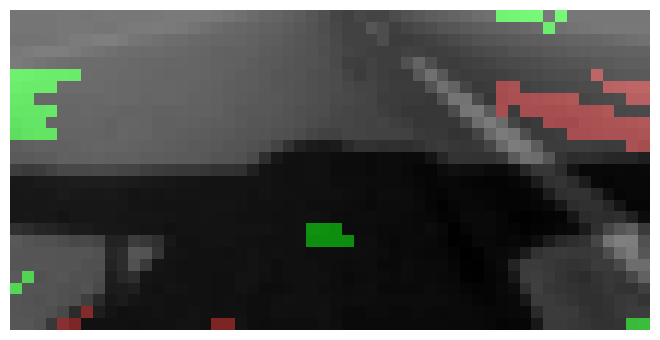}
        \caption{GT: 5.06 (Pred: 5.13)}
    \end{subfigure}
    \hfill
    \begin{subfigure}[t]{0.19\textwidth}
        \includegraphics[width=\textwidth]{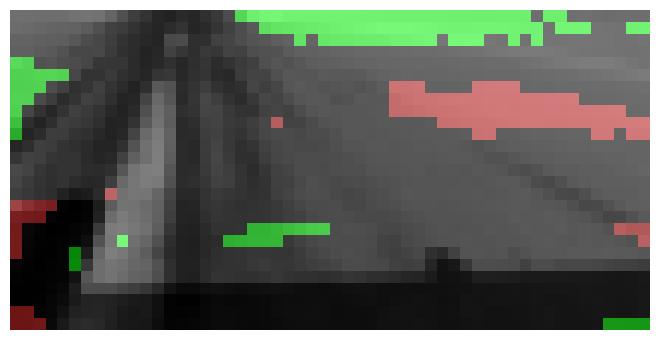}
        \caption{GT: 1.55 (Pred: 1.40)}
    \end{subfigure}
    \hfill
    \begin{subfigure}[t]{0.19\textwidth}
        \includegraphics[width=\textwidth]{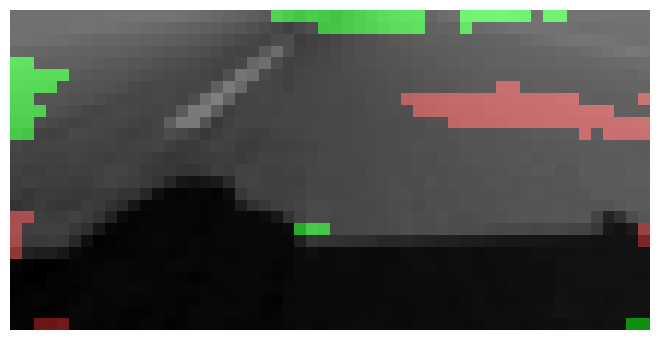}
        \caption{GT: -1.21 (Pred: -1.16)}
    \end{subfigure}
    \hfill
    \begin{subfigure}[t]{0.19\textwidth}
        \includegraphics[width=\textwidth]{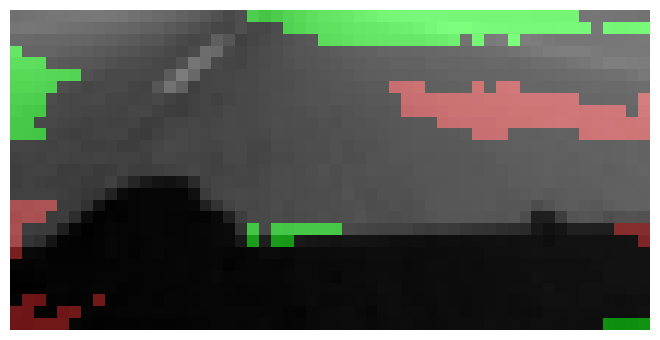}
        \caption{GT: -0.70 (Pred: -0.83)}
    \end{subfigure}
    \hfill
    \begin{subfigure}[t]{0.19\textwidth}
        \includegraphics[width=\textwidth]{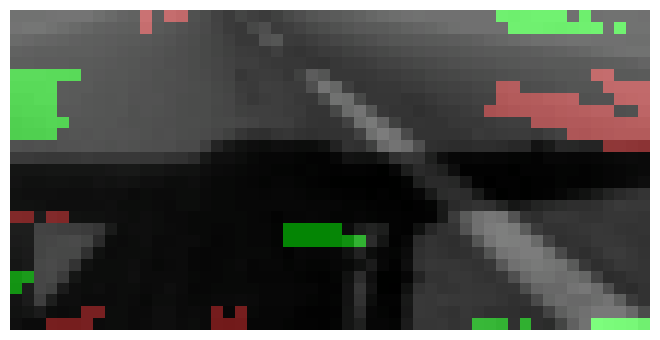}
        \caption{GT: 4.90 (Pred: 5.01)}
    \end{subfigure}
    \caption{Application of \vitax{} to the TaxiNet regression task. Top row: original input images. Bottom row: \vitax{} explanations highlighting features relevant to the predicted steering angle (values shown as Ground Truth (Model Prediction)). Green indicates positive perturbations, red indicates negative, relative to their influence on the output.}
    \label{fig:taxinet}
\end{figure}

\subsection{Qualitative Analysis and Ablation Studies (RQ3)}\label{subsec:RQ3}

We now examine the qualitative nature of \vitax{}'s explanations and conduct ablation studies on key components—the perturbation magnitude $\epsilon$, the heuristic ranking function, and the underlying solver—to understand their impact on the resulting explanations.

\paragraph{Qualitative Effectiveness}\label{subsec:qualitative_eval}

We examine \vitax{}'s application to the GTSRB (Figure~\ref{fig:GTSRB-straight}) and EMNIST (Figure~\ref{fig:emnist}) datasets to qualitatively assess the nature of its targeted semifactual explanations. These results demonstrate \vitax{}'s ability to identify minimal, critical feature subsets pertinent to understanding the resilience against specific target classes.

\begin{figure}[!t]
    \centering
    \scalebox{0.7}{%
    \begin{minipage}[t]{\textwidth}
        \centering
        \begin{subfigure}[t]{0.19\textwidth}
            \includegraphics[width=\textwidth]{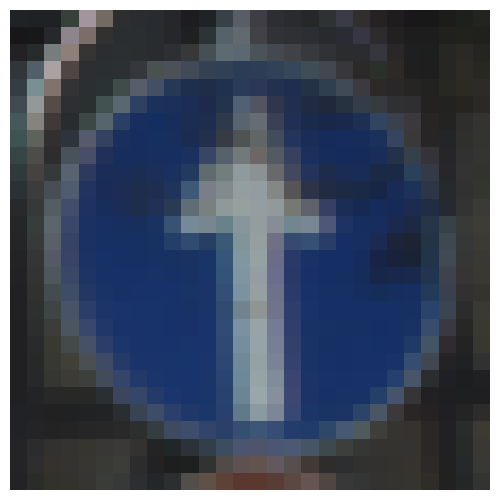}
            \subcaption{Straight (Original)}
        \end{subfigure}
        \hfill
        \begin{subfigure}[t]{0.19\textwidth}
            \includegraphics[width=\textwidth]{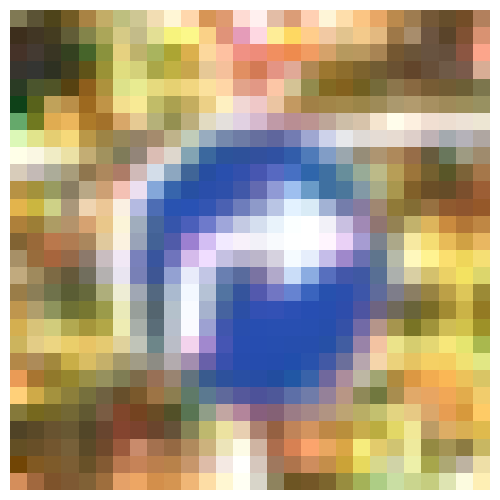}
            \subcaption{Turn Right (Target)}
        \end{subfigure}
        \hfill
        \begin{subfigure}[t]{0.19\textwidth}
            \includegraphics[width=\textwidth]{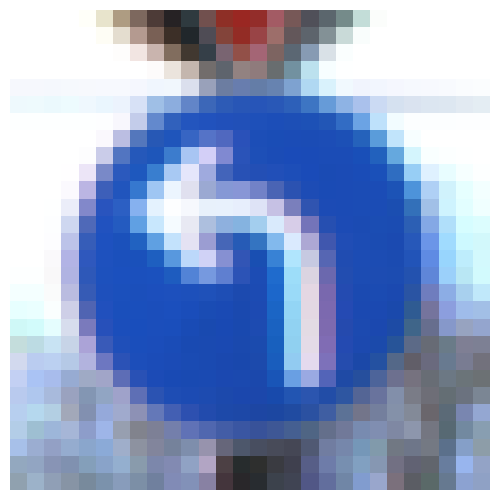}
            \subcaption{Turn Left (Target)}
        \end{subfigure}
        \hfill
        \begin{subfigure}[t]{0.19\textwidth}
            \includegraphics[width=\textwidth]{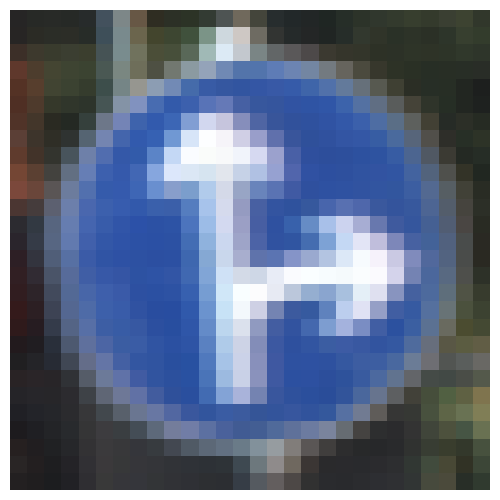}
            \subcaption{Straight or Right (Target)}
        \end{subfigure}
        \hfill
        \begin{subfigure}[t]{0.19\textwidth}
            \includegraphics[width=\textwidth]{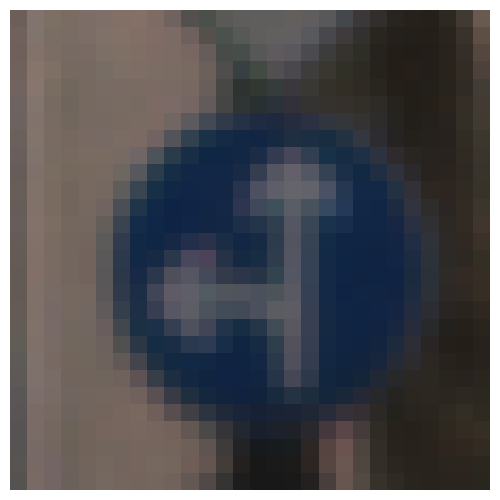}
            \subcaption{Straight or Left (Target)}
        \end{subfigure}

        \begin{subfigure}[t]{0.19\textwidth}
            \centering
            Targeted Explanation (\vitax{})
        \end{subfigure}
        \hfill
        \begin{subfigure}[t]{0.19\textwidth}
            \includegraphics[width=\textwidth]{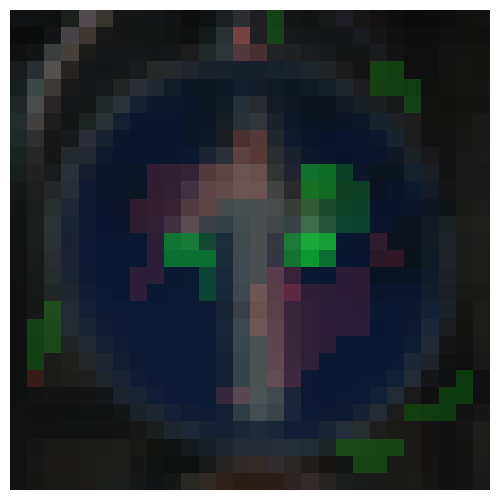}
        \end{subfigure}
        \hfill
        \begin{subfigure}[t]{0.19\textwidth}
            \includegraphics[width=\textwidth]{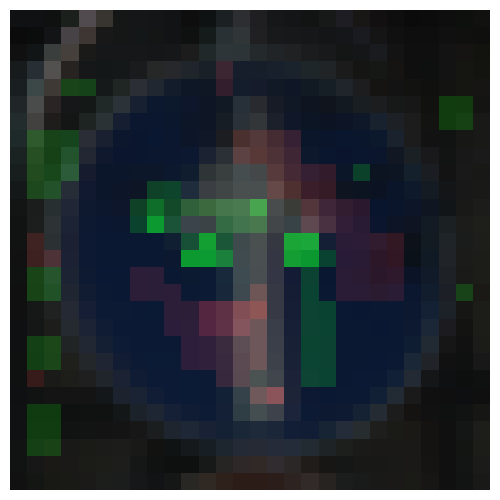}
        \end{subfigure}
        \hfill
        \begin{subfigure}[t]{0.19\textwidth}
            \includegraphics[width=\textwidth]{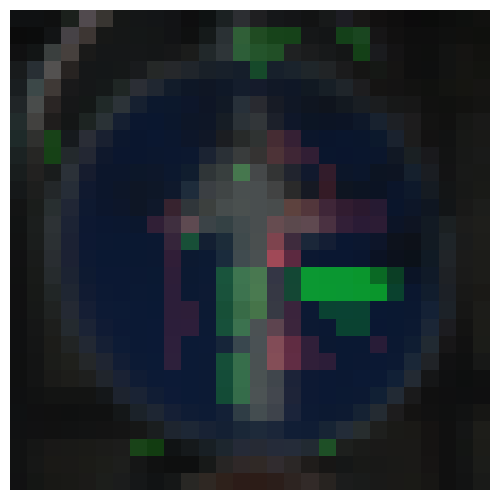}
        \end{subfigure}
        \hfill
        \begin{subfigure}[t]{0.19\textwidth}
            \includegraphics[width=\textwidth]{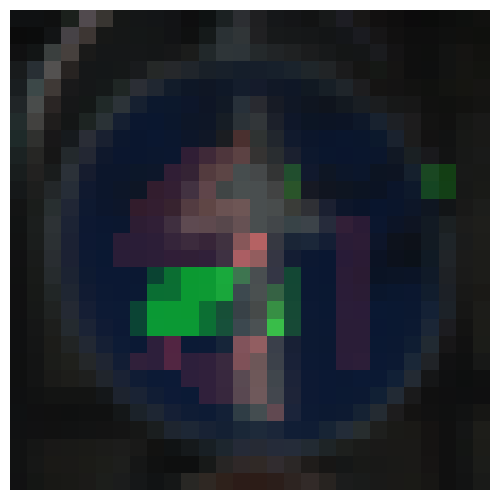}
        \end{subfigure}
          
    \end{minipage}}
    \caption{GTSRB Targeted Explanations. The first row shows the original ``Straight'' sign and various targets. The second row shows the \vitax{} semifactual explanations (Feature Set A) critical for the resilience of ``Straight'' against each specific target.}
    \label{fig:GTSRB-straight}
\end{figure}

In the GTSRB example (Figure~\ref{fig:GTSRB-straight}), we analyze the resilience of the ``Straight'' sign against different directional alternatives. When the target is ``Turn Right'' (Col 2) or ``Turn Left'' (Col 3), \vitax{} correctly identifies the specific side of the vertical arrow shaft as the critical feature set $A$. This aligns with intuition: these are the precise areas where the arrowheads for left or right turns would manifest. The explanation guarantees that the ``Straight'' classification persists even if these critical areas are perturbed, defining the boundary for this specific transition.

\begin{figure}[!t]
\centering
\begin{subfigure}{0.10\textwidth}
  \includegraphics[width=\linewidth]{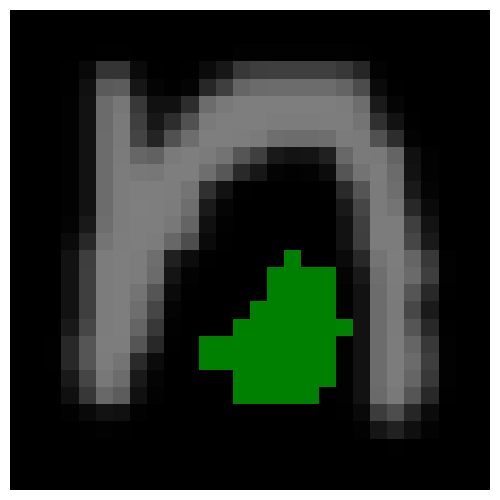}
  \caption{n to A}
\end{subfigure}
\hfill
\begin{subfigure}{0.10\textwidth}
  \includegraphics[width=\linewidth]{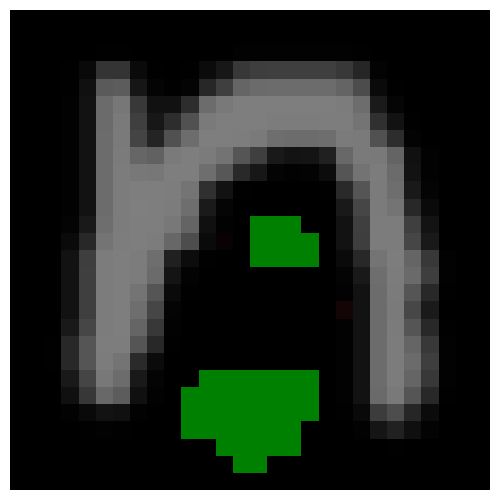}
  \caption{n to B}
\end{subfigure}
\hfill
\begin{subfigure}{0.10\textwidth}
  \includegraphics[width=\linewidth]{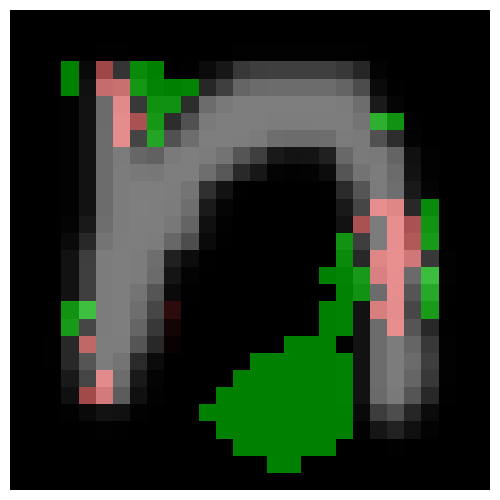}
  \caption{n to C}
\end{subfigure}
\hfill
\begin{subfigure}{0.10\textwidth}
  \includegraphics[width=\linewidth]{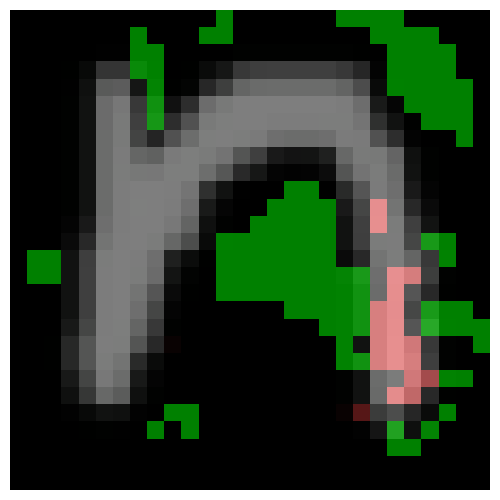}
  \caption{n to F}
\end{subfigure}
\hfill
\begin{subfigure}{0.10\textwidth}
  \includegraphics[width=\linewidth]{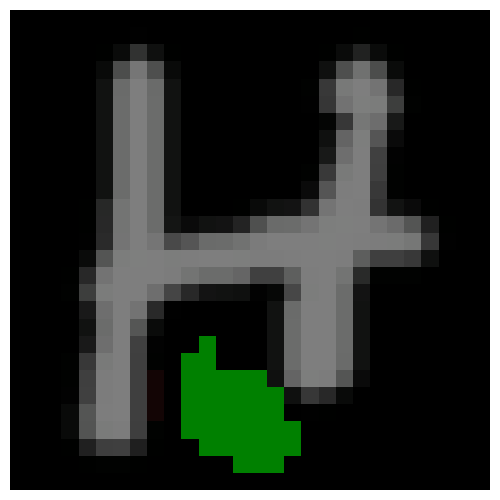}
  \caption{H to a}
\end{subfigure}
\hfill
\begin{subfigure}{0.10\textwidth}
  \includegraphics[width=\linewidth]{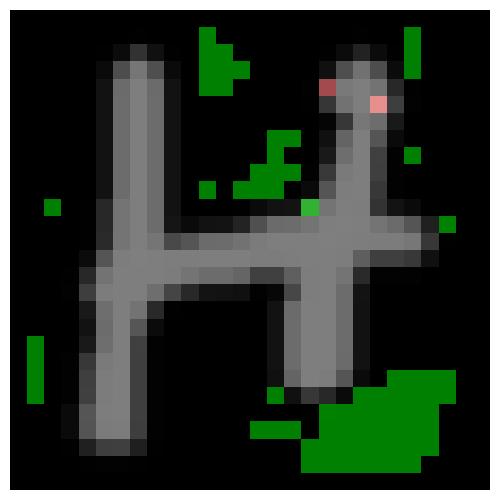}
  \caption{H to K}
\end{subfigure}
\hfill
\begin{subfigure}{0.10\textwidth}
  \includegraphics[width=\linewidth]{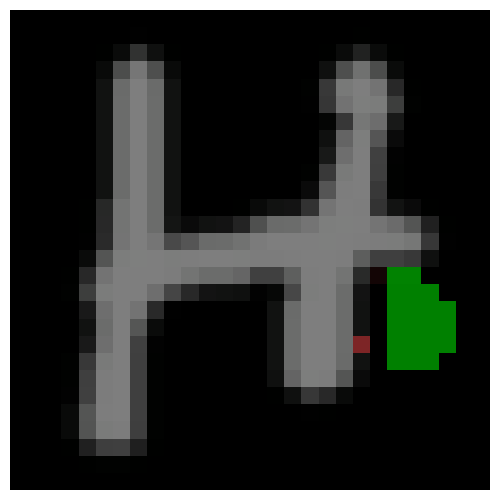}
  \caption{H to m}
\end{subfigure}
\hfill
\begin{subfigure}{0.10\textwidth}
  \includegraphics[width=\linewidth]{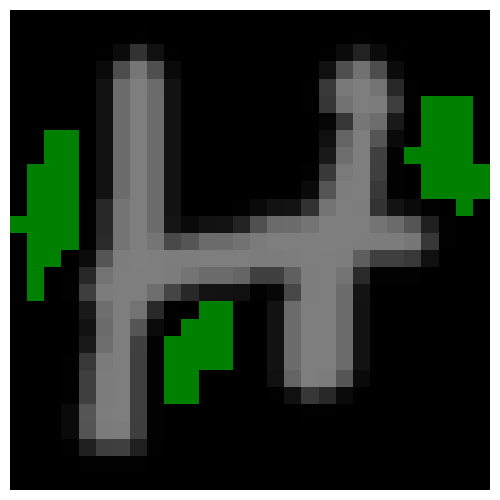}
  \caption{H to W}
\end{subfigure}
\caption{EMNIST Targeted Explanations. \vitax{} identifies the semifactual explanation relevant to specific character transitions.}
\label{fig:emnist}
\end{figure}

Similarly, on EMNIST (Figure~\ref{fig:emnist}), the explanations are highly specific to the target morphology. For the transition from `n' to `A' (a), \vitax{} highlights the center gap where a crossbar would be needed. For `H' to `m' (g), it focuses on the upper sections of the vertical lines. This demonstrates the core strength of \vitax{}: unlike general attribution methods that highlight features supporting the current class, \vitax{} isolates the features specifically relevant to the boundary with the chosen alternative.

\paragraph{Effect of Perturbation Size ($\epsilon$)}\label{subsec:epsilon_effect}

The perturbation magnitude, \( \epsilon \), is a critical parameter that controls the granularity of the explanation $A$ (Figure~\ref{fig:epsilon_compare}). It defines the magnitude of change allowed when verifying the Targeted $\epsilon$-Robustness property.

There is an inverse relationship between \( \epsilon \) and the size of the explanation set $A$. When \( \epsilon \) is small (e.g., 0.059), the allowed perturbation is minor. Consequently, a larger set of features $A$ can be included while still guaranteeing that the classification (e.g., `3') persists against the target (`8'). The explanation is broader.

As \( \epsilon \) increases (e.g., 0.216), the allowed perturbation is significant. To maintain the formal guarantee that the boundary is not crossed, the feature set $A$ must be reduced, focusing only on the most critical features prioritized by the heuristic ranking. The explanation becomes sparser and more focused on the key areas differentiating `3' from `8' (the left-side gaps). This demonstrates how $\epsilon$ acts as a ``zoom lens,'' allowing users to explore the decision boundary at varying levels of specificity.

\begin{figure}[!t]
    \centering
    \begin{subfigure}[t]{0.16\linewidth}
        \centering
        \includegraphics[width=\textwidth]{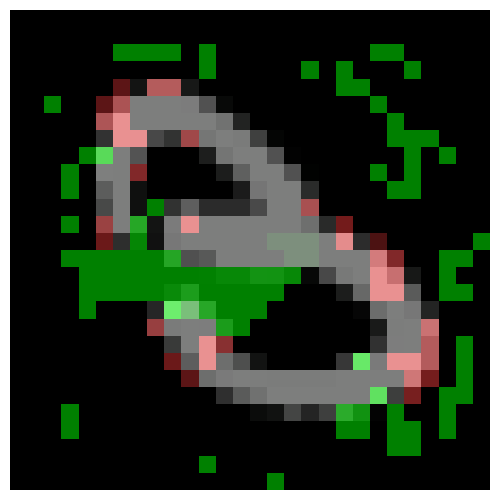}
        \caption{$\epsilon=0.059$}
    \end{subfigure}
    \begin{subfigure}[t]{0.16\linewidth}
        \centering
        \includegraphics[width=\textwidth]{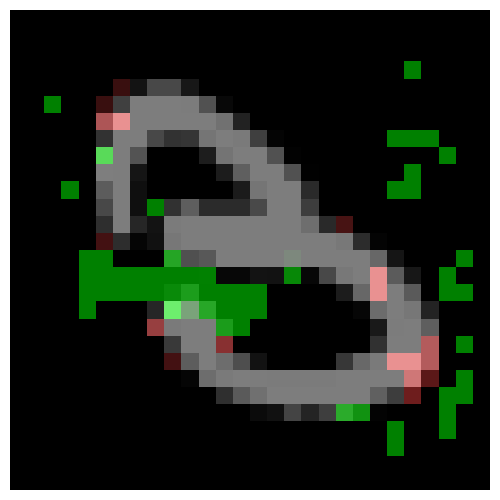}
        \caption{$\epsilon=0.098$}
    \end{subfigure}
    \begin{subfigure}[t]{0.16\linewidth}
        \centering
        \includegraphics[width=\textwidth]{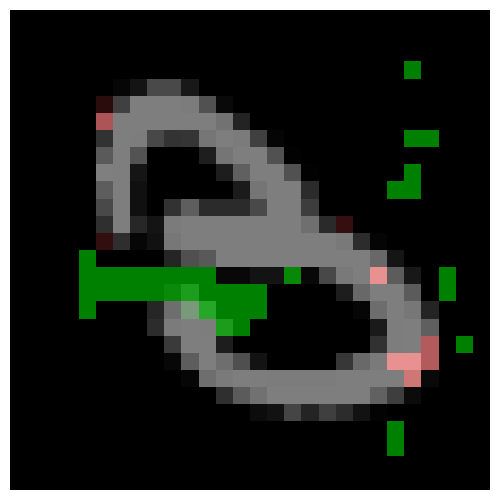}
        \caption{$\epsilon=0.137$}
    \end{subfigure}
    \begin{subfigure}[t]{0.16\linewidth}
        \centering
        \includegraphics[width=\textwidth]{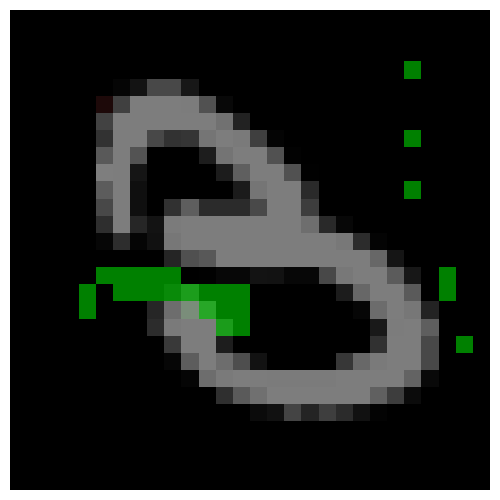}
        \caption{$\epsilon=0.216$}
    \end{subfigure}
    \caption{\vitax{}'s targeted explanation for transitioning MNIST `3' (original) towards `8' (target) at different perturbation magnitudes $\epsilon$. As $\epsilon$ increases, the minimal set $A$ satisfying Targeted $\epsilon$-Robustness tends to become smaller and more focused.}
    \label{fig:epsilon_compare}
\end{figure}

\paragraph{Evaluating Heuristic Ranking Functions}\label{subsec:heuristic_eval}

\begin{figure}[t]
    \centering
    \begin{subfigure}[b]{0.45\textwidth}
        \centering
        \includegraphics[width=\textwidth]{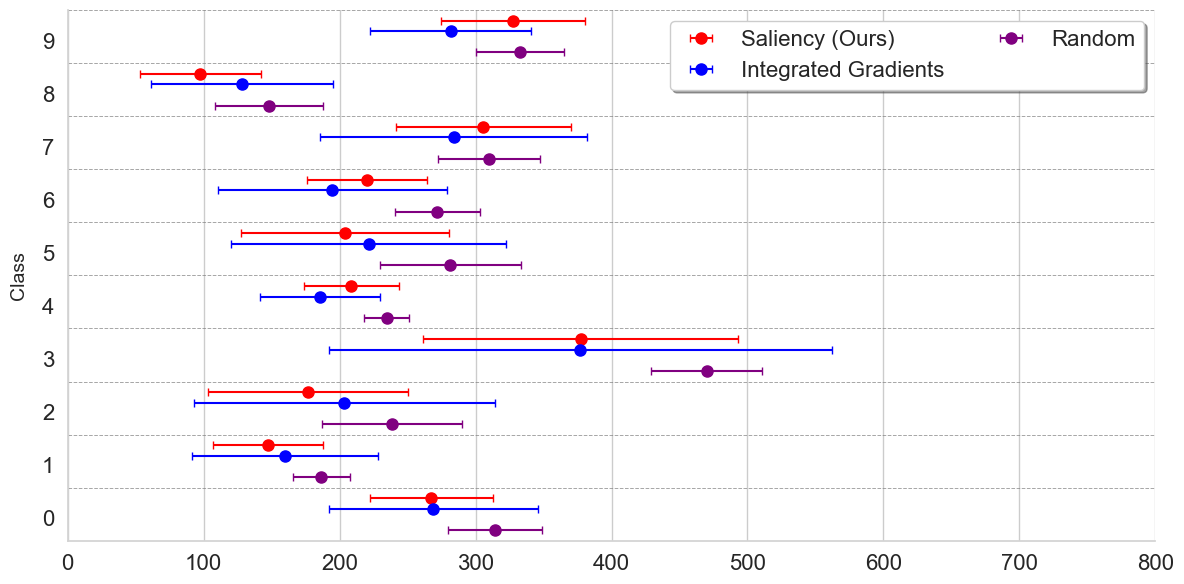}
        \caption{MNIST}
        \label{fig:sub1}
    \end{subfigure}
    \begin{subfigure}[b]{0.45\textwidth}
        \centering
        \includegraphics[width=\textwidth]{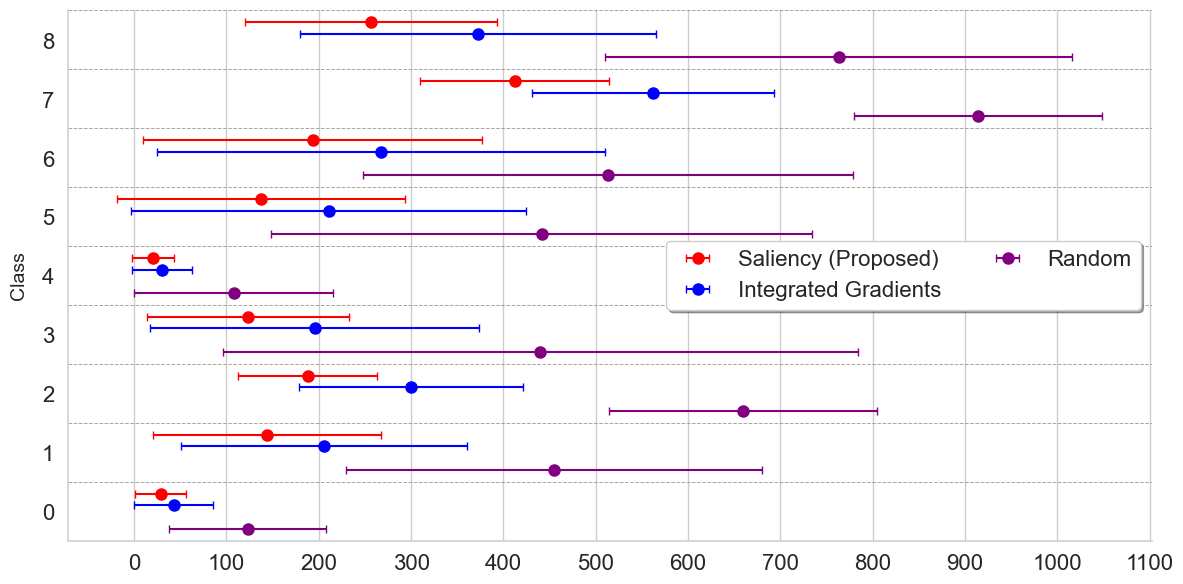}
        \caption{GTSRB}
        \label{fig:sub2}
    \end{subfigure}
    \caption{Comparison of explanation cardinality (Number of Pixels, x-axis) resulting from different heuristic ranking functions across classes (y-axis). Lower cardinality is better. Sensitivity consistently yields the most minimal explanations. (GTSRB shows the first nine classes for readability).}
    \label{fig:heuristic}
\end{figure}

The choice of the heuristic ranking function ($\pi$) is crucial for the efficiency and quality (minimality) of the \vitax{} explanation. While the formal guarantee holds regardless of the heuristic (\vitax{} guarantees $A$ is the maximal robust prefix of $\pi$), an effective heuristic guides the search towards smaller feature sets faster. We compared our proposed sensitivity-based approach (gradient saliency towards the target class) against Integrated Gradients (IG) \cite{sundararajan_axiomatic_2017} and a Random ranking baseline.

The sensitivity-based heuristic consistently leads to explanations with significantly lower cardinality across both MNIST and GTSRB datasets (Figure~\ref{fig:heuristic}). Random ranking performs the worst, often resulting in very large feature sets, as the search cannot efficiently prioritize relevant features.

The superiority of the sensitivity heuristic stems from its direct alignment with the goal of \vitax{}: by prioritizing features based on their direct impact on the target class logits, the search efficiently pinpoints the smallest prefix of features that defines the decision boundary. This analysis confirms that a well-aligned heuristic is pivotal for delivering minimal and focused explanations.

\paragraph{Impact of Solver Choice}\label{subsec:solver_scalability}

The performance of \vitax{} is dependent on the underlying reachability solver. We compared different methods provided by the NNV tool—including Exact Star, Approx-Star, Relax-Star variants, and CP-Star—on a smaller MNIST model (Table~\ref{star-methods}) to assess the trade-offs between precision and speed.

The Exact Star method provides sound and complete analysis but is computationally expensive (30.71s). Approx-Star, which is sound but incomplete (it over-approximates the reach set), offers a significant speedup (2.04s, approximately 15 times faster). CP-Star, while crucial for complex architectures (Section 4.3), was the slowest among the MLP architectures (52.18s).

Interestingly, the Approx-Star method yields explanations with lower cardinality (111.46 vs 129.66 for Exact Star) while maintaining nearly identical fidelity (0.218 vs 0.220) and robustness (0.59 vs 0.60). Because Approx-Star over-approximates the reach set, it is more conservative in certifying robustness. This may force the search algorithm to identify a smaller, more critical feature set $A$ to satisfy the Targeted $\epsilon$-Robustness property under the over-approximation. Relax-Star methods offer marginal speed improvements with slight reductions in fidelity.

This analysis, combined with the findings in Section~\ref{subsec:RQ2}, demonstrates the importance of \vitax{}'s solver-agnostic design. For simpler models such as MLPs, Approx-Star generally provides the optimal trade-off, enabling the efficient delivery of formally sound explanations at a significantly reduced computational cost (as also seen in Figure~\ref{fig:combined_scalabilitiy_input_sizes}). However, the choice of solver is architecture-dependent; for deeper and more complex networks like ResNets or Inception (Table~\ref{tab:complex_architectures_cpstar}), alternative solvers such as CP-Star may be necessary to handle the complexity of the reachability analysis.

\begin{table}[t]
\small
  \caption{Comparison of Different Reachability Solver Methods within \vitax{} (on a smaller MNIST model).}
  \label{star-methods}
  \centering
  \begin{tabular}{lcccl}
    \toprule
    Solver Method & Time (s) $\downarrow$ & Card. $\downarrow$ & Fidelity $\uparrow$  & Rob. (NE) $\uparrow$\\
    \midrule
    Exact Star     & 30.71        & 129.66       & 0.220              & 0.60\\
    Approx Star    & 2.04         & 111.46       & 0.218              & 0.59\\
    Relax Star 50\% & 2.06         & 111.46       & 0.218              & 0.59\\
    Relax Star 75\% & 2.03         & 109.91       & 0.217              & 0.59\\
    Relax Star 85\% & 2.00       & 109.91       & 0.217              & 0.59\\
    CP Star        & 52.18        & 130.28       & 0.211 & 0.59\\
   \bottomrule
  \end{tabular}
\end{table}

\section{Related Work}\label{sec:related_work}

The field of Explainable AI (XAI) aims to render the behavior of complex models transparent. We review the landscape relevant to \vitax{}, focusing on the evolution from heuristic methods to the emerging demand for formal guarantees and robustness.

\subsection{Heuristic XAI: Attribution and Surrogates}
Attribution-based methods are pivotal, highlighting influential input data portions that contribute to a model's decision. Examples include Grad-CAM, which uses gradient-based localization \cite{selvaraju_grad-cam_2020}, Integrated Gradients (IG) that assess feature significance via model gradients \cite{sundararajan_axiomatic_2017}, and DeepLIFT \cite{shrikumar_learning_2019}. Game-theoretic approaches like SHAP (SHapley Additive exPlanations) values use game theory to assign importance \cite{lundberg_unified_2017}. Other techniques include the Input X Gradient method, which multiplies the input by the output gradients \cite{shrikumar_not_2016}, and Guided Backpropagation, which enhances standard backpropagation for visualization \cite{springenberg_striving_2015}. These methods, among others \cite{linardatos2020explainable, wang2024microxercise}, collectively enhance neural network interpretability, though they often provide only correlational insights and have been shown to lack stability.

Instead of analyzing the complex model directly, local surrogate models like LIME (Local Interpretable Model-agnostic Explanations) \cite{ribeiro_why_2016} train simpler, interpretable models to approximate the black-box behavior locally. Anchors \cite{ribeiro_anchors_2018} identifies a minimal set of feature conditions sufficient to lock in a prediction. These provide plausible local explanations but lack formal assurances of their fidelity. Some approaches also utilize secondary models or delve into causal effects for classification tasks \cite{karimi_algorithmic_2020,tran_unsupervised_2022}.


\subsection{Robustness and Formal Guarantees in XAI}
The fragility of standard heuristic methods has motivated a shift towards rigorous explanations. This involves two main thrusts: ensuring the stability of explanations (Robust XAI) and providing mathematical guarantees about their content (Formal XAI).

Regarding robustness of explanations, small input perturbations should not lead to vastly different explanations. Wicker et al.~\cite{wicker2022robust} introduced constraints for robust feature attribution (e.g., saliency maps). Similarly, the robustness of CFEs is critical, ensuring that suggested changes remain valid under noise or model shifts~\cite{jiang2024robust}. Verification techniques have been employed to certify this validity. For instance, Leofante and Lomuscio \cite{leofante2023robust} introduced an approach to formally verify robust CFEs for human-neural multi-agent systems.

Formal methods in XAI aim to provide rigorous guarantees for explanations \cite{muller2022prima,zelazny2022optimizing,singh2019abstract}. Many works define an explanation as a minimal subset of input features critical for a model's decision, such that perturbations outside this subset ideally do not alter the prediction. Examples include abductive explanations \cite{ignatiev2019abduction}, sufficient reasoning \cite{darwiche_reasons_2020}, prime implicants \cite{shih_symbolic_2018}, and contrastive explanations \cite{miller2019explanation}. The concept of \emph{bounded} perturbations as explanations has been introduced, particularly in the NLP domain \cite{malfa_guaranteed_2021}. Building on these foundations, VeriX presents an iterative method to find the smallest number of features within continuous $\epsilon$-ball perturbations that constitute an explanation \cite{wu_verix_2023}. These methods often prioritize global cardinality minimality, which is generally NP-hard. Such efforts are closely related to the broader field of neural network verification, which often employs traditional software verification techniques like SMT \cite{katz2017reluplex} and MILP \cite{tjeng2017mipverify} for exact analysis, though these can face scalability challenges. Consequently, faster, albeit sometimes less precise, methods like abstract interpretation and probabilistic assessments have been developed. Significant research investigates how automated reasoning, through abstraction \cite{muller2022prima,zelazny2022optimizing,wu2022scalable,singh2019abstract,zhang2018efficient,wang2018efficient,wang2018formal,Gehr2018ai2,anderson2019optimization,tran2020verification,wei2023convex,zhang2024optimizing} and search techniques \cite{ehlers2017formal,katz2017reluplex,katz2019marabou,huang2017safety,ferrari2022complete,geng2023towards}, can verify network specifications \cite{liu2021algorithms,huang2020survey}. Formal logic-based approaches have also been explored for enforcing temporal properties in neural networks~\cite{ma2020stlnet} and for predictive monitoring with calibrated uncertainty~\cite{ma2021predictive}. Recent advances also include neural network-based meta-verification strategies that predict safety properties to reduce computational load \cite{elboher2022neural}.

\vitax{} uniquely intersects these areas and is distinct from existing paradigms in several key ways. \textbf{(1) Derivation vs. Stability:} Unlike methods that improve the robustness \emph{of} an existing explanation (e.g., Wicker et al. \cite{wicker2022robust} certify that gradient-based explanations remain stable when inputs or model parameters are perturbed), \vitax{} uses a formal robustness property (Targeted $\epsilon$-Robustness) to \emph{derive} the explanation itself—the explanation is the direct result of verifying which features satisfy the robustness property. For example, Wicker et al.'s approach would compute a gradient-based explanation (e.g., feature attribution scores) for an image classified as `7', then verify the explanation remains consistent under perturbations. In contrast, \vitax{} directly identifies the minimal feature set (e.g., specific pixels) that can be perturbed by $\epsilon$ without causing a 7$\to$9 flip, with the robustness verification defining the explanation. \textbf{(2) Semifactual vs. Counterfactual Guarantees:} \vitax{} differs fundamentally from approaches focused on robust CFEs (e.g., Leofante and Lomuscio \cite{leofante2023robust}). While CFE methods guarantee the robustness of a \emph{decision flip}, \vitax{} guarantees the robustness of the \emph{non-flip} (the semifactual persistence). \textbf{(3) Scalability and Targeting:} Unlike Formal XAI methods focused on global minimality and general sufficiency (e.g., VeriX \cite{wu_verix_2023}, Ignatiev et al. \cite{ignatiev2019abduction}), \vitax{} sacrifices the intractable goal of global minimality for scalability (using $\mathcal{O}(\log_2(N))$ oracle calls) and provides targeted insights into specific class boundaries.

\section{Discussion}\label{sec:discussion}
Our empirical results demonstrate the effectiveness of \vitax{} in generating precise, formally-backed explanations. In what follows, we examine four aspects of this contribution: (1) the nature and scope of the Targeted $\epsilon$-Robustness guarantee, (2) how \vitax{} produces verified semifactual explanations, (3) its fundamental distinction from counterfactual explanation paradigms, and (4) its positioning within the broader XAI landscape.

\paragraph{The Nature of the Guarantee for Targeted Semifactual Robustness}
A key aspect of \vitax{} is the nature of the formal guarantee it provides. The Targeted $\epsilon$-Robustness (Definition~\ref{def:targeted_robustness_def}) is specifically local to the given input $x$, the identified feature subset $A$, and the chosen target class $\mathbf{t}$. This means that for the explanation $A$, \vitax{} formally verifies that perturbing these specific features by $\epsilon$ does not cause the model's classification for the original class $\mathbf{y}$ to flip to the target class $\mathbf{t}$. This provides a semifactual guarantee: ``even if features in $A$ are perturbed by $\epsilon$, the classification remains $\mathbf{y}$.''

This type of guarantee differs significantly from formal methods that aim to prove, for instance, the robustness of a model over an entire $\epsilon$-ball neighborhood around $x$ against \emph{any} adversarial perturbation towards \emph{any} incorrect class. Such global or comprehensive robustness certifications provide broad assurances about the model's stability. In contrast, \vitax{} offers a more nuanced and specific guarantee that is intrinsically tied to the explanation itself. Rather than certifying the model's general invulnerability, \vitax{} certifies a property of the \emph{explanation's features}: that they represent a point of semifactual robustness against a specific, targeted alternative.

Unlike Formal XAI methods focused on global minimality and general sufficiency (e.g., \cite{ignatiev2019abduction}), which are often NP-hard and thus computationally intractable, \vitax{} makes a deliberate trade-off. It sacrifices the intractable goal of global cardinality minimality for scalability and targeted insights. The novelty of our approach lies not in the binary search algorithm itself, but in (1) the formulation of a new formal property, Targeted $\epsilon$-Robustness, and (2) the creation of a framework that integrates a sensitivity-driven heuristic with formal verification to make these explanations computationally feasible. By using $\mathcal{O}(\log_{2}(N))$ oracle calls to a formal solver, \vitax{} provides a guarantee that is both rigorous and practical: it identifies the maximal robust prefix of a well-aligned heuristic. This offers a trustworthy and scalable alternative for understanding model behavior at specific class boundaries where exact methods would be infeasible.

\paragraph{\vitax{} as Verified Semifactual Explanations}
Methods focused on robust XAI typically aim to ensure the robustness \emph{of} explanations (such as gradient-based feature attributions) against perturbations in the input; they seek to certify that explanations do not change significantly when the input changes slightly \cite{wicker2022robust}. In contrast, \vitax{} uses a robustness property of the \emph{model's output} (Targeted $\epsilon$-Robustness) to \emph{derive} the explanation itself. The explanation ($A$) is defined by the verification of the model's robustness, rather than being a post-hoc robustness certification of a pre-computed explanation.

\vitax{} generates semifactual explanations \cite{aryal2023even, kenny2023utility} that answer ``even-if'' questions: ``Even if features $A$ are perturbed by $\epsilon$, the classification remains $\mathbf{y}$.'' The key innovation is making semifactual explanations \emph{targeted}: rather than identifying features sufficient for $\mathbf{y}$ against any alternative, we focus on features most relevant to specific class transitions. This targeting occurs in two stages:  \textbf{(1)} We use sensitivity analysis ($\nabla_x f_t(x)$) to identify features most relevant to the $\mathbf{y} \to \mathbf{t}$ boundary. \textbf{(2)} We formally verify these features satisfy ``even-if'' robustness (can be perturbed by $\epsilon$ without flipping to $\mathbf{t}$). This design provides class-specific semifactual insights that are both efficient to compute ($\mathcal{O}(\log n)$ via binary search on the sensitivity ranking) and formally guaranteed (via reachability analysis).

\paragraph{Relation to Counterfactual Explanations}
Our approach differs fundamentally from Counterfactual Explanations (CFEs) \cite{haufe2024explainable, montavon_methods_2018}. CFEs identify minimal changes that flip decisions, answering ``if-only'' questions and explaining model fragility. \vitax{} answers ``even-if'' questions and explains robustness against specific alternatives.

It is crucial to clarify that while the visualizations in Figures \ref{fig:comparison_to_table1} and \ref{fig:emnist} highlight features sensitive to a transition, these figures visualize the feature set $A$, a semifactual insight about which features can change without flipping the class, not an actualized counterfactual instance that successfully flips the class.

\vitax{} also differs fundamentally from verification-based robust CFEs \cite{leofante2023robust, jiang2024robust}. Robust CFEs verify the robustness of a \emph{flip}: they find a counterfactual $x_{cf}$ where $f(x_{cf}) = \mathbf{t}$, then verify perturbations around $x_{cf}$ maintain the target class $\mathbf{t}$. The specification is: $\Phi_{\text{RCFE}}: \forall x' \in B_\epsilon(x_{cf}), f(x') = \mathbf{t}$.
\vitax{} verifies the robustness of the \emph{non-flip}: verifying that perturbations of features $A$ at the original input $x$ maintain the original class $\mathbf{y}$. The specification is: $\Phi_{\text{\vitax{}}}: \forall x' \in X_{\epsilon,A}, l_{y,A} > u_{t,A}$. These address complementary questions:  \textit{Robust CFEs:} ``What change flips to $\mathbf{t}$, and is that flip stable?'' \textit{\vitax{}:} ``What features preserve $\mathbf{y}$ against perturbations toward $\mathbf{t}$?''

Therefore, this granular, explanation-centric guarantee is not necessarily ``weaker'' than broader robustness guarantees but is differently focused, offering specific and verifiable insights into the mechanics of potential class transitions. Such understanding can be critical in applications where the conditions under which a decision \emph{doesn't} change, despite pressures towards an alternative, are as important as understanding why it might flip.

\paragraph{Positioning in the Broader XAI Landscape}
Beyond understanding the intrinsic nature of its guarantees, situating \vitax{} relative to established XAI paradigms further clarifies its contribution. \vitax{}'s approach to generating formally verified, targeted explanations offers valuable perspectives when contrasted with other XAI paradigms, notably contrastive explanations and research into adversarial perturbations \cite{miller2019explanation}. While traditional contrastive explanations often address ``Why P rather than Q?'' by showing general distinguishing features or minimal changes to achieve Q, \vitax{} provides a more specific insight: it identifies a minimal, formally verified feature subset $A$ whose $\epsilon$-perturbation is guaranteed \emph{not} to flip the classification from $\mathbf{y}$ to $\mathbf{t}$. This offers a nuanced understanding of the boundary conditions for the specific $\mathbf{y} \rightarrow \mathbf{t}$ potential transition, explaining the semifactual robustness of $\mathbf{y}$ against that particular alternative based on the identified features in $A$. Similarly, concerning adversarial robustness, \vitax{} differs from methods seeking \emph{any} minimal perturbation to cause \emph{any} misclassification \cite{chakraborty2021survey}.

\paragraph{Limitations}
While \vitax{} demonstrates strong empirical performance and provides formal guarantees, several limitations should be noted. The computational cost of formal verification remains substantial, even with $\mathcal{O}(\log_2(N))$ oracle calls. Verification can require tens of seconds per sample for complex architectures, making \vitax{} most suitable for offline analysis or safety-critical applications where formal guarantees justify this investment; future work could investigate more efficient solver strategies or parallelized verification to reduce this overhead. The quality of explanations depends on the sensitivity heuristic; while our gradient-based approach performs well empirically, the framework does not provide automatic guidance for selecting appropriate $\epsilon$ values, requiring domain expertise. Adaptive $\epsilon$ selection strategies guided by the model's local decision geometry represent a promising direction. Our guarantee is inherently local to a specific input $x$, feature subset $A$, and target class $\mathbf{t}$, and as observed in Figure~\ref{fig:robust_tot_not_tok} in Appendix~\ref{appendix:not_robust_to_others}, may not extend to robustness against all other classes $k$ when decision boundaries are closely clustered; extending the framework to jointly certify robustness across multiple target classes is a natural next step. Finally, while semifactual explanations offer unique insights, further user studies with domain experts are needed to validate whether these targeted insights are actionable for model validation and trust calibration in practice.

Despite these limitations, \vitax{} represents a meaningful step toward trustworthy, targeted formal explanations for safety-critical AI systems, bridging the gap between scalable XAI and rigorous verification.

\section{Conclusion and Future Work}\label{sec:conclusion}
We introduced \vitax{}, a formal XAI method that generates targeted semifactual explanations for deep learning models. \vitax{} identifies minimal feature subsets $A$ and provides a Targeted $\epsilon$-Robustness guarantee, formally certifying that perturbing $A$ by $\epsilon$ does not flip the classification from $\mathbf{y}$ to a chosen target class $\mathbf{t}$. This explanation-centric guarantee offers precise, verified insights into the boundary conditions for specific class transitions---a semifactual understanding of why a classification persists against a particular alternative.

Our empirical evaluations across image classification (MNIST, GTSRB, EMNIST) and regression (TaxiNet) tasks demonstrate that \vitax{} achieves higher fidelity, greater minimality, and dramatically faster runtimes compared to existing formal and heuristic XAI methods, while maintaining formal guarantees. By bridging sensitivity-driven heuristics with reachability-based verification in $\mathcal{O}(\log_2(N))$ oracle calls, \vitax{} offers a practical and scalable approach to formally grounded, targeted explanations.

Several directions remain for future work. First, we plan to explore strategies for identifying collectively influential features that may not individually rank as most sensitive but jointly approach decision boundaries. Second, we aim to extend \vitax{} to diverse data modalities, including natural language and time-series data. Finally, in close collaboration with domain experts, we intend to apply \vitax{} to safety-critical domains such as medical image analysis, validating its practical utility for model debugging, trust calibration, and responsible AI deployment.

\begin{acks}
This work was supported in part by the National Science Foundation under grants 2443803, 2220401, and 2427711. Any opinions, findings, and conclusions or recommendations expressed in this material are those of the authors and do not necessarily reflect the views of the National Science Foundation.
\end{acks}

\bibliographystyle{ACM-Reference-Format}
\bibliography{references, references_manual, verification}


\begin{thebibliography}{75}


\ifx \showCODEN    \undefined \def \showCODEN     #1{\unskip}     \fi
\ifx \showDOI      \undefined \def \showDOI       #1{#1}\fi
\ifx \showISBNx    \undefined \def \showISBNx     #1{\unskip}     \fi
\ifx \showISBNxiii \undefined \def \showISBNxiii  #1{\unskip}     \fi
\ifx \showISSN     \undefined \def \showISSN      #1{\unskip}     \fi
\ifx \showLCCN     \undefined \def \showLCCN      #1{\unskip}     \fi
\ifx \shownote     \undefined \def \shownote      #1{#1}          \fi
\ifx \showarticletitle \undefined \def \showarticletitle #1{#1}   \fi
\ifx \showURL      \undefined \def \showURL       {\relax}        \fi
\providecommand\bibfield[2]{#2}
\providecommand\bibinfo[2]{#2}
\providecommand\natexlab[1]{#1}
\providecommand\showeprint[2][]{arXiv:#2}

\bibitem[An et~al\mbox{.}(2024)]%
        {an2024formal}
\bibfield{author}{\bibinfo{person}{Ziyan An}, \bibinfo{person}{Taylor~T Johnson}, {and} \bibinfo{person}{Meiyi Ma}.} \bibinfo{year}{2024}\natexlab{}.
\newblock \showarticletitle{Formal logic enabled personalized federated learning through property inference}. In \bibinfo{booktitle}{\emph{Proceedings of the AAAI Conference on Artificial Intelligence}}, Vol.~\bibinfo{volume}{38}. \bibinfo{pages}{10882--10890}.
\newblock


\bibitem[Anderson et~al\mbox{.}(2019)]%
        {anderson2019optimization}
\bibfield{author}{\bibinfo{person}{Greg Anderson}, \bibinfo{person}{Shankara Pailoor}, \bibinfo{person}{Isil Dillig}, {and} \bibinfo{person}{Swarat Chaudhuri}.} \bibinfo{year}{2019}\natexlab{}.
\newblock \showarticletitle{Optimization and abstraction: a synergistic approach for analyzing neural network robustness}. In \bibinfo{booktitle}{\emph{Proceedings of the 40th ACM SIGPLAN conference on programming language design and implementation}}. \bibinfo{pages}{731--744}.
\newblock


\bibitem[Aryal and Keane(2023)]%
        {aryal2023even}
\bibfield{author}{\bibinfo{person}{Saugat Aryal} {and} \bibinfo{person}{Mark~T Keane}.} \bibinfo{year}{2023}\natexlab{}.
\newblock \showarticletitle{Even if explanations: Prior work, desiderata \& benchmarks for semi-factual xai}.
\newblock \bibinfo{journal}{\emph{arXiv preprint arXiv:2301.11970}} (\bibinfo{year}{2023}).
\newblock


\bibitem[Chakraborty et~al\mbox{.}(2021)]%
        {chakraborty2021survey}
\bibfield{author}{\bibinfo{person}{Anirban Chakraborty}, \bibinfo{person}{Manaar Alam}, \bibinfo{person}{Vishal Dey}, \bibinfo{person}{Anupam Chattopadhyay}, {and} \bibinfo{person}{Debdeep Mukhopadhyay}.} \bibinfo{year}{2021}\natexlab{}.
\newblock \showarticletitle{A survey on adversarial attacks and defences}.
\newblock \bibinfo{journal}{\emph{CAAI Transactions on Intelligence Technology}} \bibinfo{volume}{6}, \bibinfo{number}{1} (\bibinfo{year}{2021}), \bibinfo{pages}{25--45}.
\newblock


\bibitem[Chowdhury et~al\mbox{.}(2025)]%
        {chowdhury2025looking}
\bibfield{author}{\bibinfo{person}{Townim~Faisal Chowdhury}, \bibinfo{person}{Vu~Minh~Hieu Phan}, \bibinfo{person}{Kewen Liao}, \bibinfo{person}{Nanyu Dong}, \bibinfo{person}{Minh-Son To}, \bibinfo{person}{Anton Hengel}, \bibinfo{person}{Johan Verjans}, {and} \bibinfo{person}{Zhibin Liao}.} \bibinfo{year}{2025}\natexlab{}.
\newblock \showarticletitle{Looking in the mirror: A faithful counterfactual explanation method for interpreting deep image classification models}.
\newblock \bibinfo{journal}{\emph{arXiv preprint arXiv:2509.16822}} (\bibinfo{year}{2025}).
\newblock


\bibitem[Cohen et~al\mbox{.}(2017)]%
        {cohen2017emnist}
\bibfield{author}{\bibinfo{person}{Gregory Cohen}, \bibinfo{person}{Saeed Afshar}, \bibinfo{person}{Jonathan Tapson}, {and} \bibinfo{person}{Andre Van~Schaik}.} \bibinfo{year}{2017}\natexlab{}.
\newblock \showarticletitle{EMNIST: Extending MNIST to handwritten letters}. In \bibinfo{booktitle}{\emph{2017 international joint conference on neural networks (IJCNN)}}. IEEE, \bibinfo{pages}{2921--2926}.
\newblock


\bibitem[Darwiche and Hirth(2020)]%
        {darwiche_reasons_2020}
\bibfield{author}{\bibinfo{person}{Adnan Darwiche} {and} \bibinfo{person}{Auguste Hirth}.} \bibinfo{year}{2020}\natexlab{}.
\newblock \bibinfo{title}{On {The} {Reasons} {Behind} {Decisions}}.
\newblock
\newblock
\urldef\tempurl%
\url{https://arxiv-org.proxy.library.vanderbilt.edu/abs/2002.09284v1}
\showURL{%
\tempurl}


\bibitem[Ehlers(2017)]%
        {ehlers2017formal}
\bibfield{author}{\bibinfo{person}{Ruediger Ehlers}.} \bibinfo{year}{2017}\natexlab{}.
\newblock \showarticletitle{Formal verification of piece-wise linear feed-forward neural networks}. In \bibinfo{booktitle}{\emph{Automated Technology for Verification and Analysis: 15th International Symposium, ATVA 2017, Pune, India, October 3--6, 2017, Proceedings 15}}. Springer, \bibinfo{pages}{269--286}.
\newblock


\bibitem[Elboher et~al\mbox{.}(2022)]%
        {elboher2022neural}
\bibfield{author}{\bibinfo{person}{Yizhak~Yisrael Elboher}, \bibinfo{person}{Elazar Cohen}, {and} \bibinfo{person}{Guy Katz}.} \bibinfo{year}{2022}\natexlab{}.
\newblock \showarticletitle{Neural network verification using residual reasoning}. In \bibinfo{booktitle}{\emph{Software Engineering and Formal Methods: 20th International Conference, SEFM 2022, Berlin, Germany, September 26--30, 2022, Proceedings}}. Springer, \bibinfo{pages}{173--189}.
\newblock


\bibitem[Fern{\'a}ndez et~al\mbox{.}(2022)]%
        {fernandez2022explanation}
\bibfield{author}{\bibinfo{person}{Rub{\'e}n~R Fern{\'a}ndez}, \bibinfo{person}{Isaac~Martin de Diego}, \bibinfo{person}{Javier~M Moguerza}, {and} \bibinfo{person}{Francisco Herrera}.} \bibinfo{year}{2022}\natexlab{}.
\newblock \showarticletitle{Explanation sets: A general framework for machine learning explainability}.
\newblock \bibinfo{journal}{\emph{Information Sciences}}  \bibinfo{volume}{617} (\bibinfo{year}{2022}), \bibinfo{pages}{464--481}.
\newblock


\bibitem[Ferrari et~al\mbox{.}(2022)]%
        {ferrari2022complete}
\bibfield{author}{\bibinfo{person}{Claudio Ferrari}, \bibinfo{person}{Mark~Niklas Muller}, \bibinfo{person}{Nikola Jovanovic}, {and} \bibinfo{person}{Martin Vechev}.} \bibinfo{year}{2022}\natexlab{}.
\newblock \showarticletitle{Complete verification via multi-neuron relaxation guided branch-and-bound}.
\newblock \bibinfo{journal}{\emph{arXiv preprint arXiv:2205.00263}} (\bibinfo{year}{2022}).
\newblock


\bibitem[{FICO Community}(2019)]%
        {fico2019challenge}
\bibfield{author}{\bibinfo{person}{{FICO Community}}.} \bibinfo{year}{2019}\natexlab{}.
\newblock \bibinfo{title}{Explainable Machine Learning Challenge}.
\newblock \bibinfo{howpublished}{\url{https://community.fico.com/s/explainable-machine-learning-challenge}}.
\newblock
\newblock
\shownote{Accessed: 2025-09-30}.


\bibitem[Gainer-Dewar and Vera-Licona(2017)]%
        {gainer2017minimal}
\bibfield{author}{\bibinfo{person}{Andrew Gainer-Dewar} {and} \bibinfo{person}{Paola Vera-Licona}.} \bibinfo{year}{2017}\natexlab{}.
\newblock \showarticletitle{The minimal hitting set generation problem: algorithms and computation}.
\newblock \bibinfo{journal}{\emph{SIAM Journal on Discrete Mathematics}} \bibinfo{volume}{31}, \bibinfo{number}{1} (\bibinfo{year}{2017}), \bibinfo{pages}{63--100}.
\newblock


\bibitem[Gehr et~al\mbox{.}(2018)]%
        {Gehr2018ai2}
\bibfield{author}{\bibinfo{person}{Timon Gehr}, \bibinfo{person}{Matthew Mirman}, \bibinfo{person}{Dana Drachsler{-}Cohen}, \bibinfo{person}{Petar Tsankov}, \bibinfo{person}{Swarat Chaudhuri}, {and} \bibinfo{person}{Martin~T. Vechev}.} \bibinfo{year}{2018}\natexlab{}.
\newblock \showarticletitle{{AI2:} Safety and Robustness Certification of Neural Networks with Abstract Interpretation}. In \bibinfo{booktitle}{\emph{{SP}}}. \bibinfo{publisher}{{IEEE} Computer Society}, \bibinfo{pages}{3--18}.
\newblock
\urldef\tempurl%
\url{https://doi.org/10.1109/SP.2018.00058}
\showDOI{\tempurl}


\bibitem[Geng et~al\mbox{.}(2023)]%
        {geng2023towards}
\bibfield{author}{\bibinfo{person}{Chuqin Geng}, \bibinfo{person}{Nham Le}, \bibinfo{person}{Xiaojie Xu}, \bibinfo{person}{Zhaoyue Wang}, \bibinfo{person}{Arie Gurfinkel}, {and} \bibinfo{person}{Xujie Si}.} \bibinfo{year}{2023}\natexlab{}.
\newblock \showarticletitle{Towards reliable neural specifications}. In \bibinfo{booktitle}{\emph{International Conference on Machine Learning}}. PMLR, \bibinfo{pages}{11196--11212}.
\newblock


\bibitem[Ghorbani et~al\mbox{.}(2019)]%
        {ghorbani2019interpretation}
\bibfield{author}{\bibinfo{person}{Amirata Ghorbani}, \bibinfo{person}{Abubakar Abid}, {and} \bibinfo{person}{James Zou}.} \bibinfo{year}{2019}\natexlab{}.
\newblock \showarticletitle{Interpretation of neural networks is fragile}. In \bibinfo{booktitle}{\emph{Proceedings of the AAAI conference on artificial intelligence}}, Vol.~\bibinfo{volume}{33}. \bibinfo{pages}{3681--3688}.
\newblock


\bibitem[Guidotti(2024)]%
        {guidotti2024counterfactual}
\bibfield{author}{\bibinfo{person}{Riccardo Guidotti}.} \bibinfo{year}{2024}\natexlab{}.
\newblock \showarticletitle{Counterfactual explanations and how to find them: literature review and benchmarking}.
\newblock \bibinfo{journal}{\emph{Data Mining and Knowledge Discovery}} \bibinfo{volume}{38}, \bibinfo{number}{5} (\bibinfo{year}{2024}), \bibinfo{pages}{2770--2824}.
\newblock


\bibitem[Hashemi et~al\mbox{.}(2025)]%
        {hashemi2025probabilistic}
\bibfield{author}{\bibinfo{person}{Navid Hashemi}, \bibinfo{person}{Samuel Sasaki}, \bibinfo{person}{Diego~Manzanas Lopez}, \bibinfo{person}{Ipek Oguz}, \bibinfo{person}{Meiyi Ma}, {and} \bibinfo{person}{Taylor~T Johnson}.} \bibinfo{year}{2025}\natexlab{}.
\newblock \showarticletitle{Probabilistic Robustness Analysis in High Dimensional Space: Application to Semantic Segmentation Network}.
\newblock \bibinfo{journal}{\emph{arXiv preprint arXiv:2509.11838}} (\bibinfo{year}{2025}).
\newblock


\bibitem[Haufe et~al\mbox{.}(2024)]%
        {haufe2024explainable}
\bibfield{author}{\bibinfo{person}{Stefan Haufe}, \bibinfo{person}{Rick Wilming}, \bibinfo{person}{Benedict Clark}, \bibinfo{person}{Rustam Zhumagambetov}, \bibinfo{person}{Danny Panknin}, {and} \bibinfo{person}{Ahc{\`e}ne Boubekki}.} \bibinfo{year}{2024}\natexlab{}.
\newblock \showarticletitle{Explainable AI needs formal notions of explanation correctness}.
\newblock \bibinfo{journal}{\emph{arXiv preprint arXiv:2409.14590}} (\bibinfo{year}{2024}).
\newblock


\bibitem[Huang et~al\mbox{.}(2020)]%
        {huang2020survey}
\bibfield{author}{\bibinfo{person}{Xiaowei Huang}, \bibinfo{person}{Daniel Kroening}, \bibinfo{person}{Wenjie Ruan}, \bibinfo{person}{James Sharp}, \bibinfo{person}{Youcheng Sun}, \bibinfo{person}{Emese Thamo}, \bibinfo{person}{Min Wu}, {and} \bibinfo{person}{Xinping Yi}.} \bibinfo{year}{2020}\natexlab{}.
\newblock \showarticletitle{A survey of safety and trustworthiness of deep neural networks: Verification, testing, adversarial attack and defence, and interpretability}.
\newblock \bibinfo{journal}{\emph{Computer Science Review}}  \bibinfo{volume}{37} (\bibinfo{year}{2020}), \bibinfo{pages}{100270}.
\newblock


\bibitem[Huang et~al\mbox{.}(2017)]%
        {huang2017safety}
\bibfield{author}{\bibinfo{person}{Xiaowei Huang}, \bibinfo{person}{Marta Kwiatkowska}, \bibinfo{person}{Sen Wang}, {and} \bibinfo{person}{Min Wu}.} \bibinfo{year}{2017}\natexlab{}.
\newblock \showarticletitle{Safety verification of deep neural networks}. In \bibinfo{booktitle}{\emph{Computer Aided Verification: 29th International Conference, CAV 2017, Heidelberg, Germany, July 24-28, 2017, Proceedings, Part I 30}}. Springer, \bibinfo{pages}{3--29}.
\newblock


\bibitem[Ignatiev and Marques-Silva(2021)]%
        {ignatiev2021sat}
\bibfield{author}{\bibinfo{person}{Alexey Ignatiev} {and} \bibinfo{person}{Joao Marques-Silva}.} \bibinfo{year}{2021}\natexlab{}.
\newblock \showarticletitle{SAT-based rigorous explanations for decision lists}. In \bibinfo{booktitle}{\emph{Theory and Applications of Satisfiability Testing--SAT 2021: 24th International Conference, Barcelona, Spain, July 5-9, 2021, Proceedings 24}}. Springer, \bibinfo{pages}{251--269}.
\newblock


\bibitem[Ignatiev et~al\mbox{.}(2019)]%
        {ignatiev2019abduction}
\bibfield{author}{\bibinfo{person}{Alexey Ignatiev}, \bibinfo{person}{Nina Narodytska}, {and} \bibinfo{person}{Joao Marques-Silva}.} \bibinfo{year}{2019}\natexlab{}.
\newblock \showarticletitle{Abduction-based explanations for machine learning models}. In \bibinfo{booktitle}{\emph{Proceedings of the AAAI Conference on Artificial Intelligence}}, Vol.~\bibinfo{volume}{33}. \bibinfo{pages}{1511--1519}.
\newblock


\bibitem[Jiang et~al\mbox{.}(2023)]%
        {jiang_formalising_2023}
\bibfield{author}{\bibinfo{person}{Junqi Jiang}, \bibinfo{person}{Francesco Leofante}, \bibinfo{person}{Antonio Rago}, {and} \bibinfo{person}{Francesca Toni}.} \bibinfo{year}{2023}\natexlab{}.
\newblock \showarticletitle{Formalising the {Robustness} of {Counterfactual} {Explanations} for {Neural} {Networks}}.
\newblock \bibinfo{journal}{\emph{Proceedings of the AAAI Conference on Artificial Intelligence}} \bibinfo{volume}{37}, \bibinfo{number}{12} (\bibinfo{date}{June} \bibinfo{year}{2023}), \bibinfo{pages}{14901--14909}.
\newblock
\showISSN{2374-3468}
\urldef\tempurl%
\url{https://doi.org/10.1609/aaai.v37i12.26740}
\showDOI{\tempurl}


\bibitem[Jiang et~al\mbox{.}(2024)]%
        {jiang2024robust}
\bibfield{author}{\bibinfo{person}{Junqi Jiang}, \bibinfo{person}{Francesco Leofante}, \bibinfo{person}{Antonio Rago}, {and} \bibinfo{person}{Francesca Toni}.} \bibinfo{year}{2024}\natexlab{}.
\newblock \showarticletitle{Robust counterfactual explanations in machine learning: A survey}.
\newblock \bibinfo{journal}{\emph{arXiv preprint arXiv:2402.01928}} (\bibinfo{year}{2024}).
\newblock


\bibitem[Julian et~al\mbox{.}(2020)]%
        {julian2020validation}
\bibfield{author}{\bibinfo{person}{Kyle~D Julian}, \bibinfo{person}{Ritchie Lee}, {and} \bibinfo{person}{Mykel~J Kochenderfer}.} \bibinfo{year}{2020}\natexlab{}.
\newblock \showarticletitle{Validation of image-based neural network controllers through adaptive stress testing}. In \bibinfo{booktitle}{\emph{2020 IEEE 23rd international conference on intelligent transportation systems (ITSC)}}. IEEE, \bibinfo{pages}{1--7}.
\newblock


\bibitem[Karimi et~al\mbox{.}(2020)]%
        {karimi_algorithmic_2020}
\bibfield{author}{\bibinfo{person}{Amir-Hossein Karimi}, \bibinfo{person}{Bernhard Schölkopf}, {and} \bibinfo{person}{Isabel Valera}.} \bibinfo{year}{2020}\natexlab{}.
\newblock \bibinfo{title}{Algorithmic {Recourse}: from {Counterfactual} {Explanations} to {Interventions}}.
\newblock
\newblock
\urldef\tempurl%
\url{http://arxiv.org/abs/2002.06278}
\showURL{%
\tempurl}
\newblock
\shownote{arXiv:2002.06278 [cs, stat]}.


\bibitem[Katz et~al\mbox{.}(2017)]%
        {katz2017reluplex}
\bibfield{author}{\bibinfo{person}{Guy Katz}, \bibinfo{person}{Clark Barrett}, \bibinfo{person}{David~L Dill}, \bibinfo{person}{Kyle Julian}, {and} \bibinfo{person}{Mykel~J Kochenderfer}.} \bibinfo{year}{2017}\natexlab{}.
\newblock \showarticletitle{Reluplex: An efficient SMT solver for verifying deep neural networks}. In \bibinfo{booktitle}{\emph{International Conference on Computer Aided Verification}}. Springer, \bibinfo{pages}{97--117}.
\newblock


\bibitem[Katz et~al\mbox{.}(2019)]%
        {katz2019marabou}
\bibfield{author}{\bibinfo{person}{Guy Katz}, \bibinfo{person}{Derek~A. Huang}, \bibinfo{person}{Duligur Ibeling}, \bibinfo{person}{Kyle Julian}, \bibinfo{person}{Christopher Lazarus}, \bibinfo{person}{Rachel Lim}, \bibinfo{person}{Parth Shah}, \bibinfo{person}{Shantanu Thakoor}, \bibinfo{person}{Haoze Wu}, \bibinfo{person}{Aleksandar Zelji{\'{c}}}, \bibinfo{person}{David~L. Dill}, \bibinfo{person}{Mykel~J. Kochenderfer}, {and} \bibinfo{person}{Clark Barrett}.} \bibinfo{year}{2019}\natexlab{}.
\newblock \showarticletitle{The Marabou Framework for Verification and Analysis of Deep Neural Networks}. In \bibinfo{booktitle}{\emph{Computer Aided Verification}}, \bibfield{editor}{\bibinfo{person}{Isil Dillig} {and} \bibinfo{person}{Serdar Tasiran}} (Eds.). \bibinfo{publisher}{Springer International Publishing}, \bibinfo{address}{Cham}, \bibinfo{pages}{443--452}.
\newblock
\showISBNx{978-3-030-25540-4}


\bibitem[Kenny and Huang(2023)]%
        {kenny2023utility}
\bibfield{author}{\bibinfo{person}{Eoin Kenny} {and} \bibinfo{person}{Weipeng Huang}.} \bibinfo{year}{2023}\natexlab{}.
\newblock \showarticletitle{The utility of “even if” semifactual explanation to optimise positive outcomes}.
\newblock \bibinfo{journal}{\emph{Advances in Neural Information Processing Systems}}  \bibinfo{volume}{36} (\bibinfo{year}{2023}), \bibinfo{pages}{52907--52935}.
\newblock


\bibitem[Kenny and Keane(2021)]%
        {kenny2021generating}
\bibfield{author}{\bibinfo{person}{Eoin~M Kenny} {and} \bibinfo{person}{Mark~T Keane}.} \bibinfo{year}{2021}\natexlab{}.
\newblock \showarticletitle{On generating plausible counterfactual and semi-factual explanations for deep learning}. In \bibinfo{booktitle}{\emph{Proceedings of the AAAI Conference on Artificial Intelligence}}, Vol.~\bibinfo{volume}{35}. \bibinfo{pages}{11575--11585}.
\newblock


\bibitem[Key(2025)]%
        {key2025sat}
\bibfield{author}{\bibinfo{person}{Hojer Key}.} \bibinfo{year}{2025}\natexlab{}.
\newblock \showarticletitle{A SAT-centered XAI method for Deep Learning based Video Understanding}.
\newblock \bibinfo{journal}{\emph{arXiv preprint arXiv:2503.23870}} (\bibinfo{year}{2025}).
\newblock


\bibitem[LeCun et~al\mbox{.}(2009)]%
        {lecun2009mnist}
\bibfield{author}{\bibinfo{person}{Yann LeCun}, \bibinfo{person}{Corinna Cortes}, {and} \bibinfo{person}{Christopher~JC Burges}.} \bibinfo{year}{2009}\natexlab{}.
\newblock \showarticletitle{The MNIST database of handwritten digits (2010)}.
\newblock \bibinfo{journal}{\emph{URL http://yann. lecun. com/exdb/mnist}} (\bibinfo{year}{2009}), \bibinfo{pages}{7}.
\newblock


\bibitem[Leofante and Lomuscio(2023)]%
        {leofante2023robust}
\bibfield{author}{\bibinfo{person}{Francesco Leofante} {and} \bibinfo{person}{Alessio Lomuscio}.} \bibinfo{year}{2023}\natexlab{}.
\newblock \showarticletitle{Robust explanations for human-neural multi-agent systems with formal verification}. In \bibinfo{booktitle}{\emph{European Conference on Multi-Agent Systems}}. Springer, \bibinfo{pages}{244--262}.
\newblock


\bibitem[Linardatos et~al\mbox{.}(2020)]%
        {linardatos2020explainable}
\bibfield{author}{\bibinfo{person}{Pantelis Linardatos}, \bibinfo{person}{Vasilis Papastefanopoulos}, {and} \bibinfo{person}{Sotiris Kotsiantis}.} \bibinfo{year}{2020}\natexlab{}.
\newblock \showarticletitle{Explainable ai: A review of machine learning interpretability methods}.
\newblock \bibinfo{journal}{\emph{Entropy}} \bibinfo{volume}{23}, \bibinfo{number}{1} (\bibinfo{year}{2020}), \bibinfo{pages}{18}.
\newblock


\bibitem[Liu et~al\mbox{.}(2021)]%
        {liu2021algorithms}
\bibfield{author}{\bibinfo{person}{Changliu Liu}, \bibinfo{person}{Tomer Arnon}, \bibinfo{person}{Christopher Lazarus}, \bibinfo{person}{Christopher Strong}, \bibinfo{person}{Clark Barrett}, \bibinfo{person}{Mykel~J Kochenderfer}, {et~al\mbox{.}}} \bibinfo{year}{2021}\natexlab{}.
\newblock \showarticletitle{Algorithms for verifying deep neural networks}.
\newblock \bibinfo{journal}{\emph{Foundations and Trends{\textregistered} in Optimization}} \bibinfo{volume}{4}, \bibinfo{number}{3-4} (\bibinfo{year}{2021}), \bibinfo{pages}{244--404}.
\newblock


\bibitem[Lopez et~al\mbox{.}(2023)]%
        {lopez2023nnv}
\bibfield{author}{\bibinfo{person}{Diego~Manzanas Lopez}, \bibinfo{person}{Sung~Woo Choi}, \bibinfo{person}{Hoang-Dung Tran}, {and} \bibinfo{person}{Taylor~T Johnson}.} \bibinfo{year}{2023}\natexlab{}.
\newblock \showarticletitle{NNV 2.0: the neural network verification tool}. In \bibinfo{booktitle}{\emph{International Conference on Computer Aided Verification}}. Springer, \bibinfo{pages}{397--412}.
\newblock


\bibitem[Lundberg and Lee(2017)]%
        {lundberg_unified_2017}
\bibfield{author}{\bibinfo{person}{Scott~M Lundberg} {and} \bibinfo{person}{Su-In Lee}.} \bibinfo{year}{2017}\natexlab{}.
\newblock \showarticletitle{A {Unified} {Approach} to {Interpreting} {Model} {Predictions}}.
\newblock \bibinfo{journal}{\emph{Advances in neural information processing systems}}  \bibinfo{volume}{30} (\bibinfo{year}{2017}).
\newblock


\bibitem[Ma et~al\mbox{.}(2020)]%
        {ma2020stlnet}
\bibfield{author}{\bibinfo{person}{Meiyi Ma}, \bibinfo{person}{Ji Gao}, \bibinfo{person}{Lu Feng}, {and} \bibinfo{person}{John Stankovic}.} \bibinfo{year}{2020}\natexlab{}.
\newblock \showarticletitle{STLnet: Signal temporal logic enforced multivariate recurrent neural networks}.
\newblock \bibinfo{journal}{\emph{Advances in Neural Information Processing Systems}}  \bibinfo{volume}{33} (\bibinfo{year}{2020}), \bibinfo{pages}{14604--14614}.
\newblock


\bibitem[Ma et~al\mbox{.}(2021)]%
        {ma2021predictive}
\bibfield{author}{\bibinfo{person}{Meiyi Ma}, \bibinfo{person}{John Stankovic}, \bibinfo{person}{Ezio Bartocci}, {and} \bibinfo{person}{Lu Feng}.} \bibinfo{year}{2021}\natexlab{}.
\newblock \showarticletitle{Predictive monitoring with logic-calibrated uncertainty for cyber-physical systems}.
\newblock \bibinfo{journal}{\emph{ACM Transactions on Embedded Computing Systems (TECS)}} \bibinfo{volume}{20}, \bibinfo{number}{5s} (\bibinfo{year}{2021}), \bibinfo{pages}{1--25}.
\newblock


\bibitem[Malfa et~al\mbox{.}(2021)]%
        {malfa_guaranteed_2021}
\bibfield{author}{\bibinfo{person}{Emanuele~La Malfa}, \bibinfo{person}{Rhiannon Michelmore}, \bibinfo{person}{Agnieszka~M. Zbrzezny}, \bibinfo{person}{Nicola Paoletti}, {and} \bibinfo{person}{Marta Kwiatkowska}.} \bibinfo{year}{2021}\natexlab{}.
\newblock \showarticletitle{On {Guaranteed} {Optimal} {Robust} {Explanations} for {NLP} {Models}}, Vol.~\bibinfo{volume}{3}. \bibinfo{pages}{2658--2665}.
\newblock
\urldef\tempurl%
\url{https://doi.org/10.24963/ijcai.2021/366}
\showDOI{\tempurl}
\newblock
\shownote{ISSN: 1045-0823}.


\bibitem[Marques-Silva and Ignatiev(2022)]%
        {marques-silva_delivering_2022}
\bibfield{author}{\bibinfo{person}{Joao Marques-Silva} {and} \bibinfo{person}{Alexey Ignatiev}.} \bibinfo{year}{2022}\natexlab{}.
\newblock \showarticletitle{Delivering {Trustworthy} {AI} through {Formal} {XAI}}.
\newblock \bibinfo{journal}{\emph{Proceedings of the AAAI Conference on Artificial Intelligence}} \bibinfo{volume}{36}, \bibinfo{number}{11} (\bibinfo{date}{June} \bibinfo{year}{2022}), \bibinfo{pages}{12342--12350}.
\newblock
\showISSN{2374-3468}
\urldef\tempurl%
\url{https://doi.org/10.1609/aaai.v36i11.21499}
\showDOI{\tempurl}
\newblock
\shownote{Number: 11}.


\bibitem[Miller(2019)]%
        {miller2019explanation}
\bibfield{author}{\bibinfo{person}{Tim Miller}.} \bibinfo{year}{2019}\natexlab{}.
\newblock \showarticletitle{Explanation in artificial intelligence: Insights from the social sciences}.
\newblock \bibinfo{journal}{\emph{Artificial intelligence}}  \bibinfo{volume}{267} (\bibinfo{year}{2019}), \bibinfo{pages}{1--38}.
\newblock


\bibitem[Montavon et~al\mbox{.}(2018)]%
        {montavon_methods_2018}
\bibfield{author}{\bibinfo{person}{Grégoire Montavon}, \bibinfo{person}{Wojciech Samek}, {and} \bibinfo{person}{Klaus-Robert Müller}.} \bibinfo{year}{2018}\natexlab{}.
\newblock \showarticletitle{Methods for {Interpreting} and {Understanding} {Deep} {Neural} {Networks}}.
\newblock \bibinfo{journal}{\emph{Digital Signal Processing}}  \bibinfo{volume}{73} (\bibinfo{date}{Feb.} \bibinfo{year}{2018}), \bibinfo{pages}{1--15}.
\newblock
\showISSN{10512004}
\urldef\tempurl%
\url{https://doi.org/10.1016/j.dsp.2017.10.011}
\showDOI{\tempurl}
\newblock
\shownote{arXiv:1706.07979 [cs, stat]}.


\bibitem[M{\"u}ller et~al\mbox{.}(2022)]%
        {muller2022prima}
\bibfield{author}{\bibinfo{person}{Mark~Niklas M{\"u}ller}, \bibinfo{person}{Gleb Makarchuk}, \bibinfo{person}{Gagandeep Singh}, \bibinfo{person}{Markus P{\"u}schel}, {and} \bibinfo{person}{Martin Vechev}.} \bibinfo{year}{2022}\natexlab{}.
\newblock \showarticletitle{PRIMA: general and precise neural network certification via scalable convex hull approximations}.
\newblock \bibinfo{journal}{\emph{Proceedings of the ACM on Programming Languages}} \bibinfo{volume}{6}, \bibinfo{number}{POPL} (\bibinfo{year}{2022}), \bibinfo{pages}{1--33}.
\newblock


\bibitem[Paul et~al\mbox{.}(2024)]%
        {paul2024formal}
\bibfield{author}{\bibinfo{person}{Sushmita Paul}, \bibinfo{person}{Jinqiang Yu}, \bibinfo{person}{Jip~J Dekker}, \bibinfo{person}{Alexey Ignatiev}, {and} \bibinfo{person}{Peter~J Stuckey}.} \bibinfo{year}{2024}\natexlab{}.
\newblock \showarticletitle{Formal Explanations for Neuro-Symbolic AI}.
\newblock \bibinfo{journal}{\emph{arXiv preprint arXiv:2410.14219}} (\bibinfo{year}{2024}).
\newblock


\bibitem[Ribeiro et~al\mbox{.}(2016a)]%
        {ribeiro2016should}
\bibfield{author}{\bibinfo{person}{Marco~Tulio Ribeiro}, \bibinfo{person}{Sameer Singh}, {and} \bibinfo{person}{Carlos Guestrin}.} \bibinfo{year}{2016}\natexlab{a}.
\newblock \showarticletitle{" Why should i trust you?" Explaining the predictions of any classifier}. In \bibinfo{booktitle}{\emph{Proceedings of the 22nd ACM SIGKDD international conference on knowledge discovery and data mining}}. \bibinfo{pages}{1135--1144}.
\newblock


\bibitem[Ribeiro et~al\mbox{.}(2016b)]%
        {ribeiro_why_2016}
\bibfield{author}{\bibinfo{person}{Marco~Tulio Ribeiro}, \bibinfo{person}{Sameer Singh}, {and} \bibinfo{person}{Carlos Guestrin}.} \bibinfo{year}{2016}\natexlab{b}.
\newblock \bibinfo{title}{"{Why} {Should} {I} {Trust} {You}?": {Explaining} the {Predictions} of {Any} {Classifier}}.
\newblock
\newblock
\urldef\tempurl%
\url{http://arxiv.org/abs/1602.04938}
\showURL{%
\tempurl}
\newblock
\shownote{arXiv:1602.04938 [cs, stat]}.


\bibitem[Ribeiro et~al\mbox{.}(2018)]%
        {ribeiro_anchors_2018}
\bibfield{author}{\bibinfo{person}{Marco~Tulio Ribeiro}, \bibinfo{person}{Sameer Singh}, {and} \bibinfo{person}{Carlos Guestrin}.} \bibinfo{year}{2018}\natexlab{}.
\newblock \showarticletitle{Anchors: {High}-{Precision} {Model}-{Agnostic} {Explanations}}.
\newblock \bibinfo{journal}{\emph{Proceedings of the AAAI Conference on Artificial Intelligence}} \bibinfo{volume}{32}, \bibinfo{number}{1} (\bibinfo{date}{April} \bibinfo{year}{2018}).
\newblock
\showISSN{2374-3468}
\urldef\tempurl%
\url{https://doi.org/10.1609/aaai.v32i1.11491}
\showDOI{\tempurl}


\bibitem[Selvaraju et~al\mbox{.}(2020)]%
        {selvaraju_grad-cam_2020}
\bibfield{author}{\bibinfo{person}{Ramprasaath~R. Selvaraju}, \bibinfo{person}{Michael Cogswell}, \bibinfo{person}{Abhishek Das}, \bibinfo{person}{Ramakrishna Vedantam}, \bibinfo{person}{Devi Parikh}, {and} \bibinfo{person}{Dhruv Batra}.} \bibinfo{year}{2020}\natexlab{}.
\newblock \showarticletitle{Grad-{CAM}: {Visual} {Explanations} from {Deep} {Networks} via {Gradient}-based {Localization}}.
\newblock \bibinfo{journal}{\emph{International Journal of Computer Vision}} \bibinfo{volume}{128}, \bibinfo{number}{2} (\bibinfo{date}{Feb.} \bibinfo{year}{2020}), \bibinfo{pages}{336--359}.
\newblock
\showISSN{0920-5691, 1573-1405}
\urldef\tempurl%
\url{https://doi.org/10.1007/s11263-019-01228-7}
\showDOI{\tempurl}
\newblock
\shownote{arXiv:1610.02391 [cs]}.


\bibitem[Shih et~al\mbox{.}(2018)]%
        {shih_symbolic_2018}
\bibfield{author}{\bibinfo{person}{Andy Shih}, \bibinfo{person}{Arthur Choi}, {and} \bibinfo{person}{Adnan Darwiche}.} \bibinfo{year}{2018}\natexlab{}.
\newblock \bibinfo{title}{A {Symbolic} {Approach} to {Explaining} {Bayesian} {Network} {Classifiers}}.
\newblock
\newblock
\urldef\tempurl%
\url{https://arxiv-org.proxy.library.vanderbilt.edu/abs/1805.03364v1}
\showURL{%
\tempurl}


\bibitem[Shrikumar et~al\mbox{.}(2019)]%
        {shrikumar_learning_2019}
\bibfield{author}{\bibinfo{person}{Avanti Shrikumar}, \bibinfo{person}{Peyton Greenside}, {and} \bibinfo{person}{Anshul Kundaje}.} \bibinfo{year}{2019}\natexlab{}.
\newblock \bibinfo{title}{Learning {Important} {Features} {Through} {Propagating} {Activation} {Differences}}.
\newblock
\newblock
\urldef\tempurl%
\url{http://arxiv.org/abs/1704.02685}
\showURL{%
\tempurl}
\newblock
\shownote{arXiv:1704.02685 [cs]}.


\bibitem[Shrikumar et~al\mbox{.}(2016)]%
        {shrikumar_not_2016}
\bibfield{author}{\bibinfo{person}{Avanti Shrikumar}, \bibinfo{person}{Peyton Greenside}, \bibinfo{person}{Anna Shcherbina}, {and} \bibinfo{person}{Anshul Kundaje}.} \bibinfo{year}{2016}\natexlab{}.
\newblock \bibinfo{title}{Not {Just} a {Black} {Box}: {Learning} {Important} {Features} {Through} {Propagating} {Activation} {Differences}}.
\newblock
\newblock
\urldef\tempurl%
\url{https://arxiv.org/abs/1605.01713v3}
\showURL{%
\tempurl}


\bibitem[Simonyan et~al\mbox{.}(2014)]%
        {simonyan_deep_2014}
\bibfield{author}{\bibinfo{person}{Karen Simonyan}, \bibinfo{person}{Andrea Vedaldi}, {and} \bibinfo{person}{Andrew Zisserman}.} \bibinfo{year}{2014}\natexlab{}.
\newblock \bibinfo{title}{Deep {Inside} {Convolutional} {Networks}: {Visualising} {Image} {Classification} {Models} and {Saliency} {Maps}}.
\newblock
\newblock
\urldef\tempurl%
\url{http://arxiv.org/abs/1312.6034}
\showURL{%
\tempurl}
\newblock
\shownote{arXiv:1312.6034 [cs]}.


\bibitem[Singh et~al\mbox{.}(2019)]%
        {singh2019abstract}
\bibfield{author}{\bibinfo{person}{Gagandeep Singh}, \bibinfo{person}{Timon Gehr}, \bibinfo{person}{Markus P{\"u}schel}, {and} \bibinfo{person}{Martin Vechev}.} \bibinfo{year}{2019}\natexlab{}.
\newblock \showarticletitle{An abstract domain for certifying neural networks}.
\newblock \bibinfo{journal}{\emph{Proceedings of the ACM on Programming Languages}} \bibinfo{volume}{3}, \bibinfo{number}{POPL} (\bibinfo{year}{2019}), \bibinfo{pages}{1--30}.
\newblock


\bibitem[Springenberg et~al\mbox{.}(2015)]%
        {springenberg_striving_2015}
\bibfield{author}{\bibinfo{person}{Jost~Tobias Springenberg}, \bibinfo{person}{Alexey Dosovitskiy}, \bibinfo{person}{Thomas Brox}, {and} \bibinfo{person}{Martin Riedmiller}.} \bibinfo{year}{2015}\natexlab{}.
\newblock \bibinfo{title}{Striving for {Simplicity}: {The} {All} {Convolutional} {Net}}.
\newblock
\newblock
\urldef\tempurl%
\url{https://doi.org/10.48550/arXiv.1412.6806}
\showDOI{\tempurl}
\newblock
\shownote{arXiv:1412.6806 [cs]}.


\bibitem[Stallkamp et~al\mbox{.}(2012)]%
        {stallkamp2012man}
\bibfield{author}{\bibinfo{person}{Johannes Stallkamp}, \bibinfo{person}{Marc Schlipsing}, \bibinfo{person}{Jan Salmen}, {and} \bibinfo{person}{Christian Igel}.} \bibinfo{year}{2012}\natexlab{}.
\newblock \showarticletitle{Man vs. computer: Benchmarking machine learning algorithms for traffic sign recognition}.
\newblock \bibinfo{journal}{\emph{Neural networks}}  \bibinfo{volume}{32} (\bibinfo{year}{2012}), \bibinfo{pages}{323--332}.
\newblock


\bibitem[Sundararajan et~al\mbox{.}(2017)]%
        {sundararajan_axiomatic_2017}
\bibfield{author}{\bibinfo{person}{Mukund Sundararajan}, \bibinfo{person}{Ankur Taly}, {and} \bibinfo{person}{Qiqi Yan}.} \bibinfo{year}{2017}\natexlab{}.
\newblock \showarticletitle{Axiomatic {Attribution} for {Deep} {Networks}}. In \bibinfo{booktitle}{\emph{Proceedings of the 34th {International} {Conference} on {Machine} {Learning}}}. \bibinfo{publisher}{PMLR}, \bibinfo{pages}{3319--3328}.
\newblock
\urldef\tempurl%
\url{https://proceedings.mlr.press/v70/sundararajan17a.html}
\showURL{%
\tempurl}
\newblock
\shownote{ISSN: 2640-3498}.


\bibitem[Tjeng et~al\mbox{.}(2017)]%
        {tjeng2017mipverify}
\bibfield{author}{\bibinfo{person}{Vincent Tjeng}, \bibinfo{person}{Kai Xiao}, {and} \bibinfo{person}{Russ Tedrake}.} \bibinfo{year}{2017}\natexlab{}.
\newblock \showarticletitle{Evaluating Robustness of Neural Networks with Mixed Integer Programming}.
\newblock \bibinfo{journal}{\emph{arXiv preprint arXiv:1711.07356}} (\bibinfo{year}{2017}).
\newblock


\bibitem[Tran et~al\mbox{.}(2020)]%
        {tran2020verification}
\bibfield{author}{\bibinfo{person}{Hoang-Dung Tran}, \bibinfo{person}{Stanley Bak}, \bibinfo{person}{Weiming Xiang}, {and} \bibinfo{person}{Taylor~T Johnson}.} \bibinfo{year}{2020}\natexlab{}.
\newblock \showarticletitle{Verification of deep convolutional neural networks using imagestars}. In \bibinfo{booktitle}{\emph{International conference on computer aided verification}}. Springer, \bibinfo{pages}{18--42}.
\newblock


\bibitem[Tran et~al\mbox{.}(2019a)]%
        {tran2019star}
\bibfield{author}{\bibinfo{person}{Hoang-Dung Tran}, \bibinfo{person}{Diago Manzanas~Lopez}, \bibinfo{person}{Patrick Musau}, \bibinfo{person}{Xiaodong Yang}, \bibinfo{person}{Luan~Viet Nguyen}, \bibinfo{person}{Weiming Xiang}, {and} \bibinfo{person}{Taylor~T Johnson}.} \bibinfo{year}{2019}\natexlab{a}.
\newblock \showarticletitle{Star-based reachability analysis of deep neural networks}. In \bibinfo{booktitle}{\emph{Formal Methods--The Next 30 Years: Third World Congress, FM 2019, Porto, Portugal, October 7--11, 2019, Proceedings 3}}. Springer, \bibinfo{pages}{670--686}.
\newblock


\bibitem[Tran et~al\mbox{.}(2019b)]%
        {tran2019fm}
\bibfield{author}{\bibinfo{person}{Hoang-Dung Tran}, \bibinfo{person}{Patrick Musau}, \bibinfo{person}{Diego~Manzanas Lopez}, \bibinfo{person}{Xiaodong Yang}, \bibinfo{person}{Luan~Viet Nguyen}, \bibinfo{person}{Weiming Xiang}, {and} \bibinfo{person}{Taylor~T. Johnson}.} \bibinfo{year}{2019}\natexlab{b}.
\newblock \showarticletitle{Star-Based Reachability Analysis for Deep Neural Networks}. In \bibinfo{booktitle}{\emph{23rd International Symposium on Formal Methods (FM'19)}}. \bibinfo{publisher}{Springer International Publishing}.
\newblock


\bibitem[Tran et~al\mbox{.}(2021)]%
        {tran2021prefilter}
\bibfield{author}{\bibinfo{person}{Hoang-Dung Tran}, \bibinfo{person}{Neelanjana Pal}, \bibinfo{person}{Diego~Manzanas Lopez}, \bibinfo{person}{Patrick Musau}, \bibinfo{person}{Xiaodong Yang}, \bibinfo{person}{Luan~Viet Nguyen}, \bibinfo{person}{Weiming Xiang}, \bibinfo{person}{Stanley Bak}, {and} \bibinfo{person}{Taylor~T. Johnson}.} \bibinfo{year}{2021}\natexlab{}.
\newblock \showarticletitle{Verification of Piecewise Deep Neural Networks: A Star Set Approach with Zonotope Pre-Filter}.
\newblock \bibinfo{journal}{\emph{Form. Asp. Comput.}} \bibinfo{volume}{33}, \bibinfo{number}{4–5} (\bibinfo{date}{aug} \bibinfo{year}{2021}), \bibinfo{pages}{519–545}.
\newblock
\showISSN{0934-5043}
\urldef\tempurl%
\url{https://doi.org/10.1007/s00165-021-00553-4}
\showDOI{\tempurl}


\bibitem[Tran et~al\mbox{.}(2022)]%
        {tran_unsupervised_2022}
\bibfield{author}{\bibinfo{person}{Thien~Q Tran}, \bibinfo{person}{Kazuto Fukuchi}, \bibinfo{person}{Youhei Akimoto}, {and} \bibinfo{person}{Jun Sakuma}.} \bibinfo{year}{2022}\natexlab{}.
\newblock \showarticletitle{Unsupervised {Causal} {Binary} {Concepts} {Discovery} with {VAE} for {Black}-{Box} {Model} {Explanation}}.
\newblock \bibinfo{journal}{\emph{Proceedings of the AAAI Conference on Artificial Intelligence}} \bibinfo{volume}{36}, \bibinfo{number}{9} (\bibinfo{date}{June} \bibinfo{year}{2022}), \bibinfo{pages}{9614--9622}.
\newblock
\showISSN{2374-3468, 2159-5399}
\urldef\tempurl%
\url{https://doi.org/10.1609/aaai.v36i9.21195}
\showDOI{\tempurl}


\bibitem[Wang et~al\mbox{.}(2024)]%
        {wang2024microxercise}
\bibfield{author}{\bibinfo{person}{Hanchen~David Wang}, \bibinfo{person}{Nibraas Khan}, \bibinfo{person}{Anna Chen}, \bibinfo{person}{Nilanjan Sarkar}, \bibinfo{person}{Pamela Wisniewski}, {and} \bibinfo{person}{Meiyi Ma}.} \bibinfo{year}{2024}\natexlab{}.
\newblock \showarticletitle{MicroXercise: A Micro-Level Comparative and Explainable System for Remote Physical Therapy}. In \bibinfo{booktitle}{\emph{2024 IEEE/ACM Conference on Connected Health: Applications, Systems and Engineering Technologies (CHASE)}}. IEEE, \bibinfo{pages}{73--84}.
\newblock


\bibitem[Wang et~al\mbox{.}(2018a)]%
        {wang2018efficient}
\bibfield{author}{\bibinfo{person}{Shiqi Wang}, \bibinfo{person}{Kexin Pei}, \bibinfo{person}{Justin Whitehouse}, \bibinfo{person}{Junfeng Yang}, {and} \bibinfo{person}{Suman Jana}.} \bibinfo{year}{2018}\natexlab{a}.
\newblock \showarticletitle{Efficient formal safety analysis of neural networks}.
\newblock \bibinfo{journal}{\emph{Advances in neural information processing systems}}  \bibinfo{volume}{31} (\bibinfo{year}{2018}).
\newblock


\bibitem[Wang et~al\mbox{.}(2018b)]%
        {wang2018formal}
\bibfield{author}{\bibinfo{person}{Shiqi Wang}, \bibinfo{person}{Kexin Pei}, \bibinfo{person}{Justin Whitehouse}, \bibinfo{person}{Junfeng Yang}, {and} \bibinfo{person}{Suman Jana}.} \bibinfo{year}{2018}\natexlab{b}.
\newblock \showarticletitle{Formal security analysis of neural networks using symbolic intervals}. In \bibinfo{booktitle}{\emph{27th USENIX Security Symposium (USENIX Security 18)}}. \bibinfo{pages}{1599--1614}.
\newblock


\bibitem[Wei et~al\mbox{.}(2023)]%
        {wei2023convex}
\bibfield{author}{\bibinfo{person}{Dennis Wei}, \bibinfo{person}{Haoze Wu}, \bibinfo{person}{Min Wu}, \bibinfo{person}{Pin-Yu Chen}, \bibinfo{person}{Clark Barrett}, {and} \bibinfo{person}{Eitan Farchi}.} \bibinfo{year}{2023}\natexlab{}.
\newblock \showarticletitle{Convex bounds on the softmax function with applications to robustness verification}. In \bibinfo{booktitle}{\emph{International Conference on Artificial Intelligence and Statistics}}. PMLR, \bibinfo{pages}{6853--6878}.
\newblock


\bibitem[Wicker et~al\mbox{.}(2022)]%
        {wicker2022robust}
\bibfield{author}{\bibinfo{person}{Matthew Wicker}, \bibinfo{person}{Juyeon Heo}, \bibinfo{person}{Luca Costabello}, {and} \bibinfo{person}{Adrian Weller}.} \bibinfo{year}{2022}\natexlab{}.
\newblock \showarticletitle{Robust explanation constraints for neural networks}.
\newblock \bibinfo{journal}{\emph{arXiv preprint arXiv:2212.08507}} (\bibinfo{year}{2022}).
\newblock


\bibitem[Wu et~al\mbox{.}(2022)]%
        {wu2022scalable}
\bibfield{author}{\bibinfo{person}{Haoze Wu}, \bibinfo{person}{Clark Barrett}, \bibinfo{person}{Mahmood Sharif}, \bibinfo{person}{Nina Narodytska}, {and} \bibinfo{person}{Gagandeep Singh}.} \bibinfo{year}{2022}\natexlab{}.
\newblock \showarticletitle{Scalable verification of GNN-based job schedulers}.
\newblock \bibinfo{journal}{\emph{Proceedings of the ACM on Programming Languages}} \bibinfo{volume}{6}, \bibinfo{number}{OOPSLA2} (\bibinfo{year}{2022}), \bibinfo{pages}{1036--1065}.
\newblock


\bibitem[Wu et~al\mbox{.}(2023)]%
        {wu_verix_2023}
\bibfield{author}{\bibinfo{person}{Min Wu}, \bibinfo{person}{Haoze Wu}, {and} \bibinfo{person}{Clark Barrett}.} \bibinfo{year}{2023}\natexlab{}.
\newblock \showarticletitle{{VeriX}: {Towards} {Verified} {Explainability} of {Deep} {Neural} {Networks}}.
\newblock \bibinfo{journal}{\emph{Advances in Neural Information Processing Systems}}  \bibinfo{volume}{36} (\bibinfo{date}{Dec.} \bibinfo{year}{2023}), \bibinfo{pages}{22247--22268}.
\newblock
\urldef\tempurl%
\url{https://proceedings.neurips.cc/paper_files/paper/2023/hash/46907c2ff9fafd618095161d76461842-Abstract-Conference.html}
\showURL{%
\tempurl}


\bibitem[Zeiler and Fergus(2013)]%
        {zeiler_visualizing_2013}
\bibfield{author}{\bibinfo{person}{Matthew~D. Zeiler} {and} \bibinfo{person}{Rob Fergus}.} \bibinfo{year}{2013}\natexlab{}.
\newblock \bibinfo{title}{Visualizing and {Understanding} {Convolutional} {Networks}}.
\newblock
\newblock
\urldef\tempurl%
\url{https://doi.org/10.48550/arXiv.1311.2901}
\showDOI{\tempurl}
\newblock
\shownote{arXiv:1311.2901 [cs]}.


\bibitem[Zelazny et~al\mbox{.}(2022)]%
        {zelazny2022optimizing}
\bibfield{author}{\bibinfo{person}{Tom Zelazny}, \bibinfo{person}{Haoze Wu}, \bibinfo{person}{Clark Barrett}, {and} \bibinfo{person}{Guy Katz}.} \bibinfo{year}{2022}\natexlab{}.
\newblock \showarticletitle{On optimizing back-substitution methods for neural network verification}. In \bibinfo{booktitle}{\emph{2022 Formal Methods in Computer-Aided Design (FMCAD)}}. IEEE, \bibinfo{pages}{17--26}.
\newblock


\bibitem[Zhang et~al\mbox{.}(2018)]%
        {zhang2018efficient}
\bibfield{author}{\bibinfo{person}{Huan Zhang}, \bibinfo{person}{Tsui-Wei Weng}, \bibinfo{person}{Pin-Yu Chen}, \bibinfo{person}{Cho-Jui Hsieh}, {and} \bibinfo{person}{Luca Daniel}.} \bibinfo{year}{2018}\natexlab{}.
\newblock \showarticletitle{Efficient neural network robustness certification with general activation functions}.
\newblock \bibinfo{journal}{\emph{Advances in neural information processing systems}}  \bibinfo{volume}{31} (\bibinfo{year}{2018}).
\newblock


\bibitem[Zhang et~al\mbox{.}(2024)]%
        {zhang2024optimizing}
\bibfield{author}{\bibinfo{person}{Shiqiang Zhang}, \bibinfo{person}{Juan Campos}, \bibinfo{person}{Christian Feldmann}, \bibinfo{person}{David Walz}, \bibinfo{person}{Frederik Sandfort}, \bibinfo{person}{Miriam Mathea}, \bibinfo{person}{Calvin Tsay}, {and} \bibinfo{person}{Ruth Misener}.} \bibinfo{year}{2024}\natexlab{}.
\newblock \showarticletitle{Optimizing over trained GNNs via symmetry breaking}.
\newblock \bibinfo{journal}{\emph{Advances in Neural Information Processing Systems}}  \bibinfo{volume}{36} (\bibinfo{year}{2024}).
\newblock


\end{thebibliography}

\newpage
\clearpage
\section*{Appendix}
\appendix

\section{Overapproximation Methods for Reach Sets}\label{appendix:overapproximation}
\subsection{Definition of Reach Set and Overapproximation}

Given a function \( f: \mathbb{R}^n \rightarrow \mathbb{R}^m \), the reach set \( f(X) \) from a set of initial states \( X \subseteq \mathbb{R}^n \) represents all possible states in \( \mathbb{R}^m \) that the system can reach. An overapproximation \( \hat{f}(X) \) of this reach set is then defined as:
\[
\hat{f}(X) \supseteq f(X)
\]
This definition ensures that \( \hat{f}(X) \) contains all the values of \( f(X) \), along with potentially additional values that provide a conservative estimate. Verifying Deep Neural Networks (DNNs) is an NP-Complete problem~\cite{katz2017reluplex}, so computing an overapproximation of the output set simplifies the problem, making it more feasible to compute for larger and more complex problems. Our method can compute an (output) reachable set $Y$ of a DNN $f$,  i.e., $Y= f(X)$, 
where $Y$ can be computed exactly (sound and complete) or overapproximated (sound and incomplete).

\begin{align*} \label{eq:soundCom}
    complete: & ~~ \forall y \in Y, ~\exists x \in X~\mid~y = f(x) \\
    sound: & ~~ \forall x \in X, ~\exists y \in Y~\mid~y = f(x)
\end{align*}

\subsection{Methods for Computing Reach Sets}

Several methods exist for the overapproximation of reach sets, each with its advantages, computational implications, and limitations. We are using NNV, making use of star sets, an abstraction-based reachability tool with the following available methods:

\subsubsection*{Exact Star (Sound and Complete)} 
The Exact Star method computes the exact reach set \( f(X) \). However, it can be computationally expensive and is often infeasible for high-dimensional systems or very large NNs~\cite{tran2019fm}. This method is able to compute the exact reach set for linear and piece-wise linear layers (e.g., ReLU, Convolution, BatchNorm, etc.).

\subsubsection*{Approximate (Approx) Star (Sound)}
The Approximate Star method uses star sets to provide a close but approximate representation of the reach set. It simplifies the computation by approximating some of the constraints that define the reach set on piece-wise linear layers. This method strikes a balance between computational efficiency and accuracy~\cite{tran2021prefilter}.

\subsubsection*{Relaxed Star (Sound)}
The Relaxed Star method involves increasing the approximation by introducing a relaxation parameter \( \beta \), which defines how much the approximation can deviate from the actual reach set:
\[
\text{Relaxed Star}(\beta) \supseteq \text{Approximate Star} \supseteq \text{Exact Star}
\]
This is applicable to piece-wise linear layers, where the $\beta$ parameter determined the number of constraints to solve using linear programming (LP), and how many to overapproximate. Higher values of \( \beta \) increase the size of the overapproximation (larger number of overapproximated LP problems), thereby reducing the computational load at the cost of accuracy. A relax-star method with $\beta$ = 0 is equivalent to approx-star.

\section{Proofs for Section Soundness of Explanations 
\ref{sec:Soundness}} \label{appendix:proofs}


Additional to the proofs, we illustrate the inclusiveness of Remark \ref{theorem:subset_feature_inclusive_property} with the output space in logits as shown in Figure \ref{fig:inclusiveness} and a toy example where NN has 3 inputs and 2 outputs, and we perturbed all 3 inputs (red), 2 (yellow) or 1 (green) as shown in Figure \ref{fig:inclusives}
\begin{figure}[ht]
    \centering

    \begin{subfigure}[t]{0.23\textwidth}
        \centering
        \includegraphics[width=\textwidth]{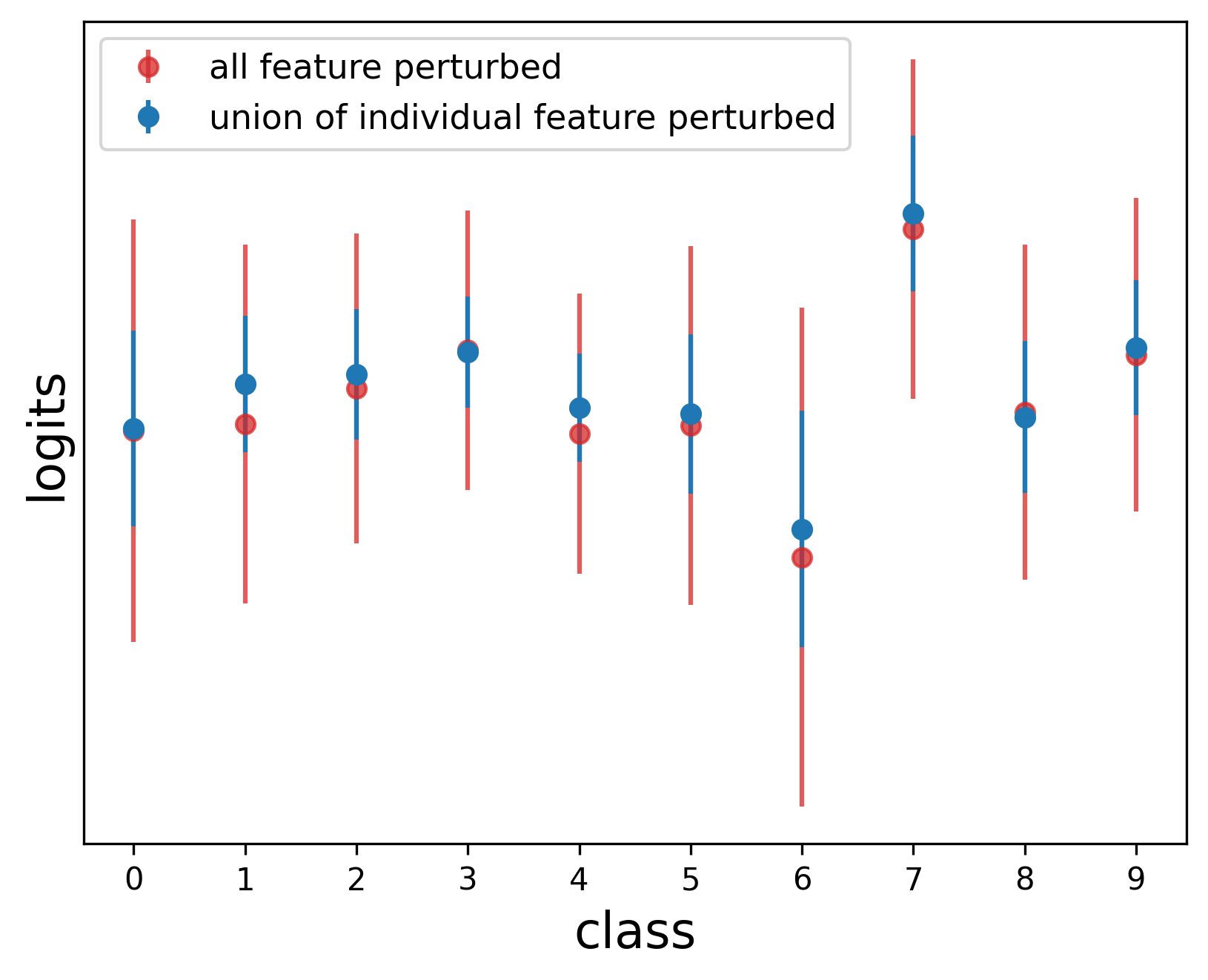}
        \caption{The range of union of individual feature perturbed (blue) versus all feature perturbed (red), based on Lemma \ref{lemma:feature_inclusive_property}. \\}
        \label{fig:inclusiveness:individual_feature}
    \end{subfigure}
    \begin{subfigure}[t]{0.23\textwidth}
        \centering
        \includegraphics[width=\textwidth]{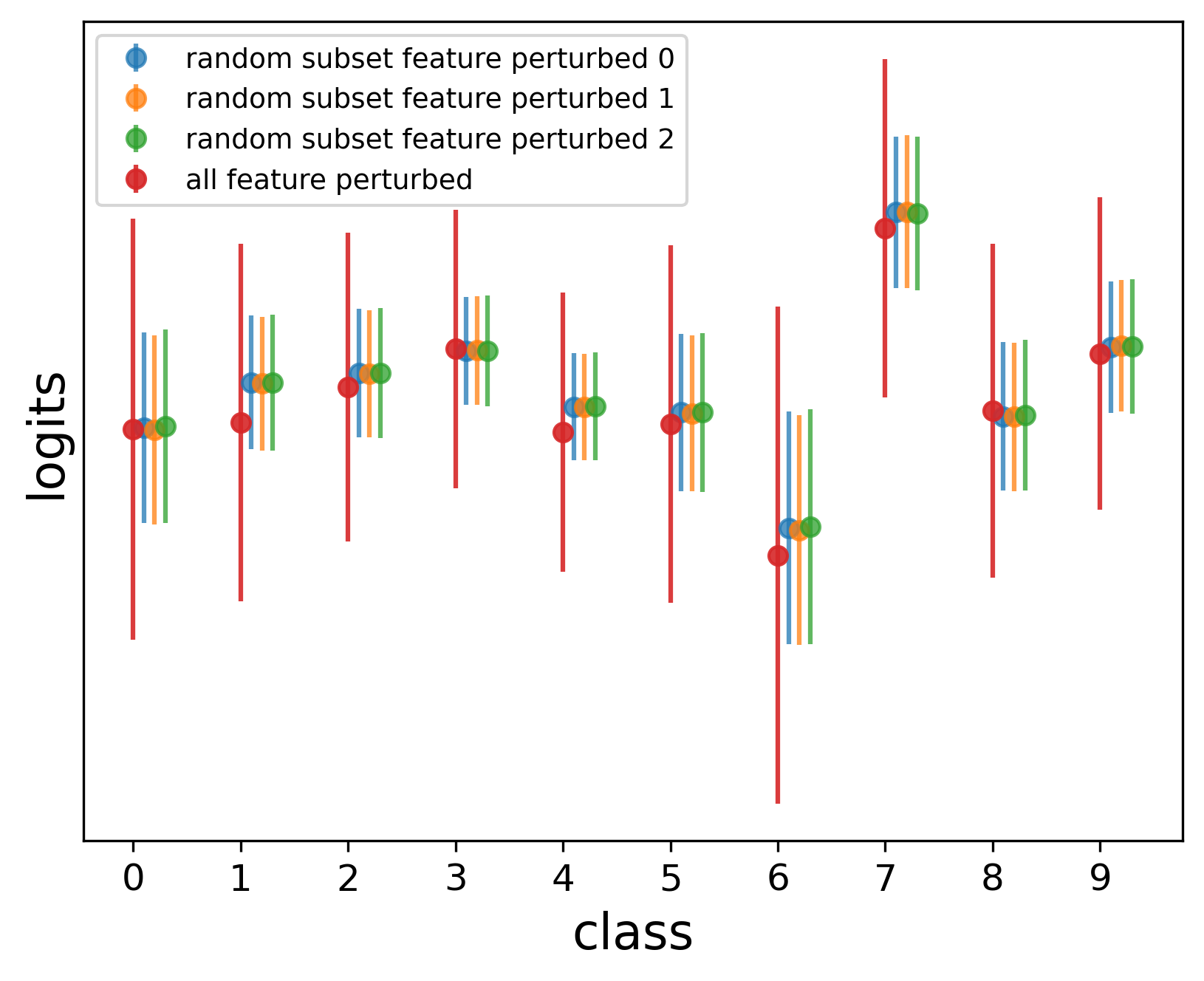} 
        \caption{The range of subset of features (that are subset of the whole set) perturbed (blue, orange, and green) versus all feature perturbed (red), based on Remark \ref{theorem:subset_feature_inclusive_property}.}
        \label{fig:inclusiveness:subset}
    \end{subfigure}
    \caption{Empirical evaluation on input space projected into one dimensional space (in output space). As shown in both of the subfigures, it shows the empirical evaluation of Lemma \ref{lemma:feature_inclusive_property} and Remark \ref{theorem:subset_feature_inclusive_property}, which satisfies what we claim. We evaluated on both approx star and exact star reachability method.}
    \label{fig:inclusiveness}
\end{figure}

\begin{figure}[htbp]
\centering
\begin{subfigure}[t]{0.15\textwidth}
    \includegraphics[width=\textwidth, height=55.5pt]{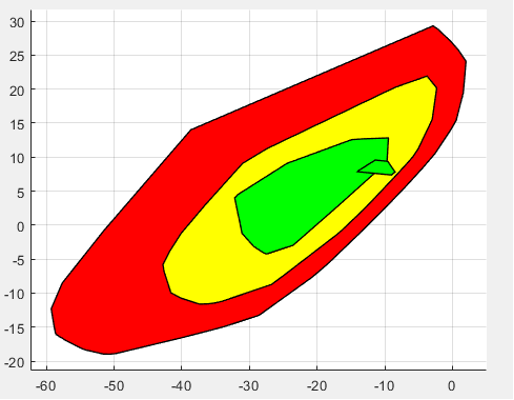}
    \caption{Feature of Red: [1, 1, 1]
Yellow: [1, 1, 0]
Green: [1, 0, 0], [0, 1, 0]}
    \label{fig:inclusive1}
\end{subfigure}
\begin{subfigure}[t]{0.15\textwidth}
    \includegraphics[width=\textwidth]{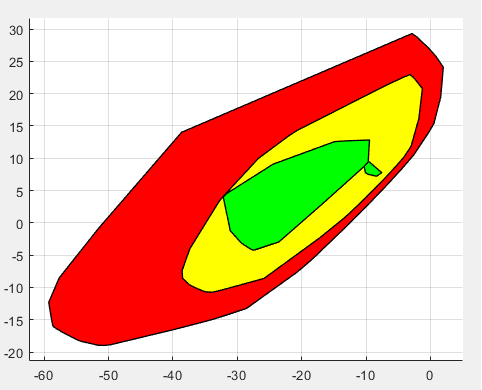}
    \caption{Feature of Red: [1, 1, 1]
Yellow: [1, 0, 1]
Green: [1, 0, 0], [0, 0, 1]
}
            \label{fig:inclusive2}
\end{subfigure}
\begin{subfigure}[t]{0.15\textwidth}
    \includegraphics[width=\textwidth]{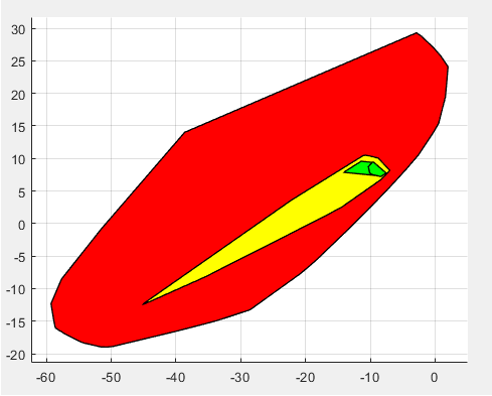}
    \caption{Feature of Red: [1, 1, 1]
Yellow: [0, 1, 1]
Green: [0, 1, 0], [0, 0, 1]
}
    \label{fig:inclusive3}
\end{subfigure}
\caption{Soundness and inclusiveness visualization on a toy example. We evaluate the feature space projected into 2 dimension output space. The individual subcaptions show the vector of features perturbed.}
\label{fig:inclusives}
\end{figure}

\subsection{Proof for Lemma \ref{lemma:feature_inclusive_property}} \label{appendix:lemma:feature_inclusive_property}

Given a model $f: \mathbb{R}^n \rightarrow \mathbb{R}^m$, for a point $x \in \mathbb{R}^n$, define the set $X_{\epsilon, i}$ for $i \in \{1,2, \ldots, n\}$ as:
$$
    X_{\epsilon, i} = \{x' \in \mathbb{R}^n : \|x' - x\|^i \leq \epsilon\} \text{ (perturbation on } i^{th} \text{ feature only)}
$$
Let $X_\epsilon$ be defined as:
$$ X_\epsilon = \{x' \in \mathbb{R}^n : \|x' - x\|_{\infty} \leq \epsilon\} $$

It holds that $Y_\epsilon = f(X_\epsilon) \supseteq \bigcup_{i=1}^n f(X_{\epsilon, i})$.

\begin{proof}
To prove that $f(X_\epsilon) \supseteq \bigcup_{i=1}^n f(X_{\epsilon, i})$, we must show that any element belonging to the set $\bigcup_{i=1}^n f(X_{\epsilon, i})$ also belongs to the set $f(X_\epsilon)$.

Let $y_{out}$ be an arbitrary element in $\bigcup_{i=1}^n f(X_{\epsilon, i})$.
By the definition of the union of sets, if $y_{out} \in \bigcup_{i=1}^n f(X_{\epsilon, i})$, then there exists at least one index $k \in \{1, 2, \ldots, n\}$ such that $y_{out} \in f(X_{\epsilon, k})$.

Since $y_{out} \in f(X_{\epsilon, k})$, there exists an input vector $x^* \in X_{\epsilon, k}$ such that $y_{out} = f(x^*)$.

According to our definition of $X_{\epsilon, k}$:
\begin{enumerate}
    \item For all indices $j \neq k$, the $j$-th component of $x^*$ is $x^*_j = x_j$.
    \item For the $k$-th index, the $k$-th component of $x^*$ is $x^*_k$, and it satisfies $|x^*_k - x_k| \leq \epsilon$.
\end{enumerate}

Now, we need to show that this $x^*$ is also an element of $X_\epsilon$.
A vector $x'$ is in $X_\epsilon$ if $\|x' - x\|_{\infty} \leq \epsilon$. This means that for all component indices $j \in \{1, \ldots, n\}$, the condition $|x'_j - x_j| \leq \epsilon$ must be satisfied.

Let's examine the absolute differences of the components of $x^*$ and $x$:
\begin{itemize}
    \item For any index $j \neq k$:
    $|x^*_j - x_j| = |x_j - x_j| = 0$.
    \item For the index $j = k$:
    $|x^*_k - x_k| \leq \epsilon$.
\end{itemize}

Since $\epsilon \geq 0$ (as it represents a magnitude of perturbation), the condition $0 \leq \epsilon$ is true.
Thus, for every index $j \in \{1, \ldots, n\}$, we have $|x^*_j - x_j| \leq \epsilon$.
This implies that the $L_\infty$ norm of the difference $x^* - x$ satisfies:
$$ \|x^* - x\|_{\infty} = \max_{j \in \{1, \ldots, n\}} |x^*_j - x_j| \leq \epsilon $$

Therefore, $x^*$ is an element of $X_\epsilon$.
Since $x^* \in X_\epsilon$ and $y_{out} = f(x^*)$, it follows by the definition of $f(X_\epsilon)$ that $y_{out} \in f(X_\epsilon)$.

As $y_{out}$ was an arbitrary element of $\bigcup_{i=1}^n f(X_{\epsilon, i})$, we have shown that every element of $\bigcup_{i=1}^n f(X_{\epsilon, i})$ is also an element of $f(X_\epsilon)$.
Therefore, $\bigcup_{i=1}^n f(X_{\epsilon, i}) \subseteq f(X_\epsilon)$, which is equivalent to the statement $f(X_\epsilon) \supseteq \bigcup_{i=1}^n f(X_{\epsilon, i})$.

This completes the direct proof. \hfill 

\end{proof}

\subsection{Proof for Remark \ref{theorem:subset_feature_inclusive_property}} \label{appendix:remark:subset_feature_inclusive_property}

\begin{proof}
To prove that $Y_{\epsilon, d} \subseteq Y_{\epsilon, l}$, it is sufficient to show that $X_{\epsilon, d} \subseteq X_{\epsilon, l}$. If this holds, then for any $y \in Y_{\epsilon, d}$, there exists an $x^* \in X_{\epsilon, d}$ such that $y = f(x^*)$. Since $x^* \in X_{\epsilon, d}$ and $X_{\epsilon, d} \subseteq X_{\epsilon, l}$, $x^*$ must also be in $X_{\epsilon, l}$. Therefore, $y = f(x^*) \in f(X_{\epsilon, l}) = Y_{\epsilon, l}$, which proves $Y_{\epsilon, d} \subseteq Y_{\epsilon, l}$.

We will prove $X_{\epsilon, d} \subseteq X_{\epsilon, l}$ by induction on $m = |l \setminus d|$, which is the number of elements in $l$ that are not in $d$. Since $d \subset l$ and it's implied by the setup ($|d|<|l|$) that $l$ is strictly larger than $d$, we have $m \geq 1$.

\noindent\textbf{Base Case:} $m=1$.
This means $l$ contains exactly one element more than $d$, and $d \subset l$. So, $l$ can be written as $l = d \cup \{j_0\}$ for some feature index $j_0 \notin d$.
Let $x^* \in X_{\epsilon, d}$. By definition of $X_{\epsilon, d}$:
\begin{enumerate}
    \item $(x^*)_k = x_k$ for all $k \notin d$.
    \item $|(x^*)_k - x_k| \leq \epsilon$ for all $k \in d$.
\end{enumerate}
We want to show that $x^* \in X_{\epsilon, l} = X_{\epsilon, d \cup \{j_0\}}$. For this, $x^*$ must satisfy:
\begin{enumerate}
    \item[(a)] $(x^*)_k = x_k$ for all $k \notin (d \cup \{j_0\})$.
    \item[(b)] $|(x^*)_k - x_k| \leq \epsilon$ for all $k \in (d \cup \{j_0\})$.
\end{enumerate}
Verify these conditions: \\

For condition (a): If $k \notin (d \cup \{j_0\})$, then $k \notin d$ and $k \neq j_0$. Since $k \notin d$, by property (1) of $x^*$, we have $(x^*)_k = x_k$. So, condition (a) is satisfied.

For condition (b): Consider $k \in (d \cup \{j_0\})$.
    If $k \in d$: By property (2) of $x^*$, $|(x^*)_k - x_k| \leq \epsilon$. This part of condition (b) is satisfied.
    If $k = j_0$: Since $j_0 \notin d$, by property (1) of $x^*$, we have $(x^*)_{j_0} = x_{j_0}$. Therefore, $|(x^*)_{j_0} - x_{j_0}| = |x_{j_0} - x_{j_0}| = 0$. Since $\epsilon \geq 0$, $0 \leq \epsilon$. This part of condition (b) is also satisfied.
Thus, $x^* \in X_{\epsilon, l}$. Since $x^*$ was an arbitrary element of $X_{\epsilon, d}$, we have $X_{\epsilon, d} \subseteq X_{\epsilon, l}$ for $m=1$.
(The specific example where $|d|=1$ (e.g., $d=\{i\}$) and $|l|=2$ (e.g., $l=\{i,j\}$ with $i \neq j$) is an instance of this base case, as $m = |l \setminus d| = 1$.)

\noindent\textbf{Inductive Hypothesis (IH):}
Assume that for some integer $k \geq 1$, the statement holds for $m=k$. That is, for any two sets of feature indices $D'$ and $L'$ such that $D' \subset L'$ and $|L' \setminus D'| = k$, we have $X_{\epsilon, D'} \subseteq X_{\epsilon, L'}$.

\noindent\textbf{Inductive Step:}
We need to prove the statement for $m=k+1$. Let $d$ and $l$ be sets of feature indices such that $d \subset l$ and $|l \setminus d| = k+1$.
Since $|l \setminus d| = k+1 \geq 1$ (as $k \geq 1$, $k+1 \geq 2$; if $k=0$ was allowed in IH, $k+1 \ge 1$), we can pick an element $j^* \in l \setminus d$.
Define an intermediate set $l_{int} = l \setminus \{j^*\}$.

Now consider the relationship between $d$ and $l_{int}$:
Since $d \subset l$ and $j^* \in (l \setminus d)$, $d$ does not contain $j^*$. Thus, $d \subseteq l \setminus \{j^*\}$, so $d \subseteq l_{int}$.
The elements in $l_{int}$ that are not in $d$ are $(l \setminus \{j^*\}) \setminus d = (l \setminus d) \setminus \{j^*\}$.
So, $|l_{int} \setminus d| = |(l \setminus d) \setminus \{j^*\}| = (k+1) - 1 = k$.

We have $d \subseteq l_{int}$ and $|l_{int} \setminus d|=k$.
\begin{itemize}
    \item If $k=0$: This would mean $l_{int} = d$. Then $|l \setminus d| = 1$. This situation is covered by the Base Case directly, where $X_{\epsilon, d} \subseteq X_{\epsilon, l}$ was proven. (Note: if IH assumes $k \geq 1$, then this $k=0$ sub-case for $l_{int} \setminus d$ won't occur here as we're proving for $k+1 \ge 2$).
    \item If $k \geq 1$: Then $d \subset l_{int}$ (if $d \neq l_{int}$) or $d=l_{int}$ (if $k=0$, but here we assume $k \ge 1$ for IH). If $d \subset l_{int}$ and $|l_{int} \setminus d| = k$, by the Inductive Hypothesis, we have $X_{\epsilon, d} \subseteq X_{\epsilon, l_{int}}$. If $d=l_{int}$ (which means $k=0$), then $X_{\epsilon,d} = X_{\epsilon,l_{int}}$ is trivially true. For our IH defined with $k \ge 1$, we must have $d \subset l_{int}$ (unless $d$ was empty and $k=0$ for $|l_{int}\setminus d|$), and $X_{\epsilon, d} \subseteq X_{\epsilon, l_{int}}$ by IH.
\end{itemize}
To simplify, since the IH is stated for $k \ge 1$: For $|l_{int} \setminus d| = k \ge 1$, we have $D'=d$ and $L'=l_{int}$, so $X_{\epsilon, d} \subseteq X_{\epsilon, l_{int}}$ by IH. If $k=0$ (meaning $|l \setminus d|=1$), this was the base case. The structure of induction ensures $k$ from IH is $\ge 1$.

So, $X_{\epsilon, d} \subseteq X_{\epsilon, l_{int}}$ (either because $d=l_{int}$ if $k=0$ which corresponds to $|l \setminus d|=1$ covered by base case, or by IH if $k \ge 1$).

Now consider $l_{int}$ and $l$. We have $l_{int} \subset l$ (since $j^* \in l$ and $j^* \notin l_{int}$) and $l = l_{int} \cup \{j^*\}$, with $j^* \notin l_{int}$.
Thus, $|l \setminus l_{int}| = 1$. By our Base Case ($m=1$), we have $X_{\epsilon, l_{int}} \subseteq X_{\epsilon, l}$.

Combining these results: $X_{\epsilon, d} \subseteq X_{\epsilon, l_{int}}$ and $X_{\epsilon, l_{int}} \subseteq X_{\epsilon, l}$.
By the transitivity of set inclusion, $X_{\epsilon, d} \subseteq X_{\epsilon, l}$.
This completes the inductive step for $m=k+1$.

Therefore, by the principle of mathematical induction, for all sets of feature indices $d, l$ such that $d \subset l$ (which implies $m = |l \setminus d| \geq 1$), it holds that $X_{\epsilon, d} \subseteq X_{\epsilon, l}$.

As established at the beginning of the proof, if $X_{\epsilon, d} \subseteq X_{\epsilon, l}$, then $Y_{\epsilon, d} = f(X_{\epsilon, d}) \subseteq f(X_{\epsilon, l}) = Y_{\epsilon, l}$.
This concludes the proof. 
\end{proof}

\subsection{Proof for Theorem \ref{theorem:soundness} (Soundness of \vitax{} Explanation)} \label{appendix:theorem:soundness}


\begin{proof}
Let $A_{found}$ be the feature subset (specifically, the prefix $\pi[0:u]$ of the heuristically ranked features $\pi$) returned by Algorithm \ref{alg:method} (ViTaX). We need to show two things:
\begin{enumerate}
    \item $A_{found}$ satisfies Targeted $\epsilon$-Robustness as defined in Definition \ref{def:targeted_robustness_def} (which involves $l_{\mathbf{y}, A_{found}} > u_{\mathbf{t}, A_{found}}$ and the conditions for other classes $\mathbf{k}$).
    \item $A_{found}$ is the largest prefix of $\pi$ for which this property holds.
\end{enumerate}

The formal property $\Phi$ that Algorithm 1 verifies for a candidate prefix $d = \pi[0:u]$ using the solver $\mathcal{V}$ is:
\[
\Phi(d) = \left( l_{\mathbf{y}, d} > u_{\mathbf{t}, d} \right) \wedge \left( \forall \mathbf{k} \in \{1, \dots, m\}, \mathbf{k} \neq \mathbf{y} \land \mathbf{k} \neq \mathbf{t} \implies u_{\mathbf{t}, d} > u_{\mathbf{k}, d} \right)
\]
This $\Phi(d)$ directly corresponds to the conditions required for $d$ to be a Targeted $\epsilon$-Robust Explanation as per Definition 2.2, along with ensuring the target class $\mathbf{t}$ is more prominent than other non-original classes $\mathbf{k}$ within the context of the perturbed subset $d$.

The reachability solver $\mathcal{V}$ is assumed to be:
\begin{itemize}
    \item \textbf{Sound (for $\mathcal{V}$):} If $\mathcal{V}$ returns `FLAG = True' for $\Phi(d)$, then $\Phi(d)$ indeed holds for all $x' \in X_{\epsilon, d}$.
    \item \textbf{Complete (for $\mathcal{V}$):} If $\Phi(d)$ holds for all $x' \in X_{\epsilon, d}$, then $\mathcal{V}$ returns `FLAG = True' for $\Phi(d)$.
\end{itemize}

1. Proving $A_{found}$ satisfies Targeted $\epsilon$-Robustness:

Algorithm \ref{alg:method} employs a binary search strategy to find a candidate set. The algorithm stores a candidate $A_{candidate} \leftarrow d$ when `FLAG = True' (i.e., $\Phi(d)$ is satisfied) and then attempts to find a larger robust subset by updating $I \leftarrow u+1$. The final set $A_{found}$ returned by the algorithm is the largest prefix $d = \pi[0:u]$ for which $\mathcal{V}$ returned `FLAG = True'.

Since $\mathcal{V}$ returned `FLAG = True` for $A_{found}$ (otherwise it would not have been selected as the final $A_{candidate}$ that becomes $A_{found}$), and $\mathcal{V}$ is sound, it directly implies that the property $\Phi(A_{found})$ holds. As $\Phi(A_{found})$ encompasses the conditions of Definition \ref{def:targeted_robustness_def}, $A_{found}$ is an explanation that satisfies Targeted $\epsilon$-Robustness.

2. Proving $A_{found}$ is the largest such prefix of $\pi$:

The binary search systematically explores prefixes of $\pi$. Let $A_{found} = \pi[0:u_{final}]$.
\begin{itemize}
    \item By construction, $\Phi(A_{found})$ holds.
    \item Consider any prefix $\pi[0:u']$ where $u' > u_{final}$. If such a prefix also satisfied $\Phi$, the binary search logic (specifically, the update $I \leftarrow u+1$ when a robust subset is found, aiming to extend it) would have led to $A_{candidate}$ being updated to this larger prefix, or the search continuing in a range that includes $u'$. The algorithm terminates when $I > J$, and $A_{found}$ is the $A_{candidate}$ corresponding to the largest $u$ for which $\Phi$ was found to hold.
    \item If we consider the specific feature $\pi[u_{final}+1]$ (the feature immediately following the last feature in $A_{found}$ according to the heuristic ranking), including it to form $A' = \pi[0:u_{final}+1]$ must have resulted in $\mathcal{V}$ returning `FLAG = False' for $\Phi(A')$. If $\mathcal{V}$ is complete, its returning `FLAG = False' means $\Phi(A')$ indeed does not hold.
\end{itemize}
Therefore, the binary search procedure, relying on the soundness and completeness of the solver $\mathcal{V}$ to correctly evaluate $\Phi$ for each tested prefix, ensures that $A_{found}$ is the largest prefix of $\pi$ for which Targeted $\epsilon$-Robustness (as embodied by $\Phi$) holds.

Thus, the explanation $A_{found}$ generated by ViTaX satisfies Targeted $\epsilon$-Robustness and is the maximal such explanation according to the heuristic feature ordering $\pi$. This completes the proof.
\end{proof}

\section{Baselines and Time Complexity }
\noindent\textbf{Brute Force Baseline.}
In principle, if we were to exhaustively test \emph{every} possible combination of $N$ features (i.e., every subset), the time complexity would be $\mathcal{O}\bigl(2^N \cdot \mathcal{V}(N)\bigr)$. However, due to computational constraints, we approximate this brute-force approach by sampling a random ordering of the $N$ features multiple times (e.g., $100$ runs). For each run, we binarily select features \emph{in the random order} until the verifier considers its robustness (one additional feature will be robust). We record fidelity for each run and select the best result over these $100$ trials. While still time-consuming, this procedure is vastly simpler than enumerating $2^N$ subsets and serves as a straightforward benchmark for comparison with more advanced explanation methods.

\medskip
\noindent\textbf{Complexities.}
\begin{itemize}
    \item \textbf{\vitax{}}: $\displaystyle \mathcal{O}\bigl(\log_2(N) \cdot \mathcal{V}(N)\bigr)$.
    \item \textbf{VeriX}: $\displaystyle \mathcal{O}\bigl(N \cdot \mathcal{V}(N)\bigr)$.
    \item \textbf{Brute Force (theoretical)}: $\displaystyle \mathcal{O}\bigl(2^N \cdot \mathcal{V}(N)\bigr)$.
\end{itemize}

\section{Baseline Adaptation Methodology}\label{appendix:baseline_adaptation}

To facilitate a fair comparison with the semifactual explanations generated by \vitax{}, the baseline methods LIME~\cite{ribeiro_why_2016} and Anchors~\cite{ribeiro_anchors_2018} were adapted from their standard implementations. The core objective was to shift these methods from generating factual explanations (explaining ``why this prediction occurred'') to generating probabilistic semifactual explanations (explaining ``which minimal features are sufficient to maintain this prediction'').

\subsection{Conceptual Framework}
The adaptation aims to identify a minimal subset of features ($A_{prob}$) such that, when these features are fixed, the model's prediction remains the original class ($y$) with high probability ($P \geq \tau$, where $\tau$ is a precision threshold), even when the remaining features ($\neg A_{prob}$) are randomized or perturbed.

\subsection{Adapting Anchors}
The original Anchors algorithm is inherently aligned with semifactual reasoning, as it seeks a minimal set of conditions (an ``anchor'') sufficient to support the current prediction. We characterize this as constructive sufficiency, where the explanation is built up greedily to find the minimal conditions required to maintain the prediction.

\paragraph{Methodology}
We formalized the Anchors implementation to explicitly optimize for semifactual precision:
\begin{enumerate}
    \item \textbf{Greedy Optimization:} The algorithm iteratively adds features to the candidate anchor set.
    \item \textbf{Precision Thresholding:} At each step, the precision is evaluated by sampling perturbations in the non-anchor features (see Algorithmic Details below). The search continues until the predefined high precision threshold $\tau$ (e.g., 0.95) is met.
    \item \textbf{Output:} The minimal feature set that satisfies the precision threshold.
\end{enumerate}
This formalization ensures that the output of Anchors is explicitly a statement of probabilistic sufficiency: ``If these features are present, the prediction persists with probability P.''

\subsection{Adapting LIME}
The original LIME algorithm learns a local linear model to approximate the decision boundary, identifying features that most contribute to the prediction. This is fundamentally an attribution method. To generate semifactual insights, we extended LIME with a secondary analysis phase

\paragraph{Methodology:}
We implemented a two-stage process:
\begin{enumerate}
    \item \textbf{Stage 1 (Feature Importance Learning):} Standard LIME is executed to learn the local linear model and determine the importance ranking of the features.
    \item \textbf{Stage 2 (Minimal Sufficiency Search):} A secondary greedy search is performed using the learned importance scores. This search asks: "Which combination of the most important features is minimally sufficient to preserve the prediction?"
    \item \textbf{Stability Validation:} Similar to the Anchors adaptation, the sufficiency of the subset is validated by sampling perturbations in the remaining features until the precision threshold $\tau$ is met.
\end{enumerate}
This adaptation shifts LIME from merely identifying contributors to identifying the minimal necessary components for prediction stability, guided by the local approximation.

\subsection{Algorithmic Details}
The core of both adaptations is the sufficiency search (utilized in Anchors optimization and LIME Stage 2). This involves a greedy selection process and a probabilistic estimation of precision.

\paragraph{Precision Estimation (Stability Validation):}
To estimate the precision of a candidate feature set $A_{cand}$, we use Monte Carlo sampling:
\begin{enumerate}
    \item Generate $N$ samples (e.g., $N=500$).
    \item For each sample, the features in $A_{cand}$ are fixed to their original values.
    \item The remaining features ($\neg A_{cand}$) are randomized (e.g., using Bernoulli sampling or replacing with baseline values).
    \item The model prediction is evaluated on the perturbed sample.
    \item Precision is calculated as the fraction of the $N$ samples where the prediction remains the original class $y$.
\end{enumerate}

\subsection{Nature of the Guarantee}
It is crucial to note that the guarantees provided by these adapted methods are probabilistic and empirical. They rely on sampling to estimate the stability of the prediction. Unlike \vitax{}, which uses formal reachability analysis to verify the entire continuous perturbation space, these methods cannot guarantee behavior between the samples and may miss adversarial counterexamples within the perturbation space.

\section{Evaluation Cont.}\label{sec:otherdatasets}

To validate the generalizability and robustness of our methodologies, we extended our evaluations to include diverse datasets: GTSRB, the EMNIST Letters, and the TaxiNet dataset. These datasets were selected to represent a broad range of challenges in image recognition, from traffic sign identification to character recognition and autonomous aircraft taxiing scenarios. For each dataset, we applied the same experimental protocols to ensure consistency in comparison. The GTSRB dataset helped us assess the effectiveness of our approach in recognizing and interpreting traffic signs under various lighting and occlusion conditions. In EMNIST Letters, our focus was on evaluating the capability to discern and classify distinct alphabetical characters, which exhibit subtle variations. Meanwhile, TaxiNet provided a unique challenge in navigating complex taxiway environments, testing the adaptability of our algorithms to spatial and contextual variability.

\vitax{} evaluates the GTSRB classification dataset with CNN model (accuracy 85.35\% and structure in Table \ref{tab:structure:gtsrb_cnn}) on all 43 classes. The input image size is $28 \times 28 \times 3$. We use a perturbation magnitude of \( \epsilon = \frac{25}{255} \) to generate counterfactual examples across all classes from single sample images. This approach helps explore the targeted explanatory power of perturbations in identifying and delineating class-specific responses.

\begin{figure*}[htbp]
    \centering
    \begin{subfigure}[t]{0.15\textwidth}
        \centering
        \includegraphics[width=\textwidth]{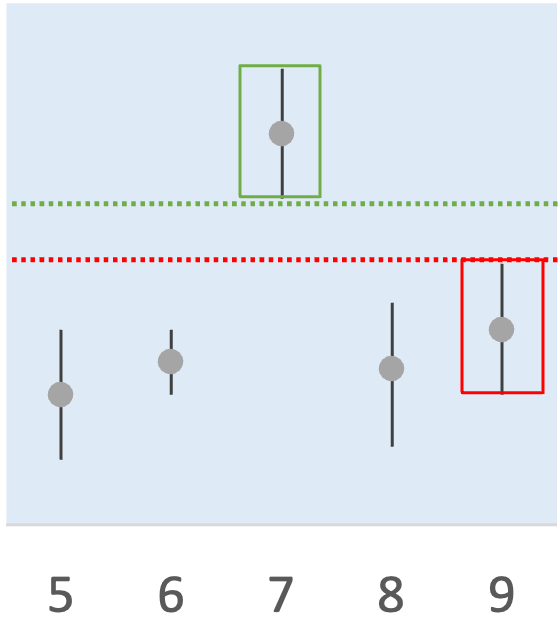}
        \caption{robust (not at the boundary)}
    \end{subfigure}
    \hspace
    {10pt}
    \begin{subfigure}[t]{0.15\textwidth}
        \centering
        \includegraphics[width=\textwidth]{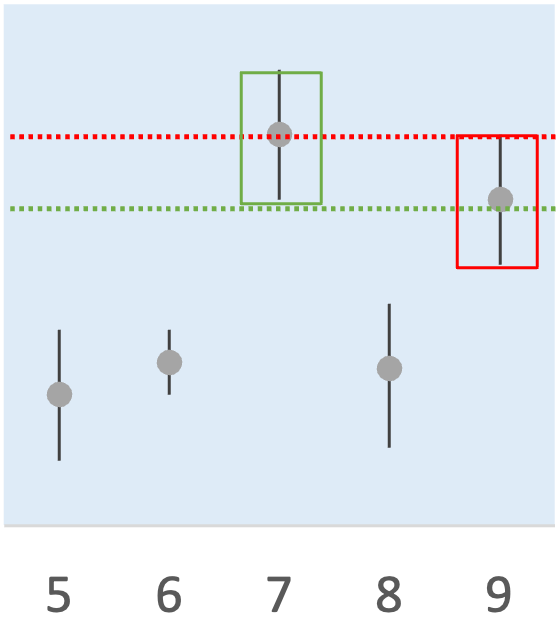}
        \caption{unrobust to $\mathbf{t}$  \newline}
    \end{subfigure}
    \hspace{10pt}
    \begin{subfigure}[t]{0.15\textwidth}
        \centering
        \includegraphics[width=\textwidth]{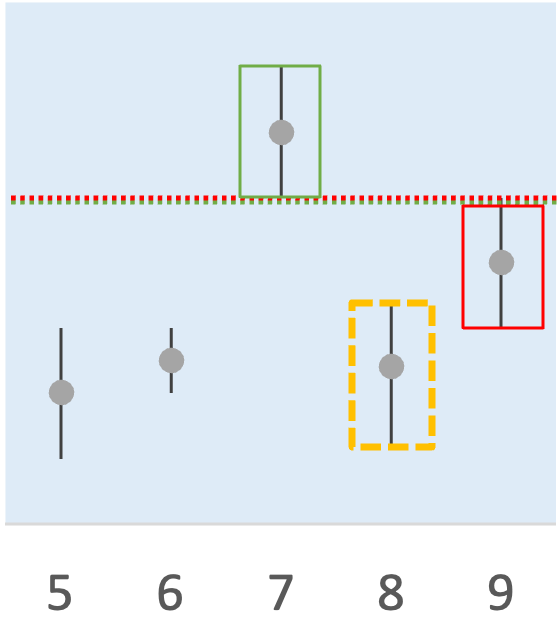}
        \caption{robust to $\mathbf{t}$ and to $\mathbf{k}$ \\}
        \label{fig:robust_tok_tot}
    \end{subfigure}
    \hspace{10pt}
    \begin{subfigure}[t]{0.15\textwidth}
        \centering
        \includegraphics[width=\textwidth]{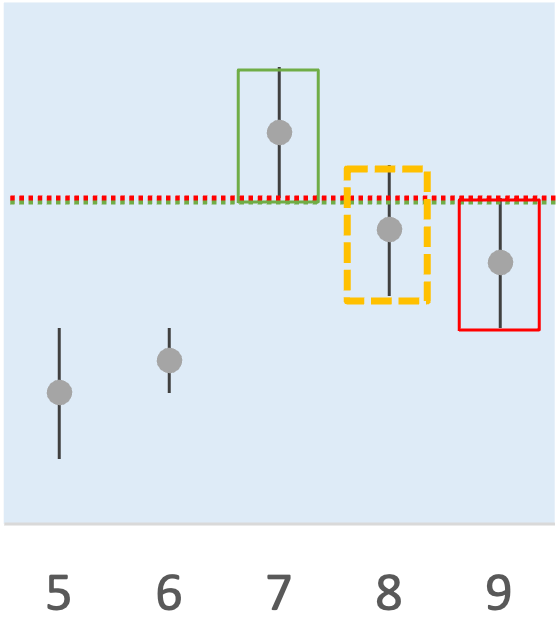}
        \caption{robust to $\mathbf{t}$, not to $\mathbf{k}$}
        \label{fig:robust_tot_not_tok}
    \end{subfigure}
    \caption{An example of potentially  not robust to other classes $k$ in projected range of logits. Green box is label $\mathbf{y}$, red box is target $\mathbf{t}$, and yellow box is other classes $\mathbf{k}$. In this case, we are arguing, given our Equation \ref{eq:robustness} and \ref{eq:targeted_robustness}, there are potentially cases in (c) and (d). The yellow box $\mathbf{k}$ could cross the boundary while the subset of feature is robust to $\mathbf{t}$.}
\end{figure*}



\subsection{Not Robust to Other Classes $\mathbf{k}$}\label{appendix:not_robust_to_others}

\begin{figure}[htbp]
    \centering
    \begin{subfigure}[t]{0.2\textwidth}
        \centering
        \includegraphics[width=\textwidth]{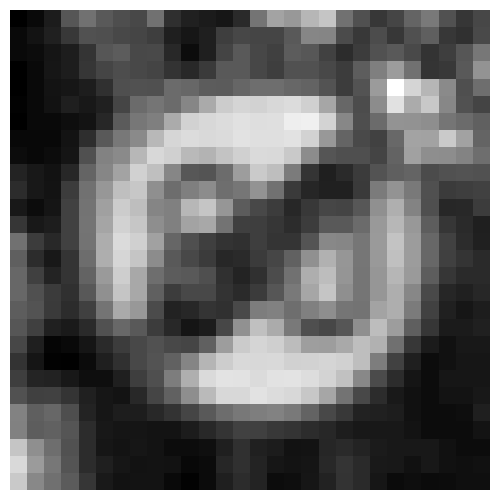}
        \caption{Speed Limit 80 End: the target we try to perturb to}
    \end{subfigure}
    \hspace{10pt}
    \begin{subfigure}[t]{0.2\textwidth}
        \centering
        \includegraphics[width=\textwidth]{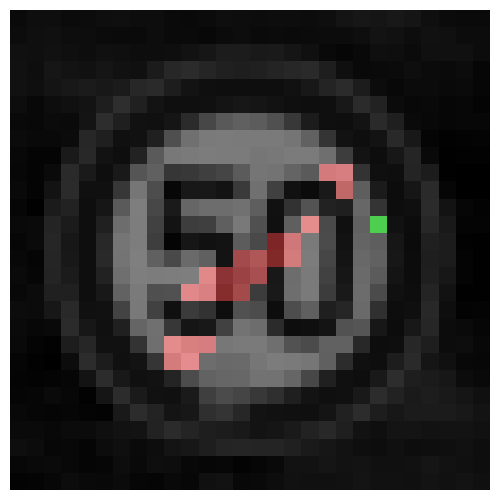}
        \caption{The perturbed of 50 to Speed Limit 80 End}
    \end{subfigure}
    \caption{GTSRB Explanation on 50 to Speed Limit 80 End}
    \label{fig:not_robust_to_others_but_useful}
\end{figure}

\paragraph{Robustness to Some Classes} We train on different combination of number of classes in order to find the optimal way of robustness to $\mathbf{t}$ and to $\mathbf{k}$, as shown in Figure \ref{fig:robust_tok_tot}. One thing that we notice is that though there are cases where we are not able to find $\mathbf{t}$-targeted explanation that are also robust to all $\mathbf{k}$. For instance, labeled-7 sample from MNIST goes to 9 as a targeted explanation. But upper bound of other classes, such as 8, overlaps with 9, as shown in \ref{fig:robust_tot_not_tok}. So this subset of features would be not robust for other classes but only for class 7. We argue though this might not be most robust or explainable features, we believe there is some usefulness in it. 

As shown in Figure \ref{fig:not_robust_to_others_but_useful}, the image we are perturbing is `50' and give us very interesting result. Though in the robustness solver, it also returns us that the targeted explanation it generated not robust to image of speed limit `60', but the explanation shows the red line to attempt to find the targeted explanation for speed limit 80 end. 
\begin{figure}[t]
    \centering
    \begin{subfigure}[b]{0.23\textwidth}
        \centering
        \includegraphics[width=\textwidth]{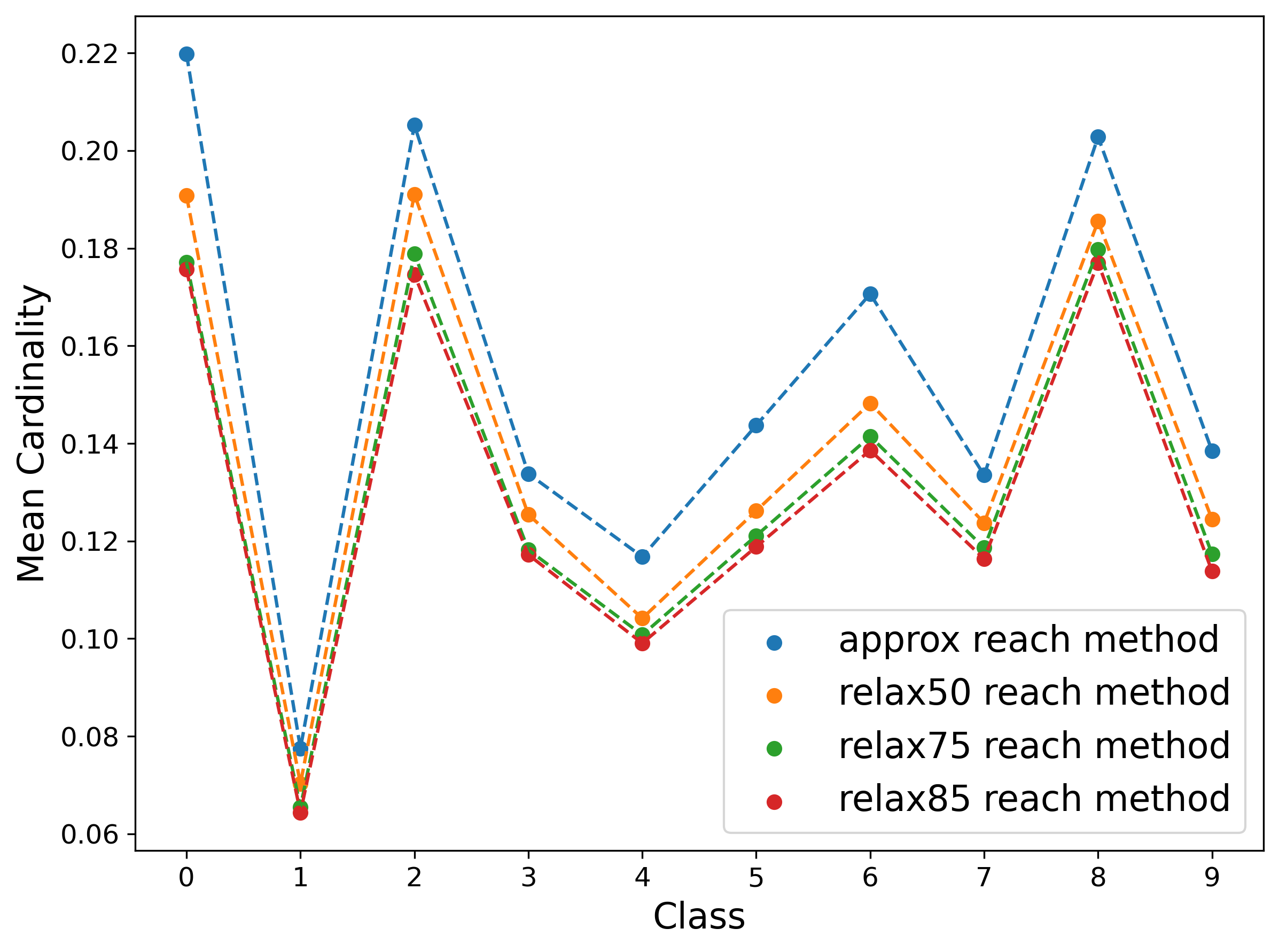}
        \caption{Cardinality for Each Class}
        \label{fig:approx_method:cardinality}
    \end{subfigure}
    \begin{subfigure}[b]{0.23\textwidth}
        \centering
        \includegraphics[width=\textwidth]{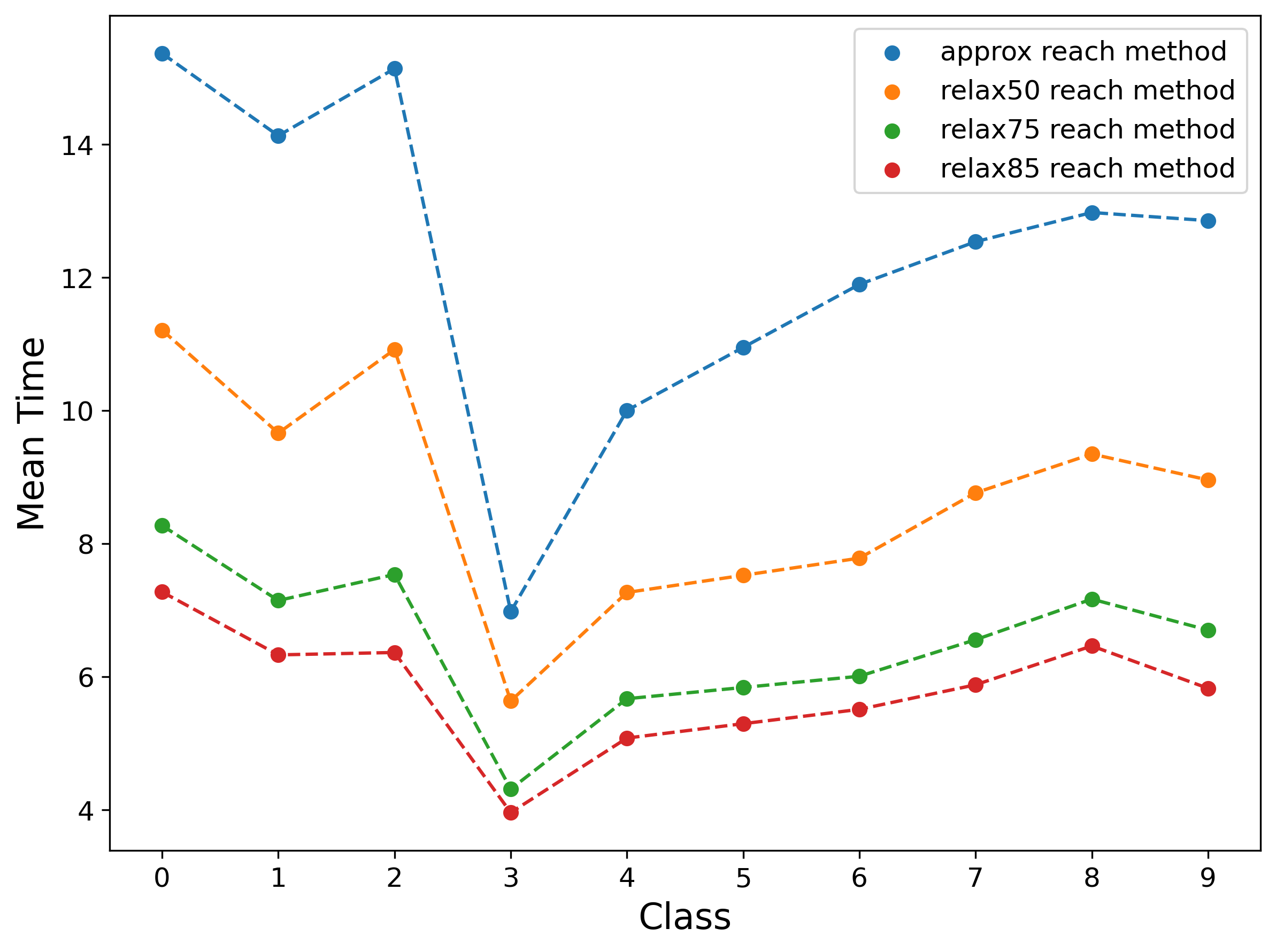}
        \caption{Time (sec) for Each Class}
        \label{fig:approx_method:time}
    \end{subfigure}

    \caption{X-axis represents the class from 0 to 9 and y-axis represents the mean cardinality value as percentage as shown in \ref{fig:approx_method:cardinality} and mean time in \ref{fig:approx_method:time}.}
    \label{fig:approx_method}
\end{figure}

\begin{figure}[t]
    \centering
    \begin{minipage}[t]{0.19\textwidth}
        \begin{subfigure}[t]{0.48\textwidth}
            \centering
            \includegraphics[width=\textwidth]{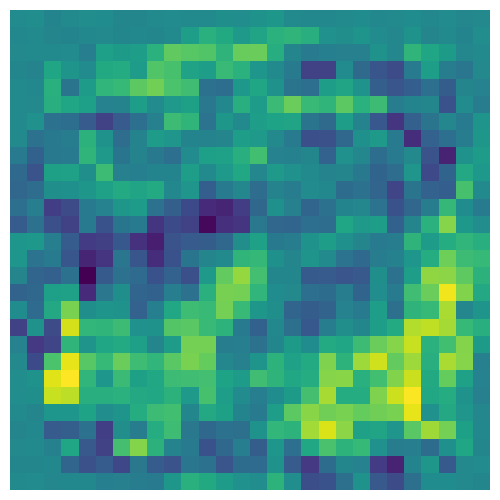}
            \label{fig:subfigure1a}
        \end{subfigure}
        \begin{subfigure}[t]{0.48\textwidth}
            \centering
            \includegraphics[width=\textwidth]{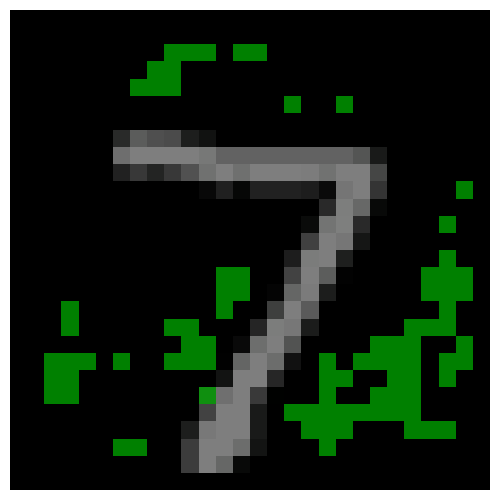}
            \label{fig:subfigure1b}
        \end{subfigure}
        \subcaption{Saliency (Proposed)}
        \label{fig:heuristic_compare:sensitivity}
    \end{minipage}
   \begin{minipage}[t]{0.19\textwidth}
        \begin{subfigure}[t]{0.48\textwidth}
            \centering
            \includegraphics[width=\textwidth]{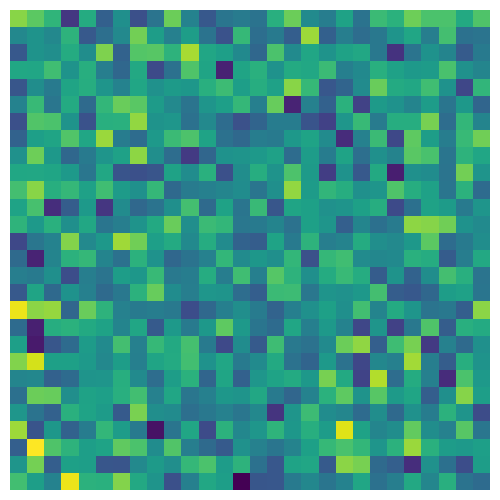}
            \label{fig:subfigure2a}
        \end{subfigure}
        \begin{subfigure}[t]{0.48\textwidth}
            \centering
            \includegraphics[width=\textwidth]{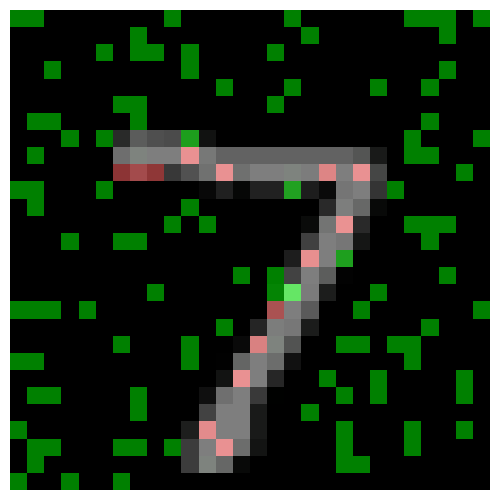}
            \label{fig:subfigure2b}
        \end{subfigure}
        \subcaption{Random}
    \end{minipage}
    \hspace{.30pt}
    \begin{minipage}[t]{0.19\textwidth}
        \begin{subfigure}[t]{0.48\textwidth}
            \centering
            \includegraphics[width=\textwidth]{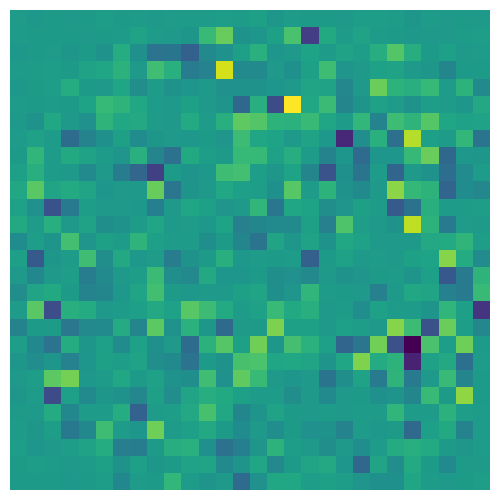}
            \label{fig:subfigure3a}
        \end{subfigure}
        \begin{subfigure}[t]{0.48\textwidth}
            \centering
            \includegraphics[width=\textwidth]{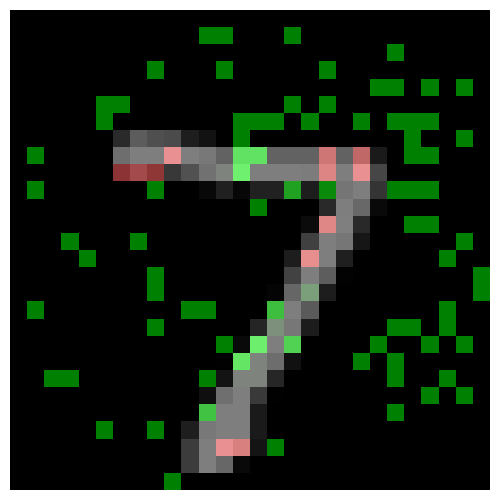}
            \label{fig:subfigure3b}
        \end{subfigure}
        \subcaption{IG}
    \end{minipage}
    \caption{\vitax{} with different heuristic functions. For each pair, the left side shows the output of the heuristic function, and the right side shows the corresponding \vitax{} explanations generated using the heuristic function's output.}
    \label{fig:heuristic_compare}
\end{figure}

\section{Impact of Approximation Methods on Reachability}\label{sec:Reachability Approximation Method}
    
\begin{table*}
\small
  \caption{The table presents the average percentage of cardinality, illustrating the proportion of features that significantly contribute to class transitions under $\epsilon$-perturbation. The columns represent the original sample classes $\mathbf{y}$, while the rows indicate the target explanation classes $\mathbf{t}$. Notably, the '-' symbol denotes cases where the target class is the same as the original class, which are purposefully excluded as they do not necessitate an explanation.}
  \label{tab:cardinality}
  \centering
  \begin{tabular}{ccccccccccc}
    \toprule
    & \multicolumn{10}{c}{Target Class $\mathbf{t}$} \\
    \cmidrule{2-11}
    Class $\mathbf{y}$ & 0 & 1 & 2 & 3 & 4 & 5 & 6 & 7 & 8 & 9 \\
    \midrule
    0 & - & 26.4\% & 21.3\% & 22.1\% & 26.8\% & 24.9\% & 21.4\% & 21.8\% & 24.5\% & 22.6\% \\
    1 & 11.0\% & - & 9.6\% & 7.6\% & 8.6\% & 8.7\% & 7.4\% & 5.8\% & 8.3\% & 8.4\% \\
    2 & 16.0\% & 13.3\% & - & 12.7\% & 18.8\% & 18.9\% & 18.0\% & 11.8\% & 17.1\% & 18.8\% \\
    3 & 17.1\% & 20.6\% & 14.6\% & - & 20.9\% & 13.3\% & 20.0\% & 19.7\% & 14.2\% & 14.5\% \\
    4 & 16.2\% & 17.7\% & 16.2\% & 17.5\% & - & 16.2\% & 14.9\% & 15.5\% & 17.0\% & 9.6\% \\
    5 & 10.7\% & 12.5\% & 11.4\% & 6.3\% & 13.8\% & - & 9.8\% & 12.4\% & 9.2\% & 10.2\% \\
    6 & 11.6\% & 18.4\% & 15.9\% & 17.4\% & 16.2\% & 13.0\% & - & 17.8\% & 16.0\% & 17.1\% \\
    7 & 18.6\% & 16.7\% & 14.6\% & 12.3\% & 17.6\% & 21.3\% & 24.4\% & - & 18.9\% & 11.5\% \\
    8 & 15.6\% & 18.2\% & 10.9\% & 10.8\% & 18.9\% & 15.4\% & 16.5\% & 15.0\% & - & 13.5\% \\
    9 & 16.3\% & 17.5\% & 17.1\% & 17.0\% & 10.0\% & 16.7\% & 18.6\% & 10.0\% & 17.9\% & - \\
    \bottomrule
  \end{tabular}
\end{table*}

We further analyze the role of different approximation methods in solvers for reachability analysis. Interestingly, we classify an unknown subset of features as unrobust (meaning shrinking down the subset), which differs from \cite{wu_verix_2023}, where important (\(\epsilon\)-robust) features are considered irrelevant, and unknown and unrobust features are seen as important. We use several sound and incomplete reach methods, including approx star, relax star (50\% relaxed), relax star (75\%), and relax star (85\%). Figure~\ref{fig:approx_method} shows that as the approximation becomes more relaxed (up to relax 85\%), the cardinality percentage decreases. \textit{Intuitively, this decrease indicates that some subsets are important and sensitive to perturbation, influencing the output significantly, while other subsets of features do not contribute to robustness}.  



\section{Cardinality Benchmarking} As show in Table \ref{tab:cardinality}, by using our method, we show that \vitax{} generates useful cardinality for evaluating and benchmarking other explanation methods, and it helps quantify feature importance variability across class transitions, indicating the proportion of features significantly contributing to explanations under $\epsilon$-perturbation. Table \ref{tab:cardinality} shows original sample classes on the vertical axis and target classes $\mathbf{t}$ on the horizontal axis, with a dashed line denoting the absence of explanation of itself. We randomly selected 100 correctly predicted test samples and calculated the average percentage of cardinality in MNIST using the same model, quantifying feature importance variability across class transitions. For example, explaining from class `0' to target classes `2', `3', `6', `7', and `9' shows consistent cardinality around 22\%, while transitioning from class `1' to other target classes shows a reduction from 22\% to around 8\%. The use of cardinality as a benchmark promotes a standardized approach for researchers, aiding in the selection and evaluation of features influenced by $\epsilon$-perturbation. \textit{This reduction suggests that cardinality proportion is vital for understanding sample robustness across classes, making it an potentially effective metric for evaluating and benchmarking explanation methods.} 


\clearpage
\section{Model Specifications} \label{appendix:models}
\tabcolsep=0.11cm
\begin{figure}[ht]
  \begin{minipage}{0.45\textwidth}
    \centering
    \captionof{table}{Structure for the MNIST MLP}
    \label{tab:structure:mnist_mlp}
    \begin{tabular}{lll}
      \toprule
      Type           & Parameters     & Activation \\
      \midrule
      Input          & \(28 \times 28 \times 1\) & -          \\
      Flatten        & -              & -          \\
      Fully Connected & 128            & ReLU       \\
      Fully Connected & 64             & ReLU       \\
      Fully Connected & 10             & softmax          \\
      \bottomrule
    \end{tabular}
  \end{minipage}%
  \begin{minipage}{0.45\textwidth}
    \centering
    \captionof{table}{Structure for CNN}
    \label{tab:structure:gtsrb_cnn}
    \begin{tabular}{lll}
      \toprule
      Type             & Parameters       & Activation \\
      \midrule
      Input            & \(28 \times 28 \times 3(1)\) & -          \\
      Convolution      & \(3 \times 3 \times 8\)   & ReLU       \\
      AVG Pool & \(2 \times 2\)            & -          \\
      Flatten          & -                         & -          \\
      F.C.  & \(43(9)(26)\)                    & softmax          \\
      \bottomrule
    \end{tabular}
  \end{minipage}
\end{figure}

\begin{figure}[ht]
  \begin{minipage}{0.45\textwidth}
    \centering
    \captionof{table}{Structure of the TaxiNet MLP}
    \label{tab:structure:taxinet_mlp}
    \begin{tabular}{lll}
      \toprule
      Type             & Parameters               & Activation \\
      \midrule
      Input            & \(27 \times 54 \times 1\) & -          \\
      Flatten          & -                        & -          \\
      Fully Connected  & 20                       & ReLU       \\
      Fully Connected  & 1                        & -          \\
      \bottomrule
    \end{tabular}
  \end{minipage}%
  \begin{minipage}{0.45\textwidth}
    \centering
    \captionof{table}{Structure of the small MLP}
    \label{tab:structure:small_mlp}  
    \begin{tabular}{lll}
      \toprule
      Type             & Parameters               & Activation \\
      \midrule
      Input            & \(28 \times 28 \times 1\) & -          \\
      Flatten          & -                        & -          \\
      Fully Connected  & 15                       & ReLU       \\
      Fully Connected  & 10                       & softmax          \\
      \bottomrule
    \end{tabular}
  \end{minipage}
\end{figure}

\begin{figure}[ht]
  \begin{minipage}{0.45\textwidth}
    \centering
    \captionof{table}{Structure of the MLP Dense}
    \label{tab:structure:mlp_dense}
    \begin{tabular}{lll}
      \toprule
      Type             & Parameters               & Activation \\
      \midrule
      Input            & \(h \times w \times c\) & -          \\
      Flatten          & -                        & -          \\
      Fully Connected  & 10                       & ReLU       \\
      Fully Connected  & 10                        & ReLU          \\
      Fully Connected  & \# classes         & softmax          \\
      \bottomrule
    \end{tabular}
  \end{minipage}%
  \begin{minipage}{0.45\textwidth}
    \centering
    \captionof{table}{Structure of the MLP Dense Large}
    \label{tab:structure:mlp_dense_large}  
    \begin{tabular}{lll}
      \toprule
      Type             & Parameters               & Activation \\
      \midrule
      Input            & \(h \times w \times c\) & -          \\
      Flatten          & -                        & -          \\
      Fully Connected  & 30                       & ReLU       \\
      Fully Connected  & 30                        & ReLU          \\
      Fully Connected  & \# classes         & softmax          \\
      \bottomrule
    \end{tabular}
  \end{minipage}
\end{figure}
\begin{figure}[h!]
  \begin{minipage}{0.45\textwidth}
    \centering
    \captionof{table}{Structure of the CNN Dense}
    \label{tab:structure:cnn_dense}
    \begin{tabular}{lll}
      \toprule
      Type             & Parameters               & Activation \\
      \midrule
      Input            & \(h \times w \times c\)  & -          \\
      Convolution  & \(3 \times 3 \times 4\)      & ReLU       \\
      Convolution  & \(3 \times 3 \times 4\)      & ReLU          \\
      Flatten          & -                        & -          \\
      Fully Connected  &  20                      & ReLU          \\
      Fully Connected  & \# classes               & softmax          \\
      \bottomrule
    \end{tabular}
  \end{minipage}%
  
\end{figure}

\end{document}